\documentclass[twoside,11pt]{article}

\usepackage{jmlr2e}

\usepackage{amsmath, amssymb}
\usepackage{mathrsfs, bm}
\usepackage{extarrows}
\usepackage{xcolor}
\usepackage{booktabs}
\usepackage{enumitem}
\usepackage{subfigure}
\usepackage{float}
\usepackage[linesnumbered,ruled,vlined]{algorithm2e}
\usepackage{multirow}

\SetKwInput{KwInput}{Input}                
\SetKwInput{KwReturn}{Return}

\SetCommentSty{mycommfont}
\usepackage{comment}

\def\bbE{\mathbb{E}}
\def\bbR{\mathbb{R}}
\def\bbP{\mathbb{P}}

\def\cB{\mathcal{B}}

\def\cD{\mathcal{D}}

\def\cG{\mathcal{G}}

\def\cN{\mathcal{N}}
\def\cU{\mathcal{U}}
\def\cP{\mathcal{P}}

\def\cS{\mathcal{S}}

\def\cX{\mathcal{X}}

\def\cZ{\mathcal{Z}}
\def\btheta{\bm{\theta}}
\DeclareMathOperator*{\argmin}{argmin}

\def\dkl{D_{\text{KL}}}
\def\pto{\overset{p}{\to}}

\newcommand{\revise}[1]{{\textcolor{black}{#1}}}
\newcommand{\rr}[1]{{\textcolor{black}{#1}}}
\newcommand{\rrr}[1]{{\textcolor{black}{#1}}}
\def\PA{\textit{pendulum\_angle}}
\def\LA{\textit{light\_angle}}
\def\SL{\textit{shadow\_length}}
\def\SP{\textit{shadow\_position}}

\newtheorem{assu}{Assumption}

\jmlrheading{23}{2022}{1-55}{1/21; Revised
7/22}{7/22}{21-0080}{Xinwei Shen, Furui Liu, Hanze Dong, Qing Lian, Zhitang Chen, and Tong Zhang}

\ShortHeadings{Weakly Supervised Disentangled Generative Causal Representation Learning}{Shen, Liu, Dong, Lian, Chen, and Zhang}
\firstpageno{1}

\begin{document}

\title{Weakly Supervised Disentangled Generative Causal Representation Learning}

\author{\name Xinwei Shen \email xshenal@ust.hk \\
       \addr Department of Mathematics \\
       The Hong Kong University of Science and Technology\\
       Hong Kong, China
       \AND
       \name Furui Liu \email liufurui@zhejianglab.com \\
       \addr Zhejiang Laboratory\\
       Hangzhou, China
       \AND
       \name Hanze Dong \email hdongaj@ust.hk \\
       \addr Department of Mathematics \\
       The Hong Kong University of Science and Technology\\
       Hong Kong, China
       \AND
       \name Qing Lian \email qlianab@ust.hk \\
       \addr Department of Computer Science \\
       The Hong Kong University of Science and Technology\\
       Hong Kong, China
       \AND
       \name Zhitang Chen \email chenzhitang2@huawei.com \\
       \addr Huawei Noah's Ark Lab\\
       Shenzhen, China
       \AND
       \name Tong Zhang \email tongzhang@ust.hk \\
       \addr Department of Computer Science and Mathematics\\
       The Hong Kong University of Science and Technology\\
       Hong Kong, China}

\editor{Yoshua Bengio}

\maketitle

\begin{abstract}
This paper proposes a Disentangled gEnerative cAusal Representation (DEAR) learning method \rrr{under appropriate supervised information}. Unlike existing disentanglement methods that enforce independence of the latent variables, we consider the general case where the underlying factors of interests can be causally related. We show that previous methods with independent priors fail to disentangle causally related factors \rrr{even under supervision}. Motivated by this finding, we propose a new disentangled learning method called DEAR that enables causal controllable generation and causal representation learning. The key ingredient of this new formulation is to use a structural causal model (SCM) as the prior distribution for a bidirectional generative model. The prior is then trained jointly with a generator and an encoder using a suitable GAN algorithm incorporated with supervised information \revise{on the ground-truth factors and their underlying causal structure}. We provide theoretical justification on the identifiability and asymptotic convergence of the proposed method. We conduct extensive experiments on both synthesized and real data sets to demonstrate the effectiveness of DEAR in causal controllable generation, and the benefits of the learned representations for downstream tasks in terms of sample efficiency and distributional robustness. 
\end{abstract}

\begin{keywords}
  disentanglement, causality, representation learning, deep generative model
\end{keywords}

\section{Introduction}

Consider the observed data $x$ from a distribution $q_x$ on $\mathcal{X}\subseteq\mathbb{R}^d$ and the latent variable $z$ from a prior $p_z$ on $\mathcal{Z}\subseteq\mathbb{R}^k$. In bidirectional generative models (BGMs), we are normally interested in learning an \textit{encoder} $E:\cX\to\cZ$ to infer latent variables and a \textit{generator} $G:\cZ\to\cX$ to generate data, to achieve both representation learning and data generation. Classical BGMs include Variational Autoencoder (VAE)~\citep{kingma2013auto} and BiGAN~\citep{Donahue2017AdversarialFL,Dumoulin2017AdversariallyLI}. 
 In representation learning, it was argued that an effective representation for downstream learning tasks should disentangle the underlying factors of variation~\citep{bengio2013representation}. In generative modeling, it is highly desirable if one can control the semantic generative factors by aligning them with the latent variables such as in StyleGAN~\citep{karras2019style}. 
Both goals can be achieved with the disentanglement of latent variable $z$, which informally means that each dimension of $z$ measures a distinct factor of variation in the data~\citep{bengio2013representation}. 

Earlier unsupervised disentanglement methods mostly regularized the VAE objective to encourage independence of learned representations~\citep{Higgins2017betaVAELB,Burgess2018UnderstandingDI,Kim2018DisentanglingBF,Chen2018IsolatingSO,Kumar2017VariationalIO}. Later, \citet{Locatelloetal19} showed that unsupervised learning of disentangled representations is impossible: many existing unsupervised methods are brittle, requiring careful supervised hyperparameter tuning or implicit inductive biases. To promote identifiability, recent work resorted to various forms of supervision~\citep{locatello2019disentangling,shu2019weakly,locatello2020weakly}. In this work, we also incorporate supervision on the ground-truth factors in the form \revise{of a certain number of annotated labels} as described in Section~\ref{sec:sup}. \revise{We will present experimental results showing that our method remains competitive with a small amount of labeled data (a minimum of around 100 samples).} 

Most of the existing methods, including those mentioned above, are built on the assumption that the underlying factors of variation are mutually independent. However, in many real-world cases, the semantically meaningful factors of interests are not independent~\citep{Bengio2020A}. Instead, such high-level variables are often causally related, i.e., connected by a causal graph. 

In this paper, we prove formally that methods with independent priors fail to disentangle causally related factors. Motivated by this observation, we propose a new method to learn disentangled generative causal representations called DEAR. The key ingredient of our formulation is a structural causal model (SCM)~\citep{pearl2000models} as the prior for latent variables in a bidirectional generative model. 
\revise{As discussed in Section~\ref{sec:learn_a}, we assume that a super-graph of the underlying causal graph is known a priori, which ranges from the causal ordering of the nodes in the graph to the true causal structure.} 
The causal model prior is then learned jointly with a generator and an encoder using a suitable GAN~\citep{goodfellow2014generative} algorithm.  
Moreover, we establish theoretical guarantees for DEAR \rr{on how it resolves the unidentifiability issue of many existing methods as well as on the asymptotic convergence of the proposed algorithm.} 

An immediate application of DEAR is causal controllable generation, which can generate data from many desired interventional distributions of the latent factors. Another useful application of disentangled representations is to use such representations in downstream tasks, leading to better sample complexity~\citep{bengio2013representation,scholkopf2012causal}. Moreover, it is believed that causal disentanglement is invariant and thus robust under distribution shifts~\citep{scholkopf2019causality,arjovsky2019invariant}. In this paper, we demonstrate these conjectures in various downstream prediction tasks for the proposed DEAR method, which has theoretically guaranteed disentanglement property.

We summarize our main contributions as follows:
\begin{itemize}
\item We formally identify a problem with previous disentangled representation learning methods using the independent prior assumption, and prove that they fail to disentangle when the underlying factors of interests are causally related, \rrr{even under supervision of the latents}.
\item We propose a new disentangled learning method, DEAR, which integrates an SCM prior into a bidirectional generative model, trained with a suitable GAN algorithm.
\item We provide theoretical justification on the identifiability\footnote{\rrr{Note that the identifiability in this work differs from that in \citet{khemakhem2020variational} in terms of goals and assumptions. See more discussions in the related work and below Proposition~\ref{thm:identify}.}} of the proposed formulation and the asymptotic convergence of our algorithm.
\item Extensive experiments are conducted on both synthesized and real data to demonstrate the effectiveness of DEAR in causal controllable generation, and the benefits of the learned representations for downstream tasks in terms of sample efficiency and distributional robustness.
\end{itemize}

\medskip
\noindent
\textbf{Notation} \ \ 
Throughout the paper, all distributions are assumed to be absolutely continuous with respect to Lebesgue measure unless indicated otherwise. For a vector $x$, let $[x]_i$ denote the $i$-th component of $x$. 
For a scalar function $h(x,y)$, let $\nabla_x h(x,y)$ denote its gradient with respect to $x$ and $\nabla^2_x h(x,y)$ denote its Hessian matrix with respect to $x$. For a vector function $g(x,y)$, let $\nabla_x g(x,y)$ denote its Jacobian matrix with respect to $x$. 
Without ambiguity, $\nabla_x$ is denoted by $\nabla$ for simplicity.
Notation $\|\cdot\|$ stands for the Euclidean norm. 
\begin{definition}[Smoothness]\label{def:smooth}
Consider a function $h(x):\bbR^{d}\to\bbR$. $h(x)$ is $\ell_0$-smooth with respect to $x$ if $h(x)$ is differentiable and its gradient is $\ell_0$-Lipschitz continuous, i.e., we have 
\begin{equation*}
	\|\nabla h(x)-\nabla h(x')\| \leq \ell_0\|x-x'\|,\quad \forall x,x'\in\bbR^{d}.
\end{equation*}
\end{definition}

\begin{definition}[Polyak-\L{ojasiewicz}]\label{def:pl}
For a set $\cS\subseteq\bbR^{d}$, consider a function $h(x):\cS\to\bbR$ and let $h^*=\min_{x\in\cS}h(x)$. Then $h(x)$ satisfies the Polyak-\L{ojasiewicz} (PL) condition if there exists $c>0$ such that for all $x\in\cS$
\[
h(x) - h^* \leq c \|\nabla h(x)\|_2^2.
\]
\end{definition}

\medskip
\noindent
\textbf{Roadmap} \ \ 
In Section~\ref{sec:related}, we discuss the related work. In Section~\ref{sec:prob_setting}, we introduce the problem setting of disentangled generative causal representation learning and identify a problem with previous methods. In Section~\ref{sec:dear_main}, we propose the model, formulation and algorithm of DEAR, and provide theoretical justifications on both identifiability and asymptotic convergence. We then present 
empirical studies concerning causal controllable generation, downstream tasks and structure learning as well as ablation studies in Section~\ref{sec:exp}, and conclude in Section~\ref{sec:conclude}. Detailed proofs of all theorems, propositions and lemmas are deferred to Appendix~\ref{app:proof}.

\section{Related work}\label{sec:related}
\textbf{VAE-based disentanglement methods.} A number of methods have been proposed to enrich the VAE loss by various regularizers to enforce the independence of the latent variables. $\beta$-VAE~\citep{Higgins2017betaVAELB} and Annealed VAE~\citep{Burgess2018UnderstandingDI} introduced extra constraints on the capacity of the latent bottleneck by adjusting the role of the KL term; Factor-VAE~\citep{Kim2018DisentanglingBF} and $\beta$-TCVAE~\citep{Chen2018IsolatingSO} encouraged the aggregated posterior (i.e., the marginal distribution of $E(x)$) to be factorized by penalizing its total correlation; DIP-VAE~\citep{Kumar2017VariationalIO} enforced a factorized aggregated posterior differently by matching its moments with those of a factorized prior. \revise{Going beyond the independence perspective, \citet{suter2019robustly} considered disentangled causal mechanisms, meaning that all the generative factors are conditionally independent given a common confounder.  This is one special case of causal relationship, while we consider more general cases where the factors can have more complex causal relationships, e.g., one factor can be a direct cause of another one.}

Based on the above methods, \citet{locatello2019disentangling} and \citet{locatello2020weakly} further incorporated supervised information on a few labels of the generative factors and pairs of observations which differ by a few factors respectively, where the former is more related to ours which is discussed detailedly in Section~\ref{sec:sup}. \citet{shu2019weakly} proposed several concepts related to disentanglement, based on which they analyzed three forms of weak supervision including restricted labeling, match pairing, and rank pairing.

\rr{Going beyond the independent prior, \citet{khemakhem2020variational} proposed a conditional VAE where the latent variables are assumed to be conditionally independent given some additionally observed variables. Built upon developments of nonlinear ICA, they presented the first principled identifiability theory of latent variable models, in particular VAEs, thus leading to a form of provable disentanglement under suitable conditions. Our work, in contrast, does not aim at achieving general identifiability of latent variable models or general provable disentanglement, but contributes to resolving the failure of existing methods in disentangling causally related factors. With this motivation, we consider more general model assumptions on the latent structure as well as generating transformations than those in \citet{khemakhem2020variational} which apply more suitably to real-world data. To achieve disentanglement of causal factors, we need to adopt a more direct and somehow stronger form of supervision than \citet{khemakhem2020variational}, i.e., we require annotated labels of true factors for a possibly small number of samples. See Appendix~\ref{app:sup} for a discussion on the two forms of supervision. The model in \citet{khemakhem2020variational}, however, has not yet been applied with the most advanced network architecture for image generation such as StyleGAN \citep{karras2019style}, nor can their conditional independent prior models the causal structure of true factors. 
Therefore, their model and theory do not apply here and our work should be regarded complementary.} 

\revise{To avoid the unidentifiability of the standard Gaussian prior caused by rotation transformations, \citet{stuhmer2020independent} proposed hierarchical non-Gaussian priors for unsupervised disentanglement, which is not rotationally invariant. However, there remains other kinds of mixing transformations that leave these priors invariant, leading to unidentifiability. Besides, their proposed priors cannot model the causal relationships.} 

\revise{
Recently, a concurrent work by \citet{trauble2021disentangled} conducted a large-scale empirical study to investigate the behavior of the most prominent disentanglement approaches on correlated data. In particular, they considered the case where the ground-truth factors exhibit pairwise correlation. Although pairwise correlation largely generalizes the independence assumption, it is less general than the causal correlation that we consider. For example, a parental node with multiple children immediately goes beyond pairwise correlation. Moreover, \citet{trauble2021disentangled} focused on verifying the problem that existing methods fail to learn disentangled representations for strongly correlated factors, while we identify the problem as a motivation to propose a method to resolve it and learn disentangled representations under the causal case.} 

\medskip
\noindent
\textbf{GAN-based disentanglement methods.} Existing GAN-based methods, including InfoGAN~\citep{chen2016infogan} and InfoGAN-CR~\citep{lin2020infogan}, differed from our proposed formulation mainly in two folds. First, they still assumed an independent prior for latent variables, so suffered from the same problem with the previous VAE-based methods mentioned above. Besides, the idea of InfoGAN-CR was to encourage each latent code to make changes that are easy to detect, which only applies well when the underlying factors are independent. Second, as a bidirectional generative modeling method, InfoGAN further required variational approximation apart from adversarial training, which is inferior to the principled formulation in BiGAN and AGES~\citep{shen2020bidirectional} that we adopt.

\medskip
\noindent
\textbf{Generative modeling involving causal models in the latent space.} CausalGAN~\citep{kocaoglu2017causalgan} and a concurrent work~\citep{moraffah2020can} of ours, were unidirectional generative models (i.e., a generative model that learns a single mapping from the latent variable to data) that build upon a cGAN~\citep{mirza2014conditional}. They assigned an SCM to the conditional attributes while leaving the latent variables as independent Gaussian noises. The limit of a cGAN is that it always requires full supervision on attributes to apply conditional adversarial training. Also, the ground-truth factors were directly fed into the generator as the conditional attributes, without any extra effort to align the dimensions between the latent variables and the underlying factors, so their models had nothing to do with disentanglement learning.
Moreover, their unidirectional nature made them unable to learn representations. 
Besides, they only considered binary factors, so the consequent semantic interpolations appear non-smooth, as shown in Appendix~\ref{app:more}. 

CausalVAE~\citep{yang2020causalvae} assigned the SCM directly on the latent variables, while built upon iVAE~\citep{khemakhem2020variational}, it adopted a conditional prior given the ground-truth factors so was also limited to a fully supervised setting. 

\revise{
GraphVAE~\citep{he2018graphvae} generalized the chain-structured latent space proposed in Ladder VAE~\citep{sonderby2016ladder} and imposed an SCM into the latent space of VAE. The motivation behind GraphVAE is to improve the expressive capacity of VAE rather than to disentangle the underlying causal factors as ours. Purely from observational data and without any supervision on the underlying factors, the impossibility result from \citet{Locatelloetal19} indicated that a VAE model cannot identify the true factors. Therefore, the representations learned by GraphVAE were not guaranteed to disentangle the generative factors, and consequently the learned SCM did not reflect the true causal structure in principle. 
Moreover, their adopted VAE loss (ELBO) required an explicit form of KL divergence between the prior and the posterior, which limited the model choice for the SCM. Specifically, GraphVAE used an additive noise model with Gaussian noises. In contrast, our method does not require the distribution induced by the SCM to be explicitly expressed and in principle allows any SCMs that can be reparametrized as a generative model (i.e., given the exogenous noises, one can generate all the variables by ancestral sampling).  
For comparison, in our experiments, we include a baseline which extends the original GraphVAE method to incorporate the same amount of supervision as ours.}

\medskip
\noindent
 \textbf{Generative modeling involving other structured latent spaces.} VLAE~\citep{zhao2017learning} decomposed the latent space into separate chunks each of which is processed at different levels of the encoder and decoder. VQ-VAE-2~\citep{razavi2019generating} used a two-level latent space along with a multi-stage generation mechanism to capture both high and low level information of data. 
\revise{SAE~\citep{leeb2020structural} encouraged a hierarchical structure in the latent space through the structural architecture of the decoder. 
 These methods essentially adopted implicit probabilistic or architectural hierarchies, in contrast to the causal structure that we impose to the latent space, and thus cannot achieve the goal of causal disentanglement. For example, the hierarchy in SAE represents the level of abstraction, in the sense that more high-level, abstract features are processed deeper in the decoder and low-level, linear features are treated towards the end of the network. Such hierarchy differs essentially from the causal structure that we consider.} 

\medskip
\revise{Other works considered inferring the latent causal structure from visual data in the reinforcement learning setting \citep{dasgupta2019causal,nair2019causal}. In particular, \citet{nair2019causal} developed learning-based approaches to induce causal knowledge in the form of directed acyclic graphs, which was then utilized in learning goal-conditioned policies. The interactive environment enables the agent to perform actions and observe their outcomes. Therefore, the resulting data involves various interventions each of which entails an SCM and thus is essentially different from the common setting in the disentanglement literature which is also considered in this paper, where the observed data are independent and identically distributed.}

\section{Problem setting}\label{sec:prob_setting}

In this section, we describe the probabilistic framework of disentanglement learning based on bidirectional generative models (BGMs) with supervision, and formalize the unidentifiablility problem with previous methods.

\subsection{Generative model}\label{sec:gen}
We follow the commonly assumed two-step data generating process that first samples the underlying generative factors, and then conditional on those factors, generates the data~\citep{kingma2013auto}. 
During the generation process, the generator induces the generated conditional $p_{G}(x|z)$ and generated joint distribution $p_{G}(x,z)=p_z(z)p_G(x|z)$. During the inference process, the encoder induces the encoded conditional $q_E(z|x)$ which can be a factorized Gaussian and the encoded joint distribution $q_{E}(x,z)=q_x(x)q_{E}(z|x)$. 

We consider the following objective for generative modeling:
\begin{equation}\label{eq:obj_gen}
L_{\text{gen}}(E,G)=\dkl(q_{E}(x,z),p_{G}(x,z)),
\end{equation} 
where $\dkl(q,p)=\int q(x,z)\log(q(x,z)/p(x,z))dxdz$ is the Kullback-Leibler (KL) divergence between two distributions. Objective (\ref{eq:obj_gen}) is shown to be equivalent to the negative evidence lower bound (ELBO), 
\begin{equation}\label{eq:vae}
	\bbE_{x\sim q_x}[-\bbE_{q_E(z|x)}\log p_G(x|z) + \dkl(q_E(z|x),p_z(z))],
\end{equation}
used in VAEs up to a constant, and ELBO allows a closed form to be optimized easily only with factorized Gaussian prior, encoder and generator~\citep{shen2020bidirectional}.

Since constraints on the latent space are required to enforce disentanglement, it is desirable that the distribution families of $q_{E}(x,z)$ and $p_{G}(x,z)$ should be large enough, especially for complex data like images. As demonstrated in literature on image generation~\citep{karras2019style,mescheder2017adversarial}, implicit distributions, where the randomness is fed into the input or intermediate layers of the network, are favored over factorized Gaussians in terms of expressiveness. Then minimizing (\ref{eq:obj_gen}) requires adversarial training, as discussed detailedly in Section~\ref{sec:alg}.

\subsection{Supervised regularizer} \label{sec:sup}

To guarantee disentanglement, we incorporate supervision when training the BGM. \revise{The first part of supervision consists of a certain number of annotated labels of the ground-truth factors,} following the similar idea in \citet{locatello2019disentangling} but with a different formulation. \revise{We leverage another part of supervision on the graph structure of the factors, which will be discussed in Section~\ref{sec:learn_a}.} 
Specifically, let $\xi\in\bbR^{m}$ be the underlying ground-truth factors of interests of data $x$, following distribution $p_\xi$, and $[y]_i$ be some continuous or discrete annotated observation of the $i$-th underlying factor $[\xi]_i$, satisfying $[\xi]_i=\bbE([y]_i|x)$ for $i=1,\dots,m$.
For example, in the case of human face images, $[y]_1$ can be the binary label indicating whether a person is young or not, and $[\xi]_1=\mathbb{E}([y]_1|x)=\bbP([y]_1=1|x)$ is the probability of being young given one image $x$.

Let $\bar{E}(x)$ be the deterministic part of the stochastic transformation $E(x)$, i.e., $\bar{E}(x)=\bbE(E(x)|x)$ by integrating out the additional randomness injected into the encoder, which is used for representation learning. For instance, consider a Gaussian encoder satisfying $E(x)|x\sim\cN(m(x),\Sigma(x))$ which can be reparametrized by $E(x)=m(x)+\Sigma(x)^\top\epsilon$ with $\epsilon\sim\cN(0,I)$. Then the deterministic part is the mean, i.e., $\bar{E}(x)=m(x)$. 

We consider the following objective:
\begin{equation}\label{eq:obj}
L(E,G)=L_{\text{gen}}(E,G)+\lambda L_{\text{sup}}(E),
\end{equation}
where the supervised regularizer is $L_{\text{sup}}=\bbE_{x,y}[l_s(E;x,y)]$ with $l_s=\sum_{i=1}^m\mathrm{CE}([\bar{E}(x)]_i,[y]_i)$ if $[y]_i$ is the binary or bounded (and normalized to $[0,1]$) continuous label of factor $[\xi]_i$, where $\mathrm{CE}(l,y)=-y\log\sigma(l)-(1-y)\log(1-\sigma(l))$ is the cross-entropy loss with $\sigma(\cdot)$ being the sigmoid function; $l_{s}=\sum_{i=1}^m([\bar{E}(x)]_i-[y]_i)^2$ if $[y]_i$ is the continuous observation of $[\xi]_i$. $\lambda>0$ is the coefficient to balance both terms. Through ablation studies in Section~\ref{sec:ablation}, we empirically find the choice of $\lambda$ insensitive to different tasks and data sets, and hence set $\lambda=5$ in all experiments. 

Note that in the objective (\ref{eq:obj}), the unsupervised generative modeling loss and the supervised regularizer are decoupled in terms of taking expectations, in contrast to the conditional GANs where supervised labels are involved in the GAN loss. This enables one to use two separate samples with different sample sizes to estimate the two terms in (\ref{eq:obj}) during training. Since in practice we may only have access to a limited amount of annotated labels, this property makes the formulation applicable in such semi-supervised settings. \revise{In the experiments, we conduct ablation studies to investigate how our method performs with varying amounts of labeled samples available.}

In addition, \citet{locatello2019disentangling} propose a regularizer $L_{\text{sup}}=\sum_{i=1}^m\bbE_{x,z}(\mathrm{CE}([\bar{E}(x)]_i,\\~[z]_i))$ involving only the latent variable $z$ which is a part of the generative model, without distinguishing the model component $z$ from the ground-truth factor $\xi$ and its observation $y$. Hence they do not establish formal theoretical justification on disentanglement. Moreover, they follow the earlier VAE-based methods to adopt a VAE loss (\ref{eq:vae}) for generative modeling with an independent prior and an additional regularizer to enforce independence of the latent variables, which suffers from the unidentifiability problem described in the next section.

\subsection{Unidentifiability with an independent prior}

Intuitively, the above supervised regularizer aims at ensuring some kind of alignment between the underlying factor $\xi$ and the latent variable $z$ in the model. 
We start with the definition of a disentangled representation following this intuition. 

\begin{definition}[Disentangled representation]\label{def}
Given the underlying factor $\xi\in\bbR^{m}$ of data $x$, a deterministic encoder $E$ is said to learn a disentangled representation with respect to $\xi$ if $\forall i=1,\dots,m$, there exists a 1-1 function $g_i$ such that $[E(x)]_i=g_i([\xi]_i)$. Further, a stochastic encoder $E$ is said to be disentangled with respect to $\xi$ if its deterministic part $\bar{E}(x)$ is disentangled with respect to $\xi$.
\end{definition}

Note that in general, the goal of disentanglement allows for permutations in the ground-truth factors. For example one may expect for all $i$ there exists $j$ which is not necessarily equal to $i$ such that $[E(x)]_i=g_j([\xi]_j)$. However since in our method we supervise each latent dimension by the annotated label of each ground-truth factor, we can expect a component-wise correspondence between $E(x)$ and $\xi$, as justified formally in Proposition~\ref{thm:identify} below. 

As introduced above, we consider the general case where the underlying factors of interests are causally related. Then the goal becomes to disentangle the causal factors. Previous methods mostly use an independent prior for $z$, which contradicts the truth. We make this formal through the following proposition, which indicates that the disentangled representation is generally unidentifiable with an independent prior.

\begin{proposition}\label{prop}
Let $E^*$ be any encoder that is disentangled with respect to $\xi$. Let $b^*=L_{\rm{sup}}(E^*)$, $a=\min_{G}L_{\rm{gen}}(E^*,G)$, and $b=\min_{\{(E,G):L_{\rm{gen}}=0\}}L_{\rm{sup}}(E)$. Assume the elements of $\xi$ are connected by a causal graph whose adjacency matrix $A_0$ is not a zero matrix. Suppose the prior $p_z$ is factorized, i.e., $p_z(z)=\prod_{i=1}^k p_i([z]_i)$. Then we have $a>0$, and either when $b^*\geq b$ or $b^*<b$ and $\lambda<\frac{a}{b-b^*}$, there exists a solution $(E',G')$ so that $E'$ is entangled and for any generator $G$, we have $L(E',G')<L(E^*,G)$. 
\end{proposition}

This proposition directly suggests that minimizing (\ref{eq:obj}) favors an entangled solution $(E',G')$ over the one with a disentangled encoder $E^*$. Thus, with an independent prior we have no way to identify the disentangled solution with $\lambda$ that is not large enough. However, in real applications, it is impossible to estimate the threshold, and too large $\lambda$ makes it difficult to learn the BGM. 
After our work was submitted, we were brought attention to a theoretical result in \citet{trauble2021disentangled} that is similar to our Proposition~\ref{prop}. 
A discussion on the two independently proposed results is given in Appendix~\ref{app:prof_prop} after the proof.
In the following section, we propose a solution to this problem.

\section{Causal disentanglement learning}\label{sec:dear_main}

In this section, we propose the DEAR method for causal disentanglement learning. We start with an introduction to the model structure in Section~\ref{sec:gen_causal}. Then we present the formulation of DEAR as well as its identifiability of disentanglement at a population level in Section~\ref{sec:identif}. The DEAR algorithm is described in Section~\ref{sec:alg} with its consistency results established in Section~\ref{sec:cons}.

\subsection{Generative model with a causal prior}\label{sec:gen_causal}

We introduce the proposed bidirectional generative model with a causal model prior, and discuss the learning of the adjacency matrix. Based on the model we describe the mechanism of causal controllable generation from interventional distributions. We further propose a composite prior to deal with the issue of setting the latent dimension.

\subsubsection{SCM prior}
We propose to use a causal model as the prior $p_z$. Specifically we adopt the general nonlinear Structural Causal Model (SCM) proposed by \citet{yu2019dag} as follows
\begin{equation}\label{eq:sem}
z=f((I-A^\top)^{-1}h(\epsilon)):=F_\beta(\epsilon),
\end{equation}
where $A$ is the weighted adjacency matrix of the directed acyclic graph (DAG) upon the $k$ elements of $z$ (i.e., $A_{ij}\neq0$ if and only if $[z]_i$ is the parent of $[z]_j$), $\epsilon$ denotes the exogenous variables following $\cN(0,I)$, $f$ and $h$ are element-wise transformations that are generally nonlinear, and $\beta=(f,h,A)$ denotes the set of parameters of $f$, $h$ and $A$, with the parameter space $\cB$. Further let $\mathbf{I}_A=\mathbf{I}(A\neq0)$ denote the corresponding binary adjacency matrix, where $\mathbf{I}(\cdot)$ is the element-wise indicator function.

When $f$ is invertible, (\ref{eq:sem}) is equivalent to 
\begin{equation}\label{eq:intervene}
f^{-1}(z)=A^\top f^{-1}(z)+h(\epsilon),
\end{equation}
which indicates that the factors $z$ satisfy a linear SCM after nonlinear transformation $f$, and enables interventions on latent variables as discussed later. 

By combining the above SCM prior and the encoder and generator introduced in Section~\ref{sec:gen}, we end up with the model structure presented in Figure~\ref{fig:model}.
Note that different from our model where $z$ is the latent variable following the prior (\ref{eq:sem}) with the goal of causal disentanglement, \citet{yu2019dag} propose a causal discovery method where variables $z$ in SCM (\ref{eq:sem}) are observed with the aim of learning the causal structure among $z$. 

\begin{figure}
\centering
\includegraphics[width=0.85\textwidth]{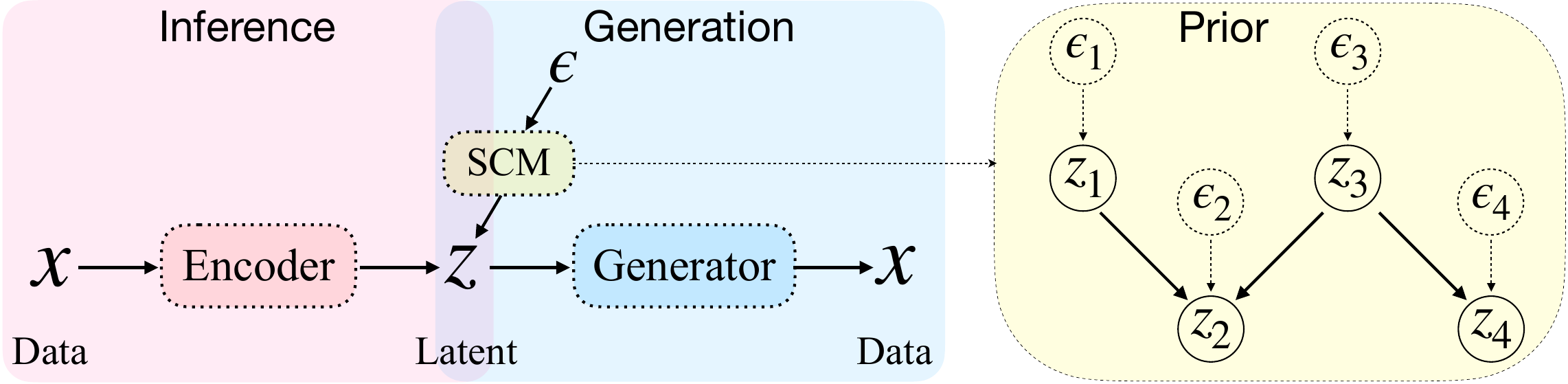}
\caption{Model structure of a BGM (left) with an SCM prior (right).}\label{fig:model}
\end{figure}

\subsubsection{\revise{Learning of $A$}}\label{sec:learn_a}

In causal structure learning, the graph is required to be acyclic. Traditional causal discovery methods such as PC~\citep{spirtes2000causation} or GES~\citep{chickering2002optimal} deal with the combinatorial problem over the discrete space of DAGs. Recently, \citet{zheng2018dags} proposed an equality constraint whose satisfaction ensures acyclicity and solved the problem with the augmented Lagrangian method, which however leads to optimization difficulties~\citep{ng2020role}. 
In addition, identifiability of the causal structure from purely observational data is known as an important issue in causal discovery. Despite a number of results on structure identifiability under various parametric or semi-parametric assumptions~\citep{zhang2012pnl,peters2014equal}, in a general nonparametric setting, however, it cannot be guaranteed. \citet{yu2019dag} did not discuss the identifiability of the SCM (\ref{eq:sem}) under general cases.

In many problems of disentanglement, we have some prior information on the causal structure of the factors of interests based on common knowledge or expertise. In particular, we may know a causal ordering of the factors. In addition to the ordering, for some factors, we may know that one particular factor cannot be a direct cause of another one, which helps us remove some redundant edges in advance. 
Therefore, in this paper with the focus on disentanglement, we utilize such prior information on the graph structure in disentanglement learning and leave incorporating causal discovery from scratch to future work. 
Formally, we assume the super-graph of the true binary graph $\mathbf{I}_{A_0}$ is given, the best case of which is the true graph while the worst is that only the causal ordering is available. Then we learn the weights of the non-zero elements of the prior adjacency matrix that indicate the sign and scale of causal effects, jointly with other parameters of the generative model using the formulation and algorithm described in Sections~\ref{sec:identif} and \ref{sec:alg}. 

As discussed in Section~\ref{sec:identif}, such prior knowledge makes the structure identifiability easy to hold. Moreover, the given super-graph ensures the acyclicity of the adjacency matrix, allowing us to get rid of the additional acyclicity constraint. 
In Section \ref{sec:exp_learnA}, we investigate how our method performs in learning the graph structure and weighted adjacency given various amounts of prior graph information. 
Note that even when a super-graph is available, to our best knowledge, no previous disentanglement method except GraphVAE~\citep{he2018graphvae} can utilize them to disentangle causal factors with guarantee, but we propose one such method and show its effectiveness. In fact, \citet{he2018graphvae} also assumed an ordering over the latent nodes by specifying that the parents of node $z_i,i=1,\dots,k-1$ come from the set $\{z_{i+1},\dots,z_k\}$. Later experiments suggest that GraphVAE shows inferior performance compared with ours. 

\subsubsection{Generation from interventional distributions}\label{sec:intervene_mech}

One immediate application of our proposed model is causal controllable generation from interventional distributions of the latent variables. We now describe the mechanism. To enable intervention under SCM (\ref{eq:intervene}), we require $f$ to be invertible. Then interventions can be formalized as operations that modify a subset of equations in (\ref{eq:intervene})~\citep{pearl2000models}.

Suppose we would like to intervene on the $i$-th dimension of $z$, i.e., Do$([z]_i=c)$, where $c$ is a constant. Once we obtain the latent factors $z$ inferred from data $x$, i.e., $z=E(x)$, or sampled from prior $p_z$, we follow the modified equations in (\ref{eq:intervene}) to obtain $z'$ on the left-hand side using ancestral sampling by performing (\ref{eq:intervene}) iteratively, where $\epsilon$ can be either fixed or resampled from its prior. Then we decode the latent factor $z'$ that follows the given interventional distribution to generate the desired sample $G(z')$. In Section~\ref{sec:contgen} we define the two types of interventions of most interests in applications.
We discuss how our method generalizes to unseen interventions in Appendix~\ref{app:unseen}.

\subsubsection{Latent dimension and composite prior}\label{sec:comp_prior}
Another issue of the model is how to set the latent dimension $k$ of the generative model, to handle which we propose the so-called composite prior. Recall that $m$ is the number of generative factors that we are interested to disentangle, for example, all the semantic concepts related to some  filed, where $m$ tends to be smaller than the total number $M$ of generative factors.
The latent dimension $k$ should be no less than $M$ to allow a sufficient degree of freedom in order to generate or reconstruct data well. Since $M$ is generally unknown in reality, we set a sufficiently large $k$, at least larger than $m$ which is a trivial lower bound of $M$. 

Then we propose to use a prior that is a composition of a causal model for the first $m$ dimensions and another distribution for the other $k-m$ dimensions to capture other factors necessary for generation, like a standard Gaussian. 
In this way the first $m$ dimensions of $z$ aim at learning the disentangled representation of the $m$ factors of interests, while the role of the remaining $k-m$ dimensions is to capture other factors that are necessary for generation whose structure is neither cared nor explicitly modeled.  
Under this model framework, we do not require the availability of annotated labels for all generative factors of data, but only the ones of our interests to disentangle are used in the supervised regularizer in (\ref{eq:obj}), which broadens the applications of our method.

\subsection{DEAR formulation}\label{sec:identif}

In this section, we first present the formulation of DEAR. 
Compared with the BGM described in Section~\ref{sec:gen}, now we have one more module to learn which is the SCM prior. Thus $p_G(x,z)$ becomes $p_{G,F}(x,z)=p_F(z)p_G(x|z)$ where $p_F(z)$ is the  distribution of $F_\beta(\epsilon)$ with $\epsilon\sim\cN(0,I)$. 
We then rewrite the generative model loss as follows
\begin{equation}\label{eq:obj_gen_new}
L_{\text{gen}}(E,G,F) = \dkl(q_{E}(x,z),p_{G,F}(x,z)).
\end{equation}

Then we propose the following formulation to learn disentangled generative causal representations:
\begin{equation}\label{eq:obj_new}
\min_{E,G,F}\ L(E,G,F):=L_{\text{gen}}(E,G,F)+\lambda L_{\text{sup}}(E). 
\end{equation}
%

\rr{Now we show the identifiability of disentanglement of DEAR in contrast to the unidentifiability result in Proposition~\ref{prop}. Proposition~\ref{thm:identify} indicates that under appropriate conditions, the DEAR formulation \eqref{eq:obj_new} at a population level can learn the disentangled representations defined in Definition~\ref{def}. 
Here, Assumption~\ref{assu1} supposes a sufficiently large capacity of the SCM in (\ref{eq:sem}) to contain the underlying distribution $p_\xi$, which is reasonable due to the generalization of the nonlinear SCM. } 


\begin{assu}
The underlying distribution $p_\xi$ belongs to the distribution family $\{p_\beta:\beta\in\cB\}$, i.e., there exits $\beta_0=(f_0,h_0,A_0)$ such that $p_\xi=p_{\beta_0}$.
\label{assu1}
\end{assu}
%


\begin{proposition}[Identifiability]\label{thm:identify}
Assume the infinite capacity of $E$ and $G$ and Assumption~\ref{assu1}. Let $(E^*,G^*,F^*)\in\argmin_{E,G,F}L(E,G,F)$ which is the solution of DEAR formulation (\ref{eq:obj_new}). Then $E^*$ is disentangled with respect to $\xi$ as defined in Definition~\ref{def}.
\end{proposition}


\rr{Note that Proposition~\ref{thm:identify} states the identifiability at the population level, i.e., the loss function is taken the expectation over distributions of both the data and labels of the true factors. Thus we clarify that 
Proposition~\ref{thm:identify} does not obtain general provable disentanglement which should be analyzed with a much weaker form of supervision on the true factors, e.g., as in \citet{khemakhem2020variational}. 
In contrast, the specific identifiability stated in Proposition~\ref{thm:identify} should be interpreted as a counterpart of the unidentifaibility result in Proposition~\ref{prop}. Specifically, Proposition~\ref{prop} shows that the independent prior used by most existing disentanglement methods causes the contradiction between the generative loss $L_{\rm gen}$ and the supervised loss $L_{\rm sup}$ in \eqref{eq:obj}, which makes the whole loss $L$ prefer an entangled model. Therefore, even with the same amount of supervised labels of true factors, those methods cannot learn a generative model with disentangled latent representations. In contrast, Proposition~\ref{thm:identify} formally suggests that due to the introduction of the SCM prior, the two loss terms $L_{\rm gen}$ and $L_{\rm sup}$ in \eqref{eq:obj_new} can be simultaneously minimized and the jointly optimal solution leads to the disentangled model.}



\subsection{Algorithm} \label{sec:alg}

In this section, we propose the algorithm to solve the above formulation (\ref{eq:obj_new}). 
Estimating $L_{\text{gen}}$ requires the unlabeled data set $\{x_1,\dots,x_N\}$ with sample size $N$, while estimating $L_{\text{sup}}$ requires a labeled data set $\{(x_{j},y_{j}):j=1,\dots,N_s\}$, where the sample size $N_s$ can be much smaller than $N$. Without loss of generality, let $S_{\cG}=\{x_1,\dots,x_N,y_1,\dots,y_{N_s}\}$ denote the training data set for the generative model.

We parametrize $E_\phi(x)$ and $G_\theta(z)$ by neural networks. As mentioned in Section~\ref{sec:gen}, to enhance the expressiveness of the generative model, we use an implicit generated conditional $p_G(x|z)$, where we inject Gaussian noises to each convolution layer in the same way as \citet{shen2020bidirectional}. Then the SCM prior $p_F(z)$ and implicit $p_G(x|z)$ make (\ref{eq:obj_gen_new}) lose an analytic form. Hence we adopt a GAN method to adversarially estimate the gradient of (\ref{eq:obj_gen_new}) as in \citet{shen2020bidirectional}. Different from their setting, the prior also involves learnable parameters, that is, the parameters $\beta$ of the SCM. In the following lemma we present the gradient formulas of (\ref{eq:obj_gen_new}). 
\begin{lemma}\label{lem:grad_formula}
Let $D^*(x,z)=\log[q_E(x,z)/p_{G,F}(x,z)]$. Then we have
\begin{equation}\label{eq:grad}
\begin{split}
\nabla_\theta L_{\rm{gen}} &= -\bbE_{z\sim p_\beta(z)}[s(x,z)\nabla_x D^*(x,z)^\top|_{x=G_\theta(z)}\nabla_\theta G_\theta(z)],\\
\nabla_\phi L_{\rm{gen}} &= \bbE_{x\sim q_x}[\nabla_z D^*(x,z)^\top|_{z=E_\phi(x)}\nabla_\phi E_\phi(x)],\\
\nabla_\beta L_{\rm{gen}} &= -\bbE_{\epsilon}[s(x,z)(\nabla_x D^*(x,z)^\top \nabla_\beta G(F_\beta(\epsilon))+\nabla_z D^*(x,z)^\top \nabla_\beta F_\beta(\epsilon))|^{x=G(F_\beta(\epsilon))}_{z=F_\beta(\epsilon)}],
\end{split}
\end{equation}
where $s(x,z)=e^{D^*(x,z)}$ is the scaling factor.
\end{lemma}

Since $D^*$ depends on the unknown densities, which makes the gradients in (\ref{eq:grad}) uncomputable directly from data, we estimate the gradients by training a discriminator $D$ via the empirical logistic regression: 
\begin{equation}\label{eq:logistic}
\min_{D'}\frac{1}{N_d}\bigg[\sum_{i:w_i=1}\log(1+e^{-D'(x_i,z_i)}) + \sum_{i:w_i=0}\log(1+e^{D'(x_i,z_i)})\bigg],
\end{equation}
where the class label $w_i=1$ if $(x_i,z_i)\sim q_E$ and $w_i=0$ if $(x_i,z_i)\sim p_{G,F}$, with $i=1,\dots,N_d$. We parametrize the discriminator using neural networks with parameter $\psi$.

Based on the above, we propose Algorithm~\ref{algo} to learn disentangled generative causal representation.

\begin{algorithm}
\DontPrintSemicolon
\KwInput{training set $S_{\cG}$, initial parameter $\phi,\theta,\beta,\psi$, batch-size $n$, meta-parameter $T$}
\For{$t=1,\dots,T$}{
\For{multiple steps}{
Sample $\{x_1,\ldots,x_n\}$ from the training set, $\{\epsilon_1,\ldots,\epsilon_n\}$ from $\cN(0,I)$\\
Generate from the causal prior $z_i=F_\beta(\epsilon_i),i=1,\dots n$\\
Update $\psi$ by descending the stochastic gradient:
$\frac{1}{n}\sum_{i=1}^n \nabla_\psi\left[\log(1+e^{-D_\psi(x_i,E_\phi(x_i))})+\log(1+e^{D_\psi(G_\theta(z_i),z_i)})\right]$
}
Sample $\{x_1,\ldots,x_n,y_1,\dots,y_{n_s}\}$, $\{\epsilon_1,\ldots,\epsilon_n\}$ as above; generate $z_i=F_\beta(\epsilon_i)$\\
Compute $\theta$-gradient:
$-\frac{1}{n}\sum_{i=1}^n s(G_\theta(z_i),z_i)\nabla_\theta D_\psi(G_\theta(z_i),z_i)$\\
Compute $\phi$-gradient:
$\frac{1}{n}\sum_{i=1}^n \nabla_\phi D_\psi(x_i,E_\phi(x_i))+\frac{\lambda}{n_s}\sum_{i=1}^{n_s} \nabla_\phi l_s(\phi;x_i,y_i)$\\
Compute $\beta$-gradient:
$-\frac{1}{n}\sum_{i=1}^n s(G(z_i),z_i)\nabla_\beta D_\psi(G_\theta(F_\beta(\epsilon_i)),F_\beta(\epsilon_i))$\\
Update parameters $\phi,\theta,\beta$ using the gradients
}
\KwReturn{$\phi,\theta,\beta$}
\caption{Disentangled gEnerative cAusal Representation (DEAR) Learning}
\label{algo}
\end{algorithm}

\subsection{Consistency}\label{sec:cons}

In this section, we show the asymptotic convergence of Algorithm~\ref{algo}.
Let $\btheta=(\theta,\phi,\beta)$ denote the set of parameters of the generative model, where $\theta$, $\phi$ and $\beta$ denote the parameters of the generator, encoder and SCM prior respectively. According to such parametrization, we write the objective function in (\ref{eq:obj_new}) as $L(\btheta)$. 
In this section, we establish the consistency result of empirical estimator $\hat\btheta$, i.e., the output of Algorithm~\ref{algo}, under the parametric setting. 
Given a discriminator $D$, the approximate gradient used in the algorithm is denoted by
\begin{equation*}
h_D(\btheta)=
\begin{bmatrix}
	-\frac{1}{N}\sum_{i=1}^N \left[s(G_\theta(z_i),z_i)\nabla_x  D(G_\theta(z_i),z_i)^\top \nabla_\theta G_\theta(z_i)\right]\\
	\frac{1}{N}\sum_{i=1}^N \nabla_z D(x_i,E_\phi(x_i))^\top\nabla_\phi E_\phi(x_i)+\frac{\lambda}{N_s}\sum_{i=1}^{N_s} \nabla_\phi l_s(\phi;x_i,y_i)\\
	-\frac{1}{N}\sum_{i=1}^N s(x,z)[\nabla_x D(x,z)^\top \nabla_\beta G(F_\beta(\epsilon_i))+\nabla_z D(x,z)^\top \nabla_\beta F_\beta(\epsilon_i)]|^{x=G(F_\beta(\epsilon_i))}_{z=F_\beta(\epsilon_i)}
\end{bmatrix}.
\end{equation*}

We first show in the following lemma that under appropriate conditions the approximate gradient $h_{\hat D}(\btheta)$ based on the solution of (\ref{eq:logistic}) converges uniformly in probability to the true gradient. Recall the definition $D^*(x,z)=\log(q_E(x,z)/p_{G,F}(x,z))$ which depends on $\btheta$. Let $\cD^*=\{D^*_{\btheta}(x,z):\btheta\in\Theta\}$ denote the true discriminator class, and $\cD=\{D(x,z)\}$ denote the modeled discriminator class with the norm $\|D\|_1=\int|D(x,z)|p^*_{\btheta}(x,z)dxdz$, where $p^*_{\btheta}(x,z)=(q_E(x,z)+p_{G,F}(x,z))/2$ which induces the probability measure $\mu^*_{\btheta}$.

\begin{lemma}\label{lem:grad_conv}
Assume the parameter space $\Theta=\{\btheta=(\theta,\phi,\beta)\}$ is compact. 
Assume the following regularity conditions hold:
\begin{enumerate}[label=\textit{C\arabic*}]\vspace{-0.1cm}
\setlength{\itemsep}{2pt}
\setlength{\parskip}{2pt}
\item $D^*_{\btheta}$ is smooth with respect to $\btheta$ over $\Theta$, as defined in Definition~\ref{def:smooth}.\label{ass:D_smooth}
\item The modeled discriminator class $\cD$ is compact, and contains the true class $\cD^*$.\label{ass:d_compact}
\item $\{\mu^*_{\btheta}:\btheta\in\Theta\}$ is uniformly tight, i.e., for any $\epsilon>0$, there exists a compact subset $K_\epsilon$ of $\cX\times\cZ$ such that for all $\btheta\in\Theta$, $\mu^*_{\btheta}(K_\epsilon)\geq 1-\epsilon$.\label{ass:tight}
\item Functions in $\cD$ have uniformly bounded function values, gradients and Hessians so that there exists a positive number $B_0<\infty$ such that $\forall D\in\cD$, $\forall x,z$, we have $|D(x,z)|\leq B_0$, $\|\nabla D(x,z)\|\leq B_0$ and $|tr(\nabla^2 D(x,z))|\leq B_0$.\label{ass:D_bound}
\item $\bar{E}_\phi$, $\nabla G_\theta$, $\nabla E_\phi$ and $\nabla F_\beta$ are uniformly bounded.\label{ass:g}
\item The training set for the discriminator is independent from that for the generative model.\label{ass:indpt_sample}
\end{enumerate}\vspace{-0.1cm}
Then there exists a sequence of $(N,N_s,N_d)\to\infty$ such that 
\begin{equation}\label{eq:unif_grad}
	\sup_{\btheta\in\Theta}\|h_{\hat{D}}(\btheta)-\nabla L(\btheta)\|\pto0,
\end{equation}
 where $\pto$ means converging in probability.
\end{lemma}

Based on the above, we obtain the consistency of DEAR algorithm in the following theorem. It indicates that when the sample sizes grow large enough, with high probability, the DEAR algorithm approximately achieves the minimum of $L(\btheta)$ which leads to the desired disentangled model according to Proposition~\ref{thm:identify}.
%
\begin{theorem}[Consistency]\label{thm:cons}
Suppose the assumptions in Lemma~\ref{lem:grad_conv} hold. Further assume the objective function $L(\btheta)$ in (\ref{eq:obj_new}) is smooth with respect to $\btheta$ and satisfies the Polyak-\L{ojasiewicz} condition in Definition \ref{def:pl}. Let $L^*=\min_{\btheta\in\Theta}L(\btheta)$
Then there exists a sequence of $(N,N_s,N_d)\to\infty$ such that $L(\hat\btheta)\pto L^*$.
\end{theorem}

\noindent
\textbf{Remark}. 
The Polyak-\L{ojasiewicz} (PL) condition \citep{polyak1963gradient} asserts that the suboptimality of a model is upper bounded by the norm of its gradient, which is a weaker condition than assumptions commonly made to ensure convergence, such as (strong) convexity. Recent literature showed that the PL condition holds for many machine learning scenarios including some deep neural networks \citep{charles2018stability,liu2020loss}.

\section{Experiments}\label{sec:exp}

We present the experimental studies in causal controllable generation in Section~\ref{sec:contgen} which demonstrate the effectiveness of DEAR in causal disentanglement and support the theory in Section~\ref{sec:dear_main}. Based on these theoretical and empirical justifications, we then apply the representations learned by DEAR in downstream prediction tasks in Section~\ref{sec:downstream}, and show the benefits of the disentangled causal representations in terms of sample efficiency and distributional robustness. In addition, we investigate the performance of DEAR in learning the causal structure and weighted adjacency of the SCM prior in Section~\ref{sec:exp_learnA}. We also provide ablation studies in terms of varying regularization strength $\lambda$ and various amounts of annotated labels in Section~\ref{sec:ablation}.\footnote{The code and data sets are available at \url{https://github.com/xwshen51/DEAR}.}

We evaluate our methods on two data sets where the ground-truth generative factors are causally related, while most data sets used in previous disentanglement work are assumed or designed to have independent generative factors, for example, in the large scale experimental study by \citet{Locatelloetal19}.  
The first data set that we use is a synthesized data set, Pendulum, similar to the one in \citet{yang2020causalvae}. As shown in Figure~\ref{fig:pend}, each image is generated by four continuous factors: \textit{pendulum\_angle}, \textit{light\_angle}, \textit{shadow\_length} and \textit{shadow\_position} whose underlying structure is given in Figure~\ref{fig:causal_structure}(a) following physical mechanisms. To make the data set realistic, we introduce random noises when generating the two effects from the causes, representing the measurement error. We further introduce 20\% corrupted data whose shadow is randomly generated, mimicking some environmental disturbance. 
The sample sizes for the training, validation and test set are all 6,724.

The second one is a real human face data set, CelebA~\citep{liu2015deep}, with 40 labeled binary attributes. Among them, we consider two groups of causally related factors of interests as shown in Figure~\ref{fig:causal_structure}(b,c). The sample sizes for the training, validation and test set are 162,770, 19,867, and 19,962.  We believe these two data sets are diverse enough to assess our methods because they cover real and synthesized data, with continuous and discrete annotated labels. 
In addition, we test our method on benchmark data sets \citep{gondal2019transfer} where the generative factors are independent. The results are given in Appendix~\ref{app:ind_benchmark}. 
All the details of the experimental setup, network architectures and the synthesized data set are given in Appendix~\ref{app:exp_detail}. Notably, all VAEs and DEAR use the same network architecture for the encoder and decoder (generator). 

\begin{figure}
\centering
\subfigure[Pendulum]{
\includegraphics[width=0.25\textwidth]{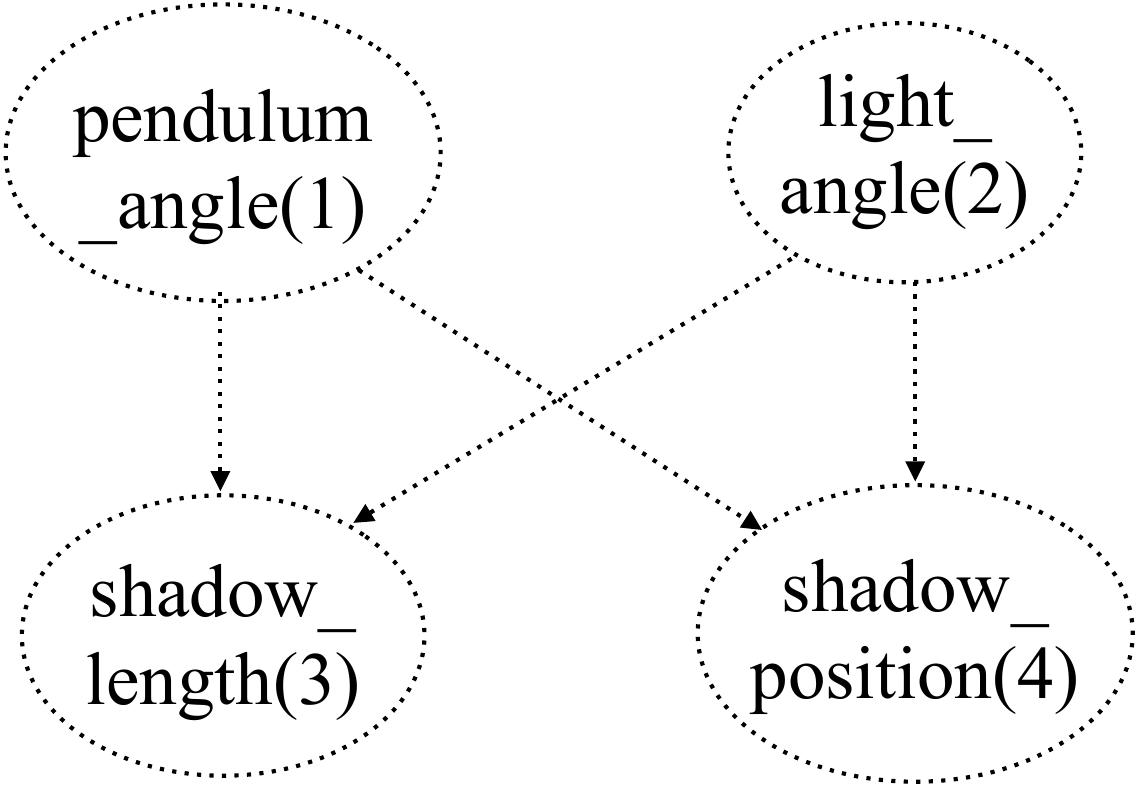}}
\subfigure[CelebA-Smile]{
\includegraphics[width=0.35\textwidth]{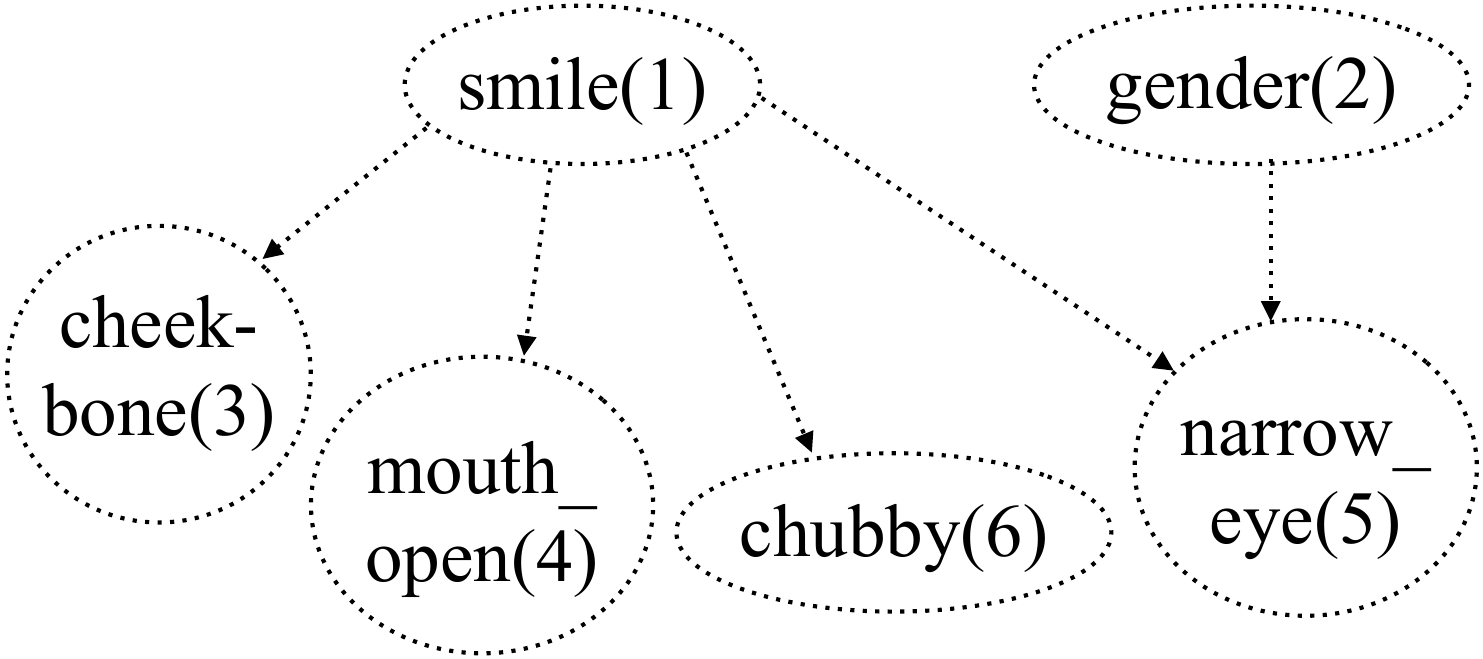}}
\subfigure[CelebA-Attractive]{
\includegraphics[width=0.35\textwidth]{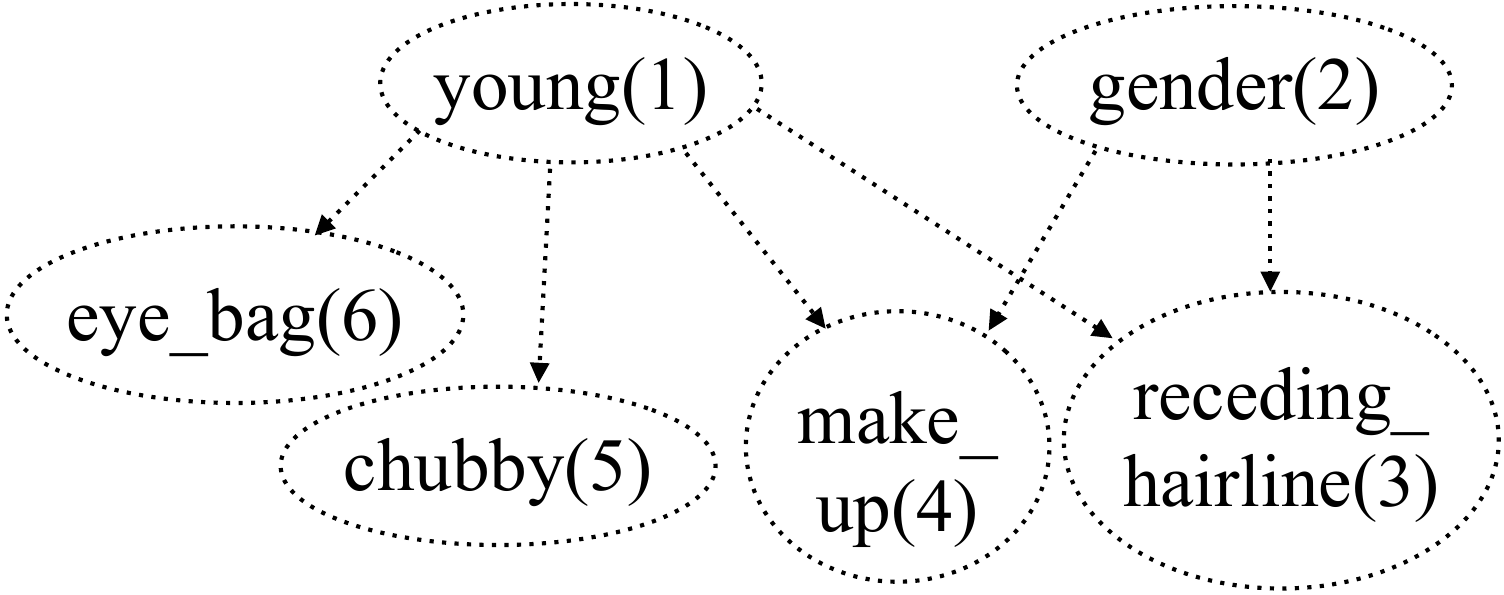}}
\caption{Underlying causal structures.}\label{fig:causal_structure}
\end{figure}

\subsection{Causal controllable generation}\label{sec:contgen}

We first investigate the performance of our methods in disentanglement through applications in causal controllable generation. Traditional controllable generation methods mainly manipulate the independent generative factors~\citep{karras2019style}, while we consider the general case where the factors are causally related. With a learned SCM as the prior, we are able to generate images from many desired interventional distributions of the latent factors. For example, we can manipulate only the cause factor while leaving its effects unchanged. Besides, the bidirectional framework presented in Figure~\ref{fig:model} enables controllable generation either from scratch or a given unlabeled image. 

We consider two types of interventions of most interests in applications. 
First, in traditional traversals, we manipulate one dimension of the latent vector while keeping the others fixed to either their inferred or sampled values~\citep{Higgins2017betaVAELB}. A causal view of such operations is an intervention on all the variables by setting them as constants with only one of them varying. 
Another interesting type of interventional distribution is to intervene on only one latent variable, i.e., $\bbP_{\text{do}([z]_i=c)}(z)$, and to observe how other variables change consequently. The proposed SCM prior enables us to conduct such interventions through the mechanism described in Section~\ref{sec:intervene_mech}. One can naturally generalize it to intervene on more than one variable. For simplicity, we only present the results of intervening on one variable in the paper. 


Figure~\ref{fig:pend}-\ref{fig:smile} illustrate the results of causal controllable generation of the proposed DEAR method and the baseline method with independent priors, S-$\beta$-VAE~\citep{locatello2019disentangling}. Results from other baselines are given in Appendix~\ref{app:more}, including S-TCVAE, S-FactorVAE which essentially make no difference due to the independence assumption, and the unidirectional generative model CausalGAN. 
\revise{In addition, we extend GraphVAE~\citep{he2018graphvae} to a supervised version, named S-GraphVAE by adding the supervised loss in the same way as DEAR and assuming the super-graph of the true graph is known a priori. However, in contrast to the composite prior in DEAR, GraphVAE assigns an SCM over the whole latent space and hence only allows a sufficiently low dimensional latent space. This makes the GraphVAE model less expressive and difficult to be applied to complex data sets with a large number of generative factors like CelebA. The qualitative results of S-GraphVAE in controllable generation are given in Appendix~\ref{app:more}.} 
Note that we do not compare with unsupervised disentanglement methods (e.g., unsupervised $\beta$-VAE, GraphVAE, etc.) because of fairness and their lack of justification.

\begin{figure}
\centering
 \vskip -0.1in
\includegraphics[width=\textwidth]{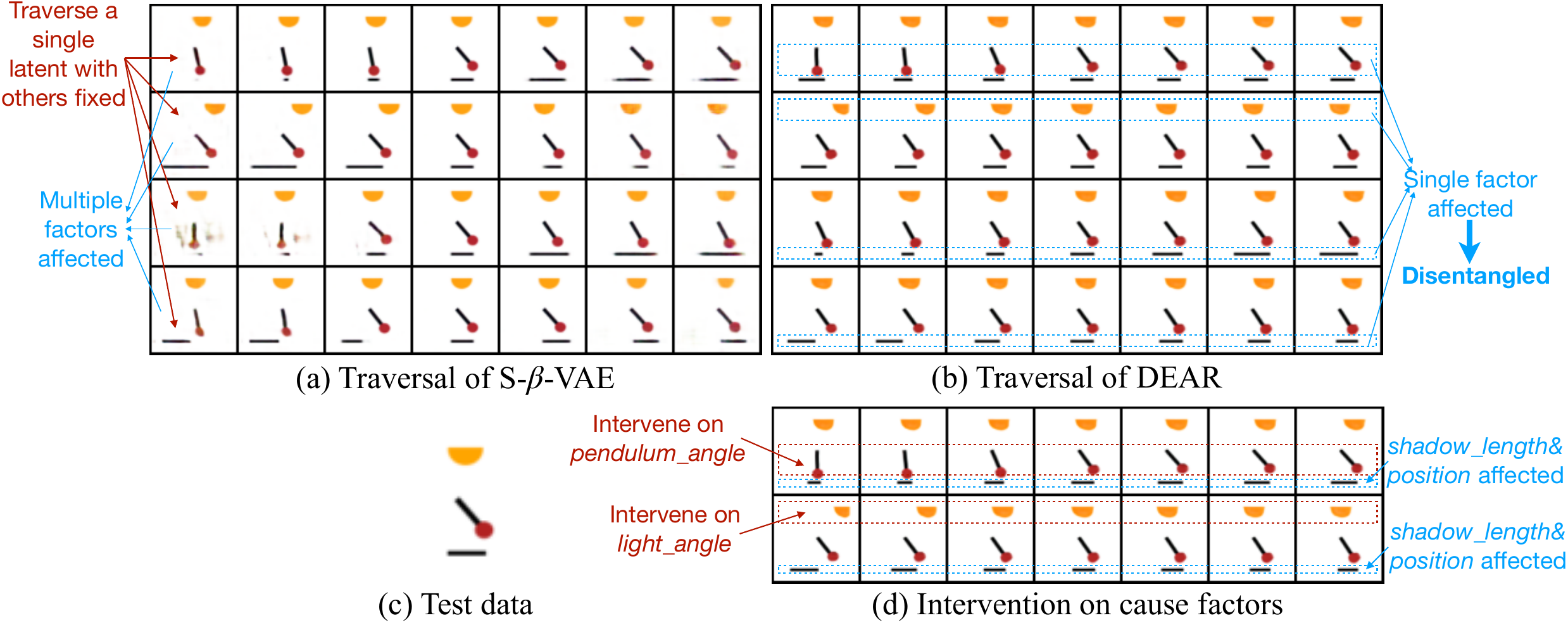}
\vskip -0.1in
\caption{\small Results in causal controllable generation on Pendulum. For example, in line 1 of (a,b) when changing the first dimension $[z]_1$ of $z$ which is supervised with the annotated label of \textit{pendulum\_angle} while keeping the others fixed, we see that the traversals of DEAR vary only in \textit{pendulum\_angle} (disentanglement), while those of S-$\beta$-VAE vary in both \textit{pendulum\_angle} and \textit{shadow\_length} (entanglement); in line 3 when changing $[z]_3$ with the others fixed, only \textit{shadow\_length} is affected with DEAR but both \textit{shadow\_length} and \textit{pendulum\_angle} are affected with S-$\beta$-VAE. In line 1 of (d) we see the intervening on \textit{pendulum\_angle} affects its effects \textit{shadow\_length} and \textit{shadow\_position}, which is consistent with the desired interventional distribution.}
\label{fig:pend}
\vskip -0.1in
\end{figure}

\begin{figure}
\centering
\vskip -0.1in
\includegraphics[width=\textwidth]{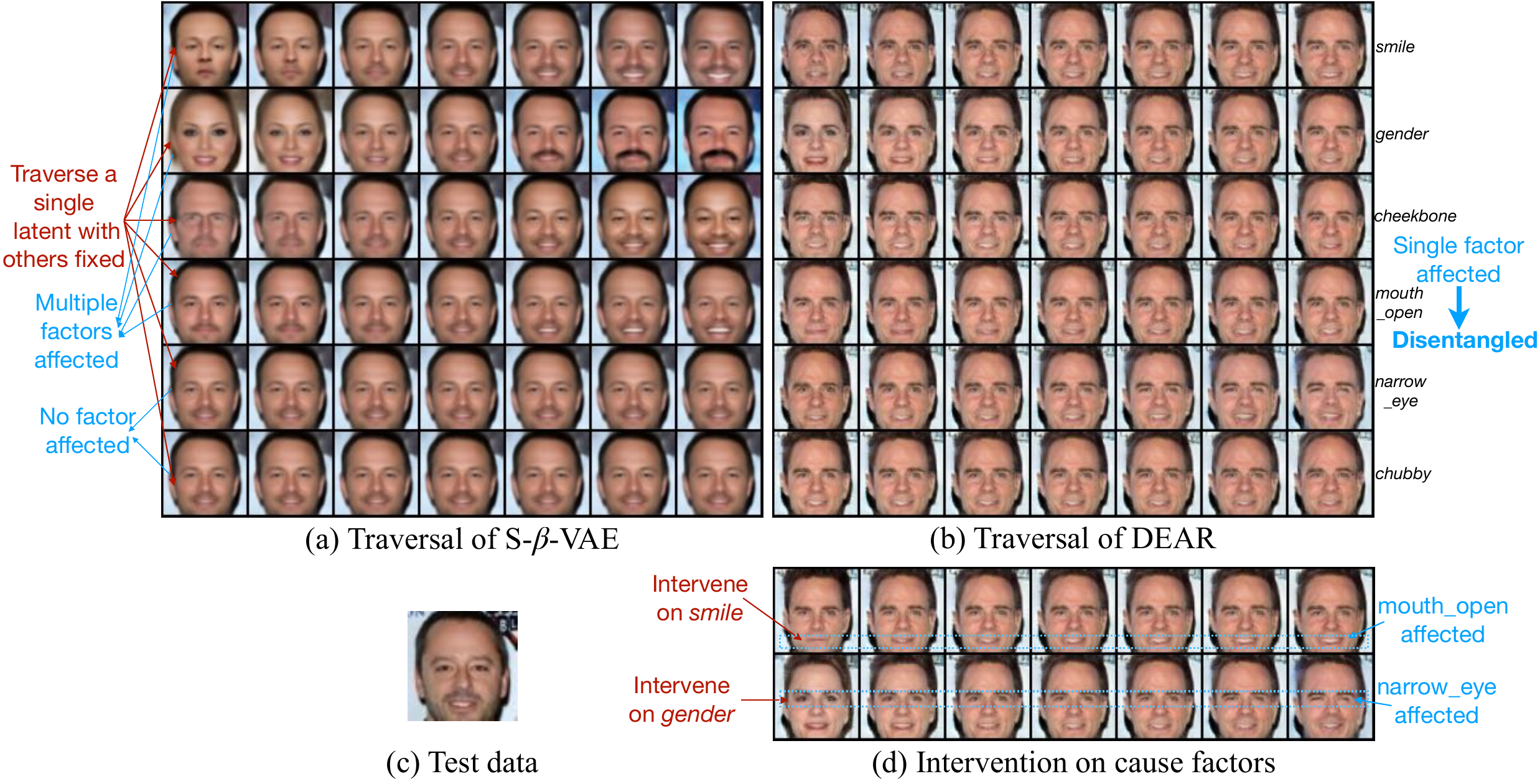}
\vskip -0.1in
\caption{\small Results in causal controllable generation on CelebA. For example, in line 1 of (a,b) when altering $[z]_1$ with the others fixed, we see that the traversals of DEAR vary only in a single factor \textit{smile} with factor \textit{mouth\_open} unaffected, while S-$\beta$-VAE entangles the two factors. In line 5-6 of (a), when changing $[z]_5$ and $[z]_6$ which are supervised with \textit{narrow\_eye} and \textit{chubby}, no factors seem to be affected, indicating that the S-$\beta$-VAE fails to learn the representations of some factors. In line 1 of (d) we see that intervening on \textit{smile} affects its effect \textit{mouth\_open}, which makes sense.}
\label{fig:smile}
\vskip -0.1in
\end{figure}

In each figure, we first infer the latent representations from a test image in block (c). The traditional traversals of the two models are given in blocks (a,b). We see that in each line when manipulating one latent dimension while keeping the others fixed, the generated images of our model vary only in a single factor, indicating that our method can disentangle the causally related factors, while those of S-$\beta$-VAE show multiple factors affected. It is worth pointing out that we are the first to achieve the disentanglement between a cause factor and its effects, while other methods tend to entangle them. One typical example is the disentanglement between \textit{smile} and its effect \textit{mouth\_open} as shown in Figure~\ref{fig:smile}. 
In block (d), we show the results of intervention on the latent variables representing the cause factors, which clearly show that intervening on a cause variable changes its effect variables. Results in Appendix~\ref{app:more} further show that intervening on an effect variable does not influence its cause. Specific examples are given in the captions. Note that without an SCM prior, S-$\beta$-VAE cannot generate data from general interventional distributions. 
More qualitative traversals from DEAR are given in Appendix~\ref{app:more}.


\subsection{Downstream task}\label{sec:downstream}

The previous section verifies the good disentanglement performance of DEAR. In this section, equipped with DEAR, we investigate and demonstrate the benefits of the learned disentangled causal representations for downstream tasks in terms of sample efficiency and distributional robustness. \revise{In Appendix \ref{app:disen_metric}, we propose a quantitative metric for causal disentanglement which is utilized to provide some justifications on the relationship between causal disentanglement and performance in downstream tasks.}

We now introduce the downstream prediction tasks. On CelebA, we consider the structure CelebA-Attractive in Figure~\ref{fig:causal_structure}(c). We artificially create a target label $\tau=1$ if \textit{young}=1, \textit{gender}=0, \textit{receding\_hairline}=0, \textit{make\_up}=1, \textit{chubby}=0, \textit{eye\_bag}=0, and $\tau=0$ otherwise, indicating one kind of attractiveness as a slim young woman with makeup and thick hair.\footnote{Note that the definition of attractiveness here only refers to one kind of attractiveness, which has nothing to do with its linguistic definition.} On the pendulum data set, we regard the label of data corruption as the target $\tau$, that is, $\tau=1$ if the data is corrupted and $\tau=0$ otherwise.
We consider the downstream tasks of predicting the target label. In both cases, the generative factors of interests in Figure~\ref{fig:causal_structure}(a,c) are causally related to $\tau$, which are the features that humans would use to do the task. Hence it is conjectured that a disentangled representation of these causal factors tends to be more data-efficient and invariant to distribution shifts.

\subsubsection{Sample efficiency}\label{sec:eff}


For a BGM including the earlier state-of-the-art supervised disentanglement methods S-VAEs~\citep{locatello2019disentangling}, the modified S-GraphVAE~\citep{he2018graphvae}, and our proposed DEAR, we use the learned encoder to embed the training data to the latent space and train an MLP classifier on top of the representations to predict the target label. All the architectures are the same for various methods with details given in Appendix~\ref{app:exp_detail}. Without an encoder, one normally needs to train a convolutional neural network with raw images as the input. Here we adopt the ResNet50 (named ResNet in Table~\ref{tab:sample_eff}) as the baseline classifier which is the architecture of the BGM encoder. Since the disentanglement methods use additional supervision of the generative factors, we consider another baseline ResNet50 (named ResNet-pretrain) that is pretrained using multi-label classification to predict the factors on the same training set. 
\revise{Unless indicated otherwise, DEAR, S-VAEs, S-GraphVAE, and ResNet-pretrain have access to the annotated labels for all training samples, and DEAR and S-GraphVAE are given the true graph structure. We provide the detailed results when there is less supervised information on labels and the graph structure in Sections~\ref{sec:ablation} and \ref{sec:exp_learnA}.}

To measure the sample efficiency, we use the statistical efficiency score defined as the average test accuracy based on 100 samples divided by the average accuracy based on 10,000/all samples, following \citet{Locatelloetal19}. Note that this metric may be misleading when a method always achieves poor accuracy with small and large training samples. Therefore, we also report the test accuracies with different training sample sizes to provide a comprehensive evaluation. 

Table~\ref{tab:sample_eff} presents the results, showing that DEAR owns the highest sample efficiency and test accuracy on both data sets. 
ResNet with raw data inputs has the lowest efficiency, although multi-label pretraining improves its performance to a limited extent. 
S-VAEs have better efficiency than the ResNet baselines but lower accuracy under the case with more training data. Since the encoders of all S-VAEs and DEAR share the same architecture, we explain the inferior performance of S-VAEs is mainly because the independent prior contradicts with the supervised loss as indicated in Proposition~\ref{prop}, making the learned representations entangled (as shown in the previous section) and less informative. 
\revise{On the Pendulum data with few underlying factors, S-GraphVAE outperforms the S-VAEs when training on a smaller sample, indicating that an SCM latent structure has advantages over the independent structure under the VAE framework. Nevertheless, even with the same amount of supervision (on both annotated labels and the same given graph structure), S-GraphVAE is still inferior to DEAR, potentially due to our better causal modeling and optimization based on a GAN algorithm. On the more complex data set CelebA, S-GraphVAE gives very poor performance, even worse than S-VAEs and ResNet.} 


In addition, we investigate the performance of DEAR under the semi-supervised setting where only 10\% of the labels are available. We find that DEAR with fewer labels has comparable sample efficiency with that in the fully supervised setting, with a sacrifice in the accuracy that is yet still comparable to other baselines which use much more supervision. 
\revise{In Section~\ref{sec:ablation}, we provide ablation studies to show how DEAR behaves in terms of varying amounts of labeled samples and different choices of the regularization strength $\lambda$.} 

\revise{We also study knowing less prior information on the causal graph structure. In the last two lines of Table~\ref{tab:sample_eff}, DEAR-SG stands for the DEAR-LIN model trained with a given super-graph (which is not a full graph) of the true graph and DEAR-O stands for the DEAR-LIN model trained with a known causal ordering. We see that DEAR-SG leads to comparable performance as DEAR with the known graph structure, while DEAR-O is slightly worse but still competitive compared with other baseline methods. 
As we will show later, on Pendulum, DEAR-O can recover the true structure and the performance in downstream tasks is identical to that of DEAR given the true structure, so we skip showing the last two lines in Table~\ref{tab:sample_eff}(b).
In Section~\ref{sec:exp_learnA}, we investigate the performance in learning the SCM and in particular, the causal structure, given various amounts of prior information about the true graph, where more insights are given to explain the comparable performance of DEAR-SG in downstream tasks.}

\begin{table}
\small
\centering
\subtable[CelebA]{
\begin{tabular}{@{}cccc@{}}
\toprule
\bf Method &\bf  100(\%) &\bf  10,000(\%) & \bf Eff(\%)\\\midrule
ResNet & 68.06\tiny{$\pm$0.19}&	79.51\tiny{$\pm$0.31}&	85.59\tiny{$\pm$0.27}\\
ResNet-pretrain& 76.84\tiny{$\pm$2.08}&	83.75\tiny{$\pm$0.93}&	91.74\tiny{$\pm$1.98}\\
S-VAE & 77.07\tiny{$\pm$1.42}&	79.87\tiny{$\pm$1.67}&	96.49\tiny{$\pm$1.68} \\
S-$\beta$-VAE & 71.78\tiny{$\pm$1.99}&	76.63\tiny{$\pm$0.24}&	93.67\tiny{$\pm$2.41}  \\
S-TCVAE & 77.10\tiny{$\pm$2.08}&	81.63\tiny{$\pm$0.20}&	94.45\tiny{$\pm$2.72}  \\
S-GraphVAE & 67.87\tiny{$\pm$1.19} &	72.09\tiny{$\pm$0.51} &	 94.14\tiny{$\pm$1.14}  \\
DEAR-LIN & 83.51\tiny{$\pm$0.77}&	84.92\tiny{$\pm$0.11}&	{98.34}\tiny{$\pm$0.81} \\
DEAR-NL & \textbf{84.44}\tiny{$\pm$0.48}&	\textbf{85.10}\tiny{$\pm$0.09}&	\textbf{99.23}\tiny{$\pm$0.51} \\\midrule
DEAR-LIN-10\% & 78.09\tiny{$\pm$0.59}&	79.54\tiny{$\pm$0.41}&	98.18\tiny{$\pm$0.49} \\
DEAR-NL-10\% & 80.30\tiny{$\pm$0.24}&	80.87\tiny{$\pm$0.12}&	\textbf{99.29}\tiny{$\pm$0.23} \\
DEAR-SG & 83.69\tiny{$\pm$0.63} & 84.91\tiny{$\pm$0.06} & 98.57\tiny{$\pm$0.67} \\
DEAR-O & 82.84\tiny{$\pm$0.68} & 84.42\tiny{$\pm$0.05} & 98.13\tiny{$\pm$0.79} \\
\bottomrule
\end{tabular}}
\hskip 0.05in
\subtable[Pendulum]{
\begin{tabular}{@{}ccc@{}}
\toprule
\bf  100(\%) &\bf  all(\%) &\bf  Eff(\%)\\\midrule
79.71\tiny{$\pm$0.98}&	90.64\tiny{$\pm$1.57}&	87.97\tiny{$\pm$2.11}\\
79.59\tiny{$\pm$0.93}&	89.16\tiny{$\pm$1.60}&	89.28\tiny{$\pm$0.59} \\
84.16\tiny{$\pm$0.69}&	90.89\tiny{$\pm$0.28}&	92.60\tiny{$\pm$0.49}  \\
79.95\tiny{$\pm$1.65}&	87.87\tiny{$\pm$0.52}&	90.98\tiny{$\pm$1.47
} \\
85.36\tiny{$\pm$1.11}&	90.33\tiny{$\pm$0.33}&	94.51\tiny{$\pm$1.31} \\
86.08\tiny{$\pm$1.61}& 91.90\tiny{$\pm$0.53} & 93.65\tiny{$\pm$1.29} \\
90.21\tiny{$\pm$0.94}&	\textbf{93.31}\tiny{$\pm$0.14}&	{96.68}\tiny{$\pm$0.89} \\
\textbf{90.62}\tiny{$\pm$0.32}&	92.57\tiny{$\pm$0.08}&	\textbf{97.93}\tiny{$\pm$0.29} \\\midrule
88.93\tiny{$\pm$1.40}&	93.18\tiny{$\pm$0.18}&	95.43\tiny{$\pm$1.33}\\
87.65\tiny{$\pm$0.46}&	91.27\tiny{$\pm$0.21}&	96.03\tiny{$\pm$0.29}  \\
\bottomrule
\end{tabular}
}
\caption{Sample efficiency and test accuracy with different training sample sizes. DEAR-LIN and -NL denote the DEAR models with linear and nonlinear $f$ respectively.}\label{tab:sample_eff}
\end{table}

\subsubsection{Distributional robustness}\label{sec:dr}

We manipulate the training data to inject spurious correlations---misleading heuristics that work for most training examples but do not always hold~\citep{sagawa2019distributionally}---between the target label and some spurious attributes. On CelebA, we regard \textit{mouth\_open} as the spurious factor; on Pendulum, we choose \textit{background\_color} $\in$ \{blue($+$), white($-$)\}. 
We manipulate the training data such that the target label is more strongly correlated with the spurious attributes. Specifically, the target label and the spurious attribute of 80\% of the examples are both positive or negative, while those of 20\% examples are opposite. 
For instance, in the manipulated training set, 80\% smiling examples in CelebA have an open mouth; 80\% corrupted examples in Pendulum are masked with a blue background. The test sets however do not have such correlations, that is, around half of the examples in the test sets of both CelebA and Pendulum have consistent target and spurious labels, leading to a distribution shift.

Intuitively these spurious attributes are not causally related to the target label, but normal independent and identically distributed (IID) based methods like empirical risk minimization (ERM) tend to exploit such easily learned spurious correlations in prediction, and hence face performance degradation when such correlation no longer exists during testing. In contrast, causal factors are regarded as invariant and thus more robust under such shifts. 

Previous sections justify both theoretically and empirically that DEAR can learn disentangled causal representations well. We then apply those representations by training a classifier upon them to predict the target label, which is conjectured to be invariant and robust. Baseline methods include ERM, multi-label ERM which is trained to predict both target label and the factors considered in disentanglement in order to have the same amount of supervision, S-VAEs that are shown unable to disentangle well in the causal case, and S-GraphVAE.

\begin{table}
\small
\centering
\subtable[CelebA]{
\begin{tabular}{@{}ccc@{}}
\toprule
\bf Method &\bf  WorstAcc(\%) &\bf  AvgAcc(\%)\\\midrule
ERM &59.12\tiny{$\pm$1.78} & 82.12\tiny{$\pm$0.26} \\
ERM-multilabel &59.17\tiny{$\pm$4.02} & 82.05\tiny{$\pm$0.25} \\
S-VAE & 60.54\tiny{$\pm$3.48}&	79.51\tiny{$\pm$0.58} \\
S-$\beta$-VAE & 63.85\tiny{$\pm$2.09}&	80.82\tiny{$\pm$0.19} \\
S-TCVAE & 64.93\tiny{$\pm$3.30} & 81.58\tiny{$\pm$0.14} \\
S-GraphVAE & 50.51\tiny{$\pm$4.43} &	 76.01\tiny{$\pm$1.73}  \\
DEAR-LIN & 76.05\tiny{$\pm$0.70}&	83.56\tiny{$\pm$0.09} \\
DEAR-NL & \textbf{76.98}\tiny{$\pm$0.66}&	\textbf{83.60}\tiny{$\pm$0.04} \\\midrule
DEAR-LIN-10\% & 71.40\tiny{$\pm$0.47}&	81.04\tiny{$\pm$0.14}  \\
DEAR-NL-10\% & 70.44\tiny{$\pm$1.02}&	81.94\tiny{$\pm$0.31} \\
DEAR-SG & 74.95\tiny{$\pm$1.14} & 83.56\tiny{$\pm$0.25} \\
DEAR-O & 74.00\tiny{$\pm$1.47} & 83.45\tiny{$\pm$0.32} \\
\bottomrule
\end{tabular}}
\subtable[Pendulum]{
\begin{tabular}{@{}cc@{}}
\toprule
 \bf WorstAcc(\%) &\bf  AvgAcc(\%)\\\midrule
 60.48\tiny{$\pm$2.73} & 87.40\tiny{$\pm$0.89}  \\
 61.70\tiny{$\pm$4.02} & 87.20\tiny{$\pm$1.00} \\ 
20.78\tiny{$\pm$4.45}&	84.26\tiny{$\pm$1.31} \\
44.12\tiny{$\pm$9.73}&	86.99\tiny{$\pm$1.78} \\
35.50\tiny{$\pm$5.57}&	86.64\tiny{$\pm$1.15} \\
54.42\tiny{$\pm$4.15}& 87.64\tiny{$\pm$2.06}\\
\textbf{75.60}\tiny{$\pm$0.27}&	\textbf{93.58}\tiny{$\pm$0.03}\\
75.39\tiny{$\pm$2.11}&	93.16\tiny{$\pm$0.04} \\\midrule
74.05\tiny{$\pm$1.56}&	92.63\tiny{$\pm$0.07} \\
73.93\tiny{$\pm$1.98}&	92.72\tiny{$\pm$0.03} \\
\bottomrule
\end{tabular}}
\caption{Distributional robustness. The worst-case and average test accuracy.}\label{tab:dr}
\end{table}

Table~\ref{tab:dr} presents the average and worst-case test accuracy to assess both the overall classification performance and distributional robustness. The worst-case~\citep{sagawa2019distributionally} accuracy refers to the following: we group the test set according to the two binary labels, the target one and the spurious attribute, into four cases and regard the group with the worst accuracy as the worst-case, which usually owns the opposite spurious correlation to the training data. 
It can be seen that the classifiers trained upon DEAR representations significantly outperform the baselines in both metrics. Particularly, when comparing the worst-case accuracy with the average one, we observe a slump from around 80 to around 60 for other methods on CelebA, while DEAR enjoys a much smaller decline. 
\revise{As in sample efficiency, S-GraphVAE suffers from a smaller drop in worst-case accuracy than S-VAEs on Pendulum, but remains inferior to DEAR. On CelebA, S-GraphVAE again shows poor performance. }

\revise{Moreover, with fewer annotated samples (i.e., 10\% of the full sample), DEAR-10\% remains competitive against baseline methods which use even more supervised labels. DEAR-SG (given the super-graph) is slightly better than DEAR-O (given the ordering), both of which are comparable to DEAR given the full structure. 
More ablation studies in terms of the labeled proportion as well as the strength of the supervised regularizer are given in Section~\ref{sec:ablation}.}

\subsection{\revise{Learning of the structure $A$}}\label{sec:exp_learnA}

In this section, we take a closer look into the learned causal structure and weighted adjacency matrix $A$ of the SCM prior given various amounts of prior graph information. 
As mentioned in Section~\ref{sec:learn_a}, the DEAR method requires prior knowledge on the super-graph of the true graph over the underlying factors of interests. The experiments shown in previous sections are all based on the given true binary structure $\mathbf{I}_{A_0}$. Here we investigate the performance in learning the causal structure on knowing various amounts of information about the graph, which ranges from the causal ordering to the true structure. 
Note that the adjacency matrices learned by DEAR-LIN and DEAR-NL are consistent up to some scaling, so in this section we only show the results from DEAR-LIN as a representative. 



Figure~\ref{fig:learned_cause} shows the learned weighted adjacency matrices when the true binary structure is given for the three underlying structures shown in Figure~\ref{fig:causal_structure}. It can be seen that the weights exhibit meaningful signs and scalings that are consistent with common knowledge. For example, the factor \textit{smile} and its effect \textit{mouth\_open} are positively correlated, that is, one is more likely to open mouth when smiling. The corresponding element in the weighted adjacency $A_{14}$ of (b) turns out positive, which makes sense. Also \textit{gender} (the indicator of being male) and its effect \textit{make\_up} are negatively correlated, that is, women tend to make up more often than men. Correspondingly, element $A_{24}$ of (c) turns out negative. 
\begin{figure}
\centering
\subfigure[Pendulum]{
\includegraphics[width=0.22\textwidth]{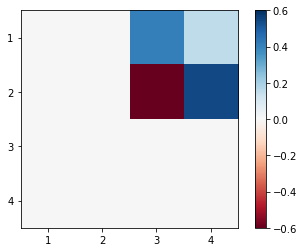}}
\hskip 0.2in
\subfigure[CelebA-Smile]{
\includegraphics[width=0.22\textwidth]{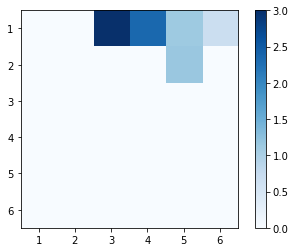}}
\hskip 0.2in
\subfigure[CelebA-Attractive]{
\includegraphics[width=0.22\textwidth]{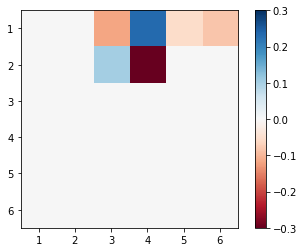}}
\caption{The weighted adjacency matrices learned by DEAR.}
\label{fig:learned_cause}
\end{figure}

\begin{figure}
\centering
\subfigure[Pendulum-O]{
\includegraphics[width=0.25\textwidth]{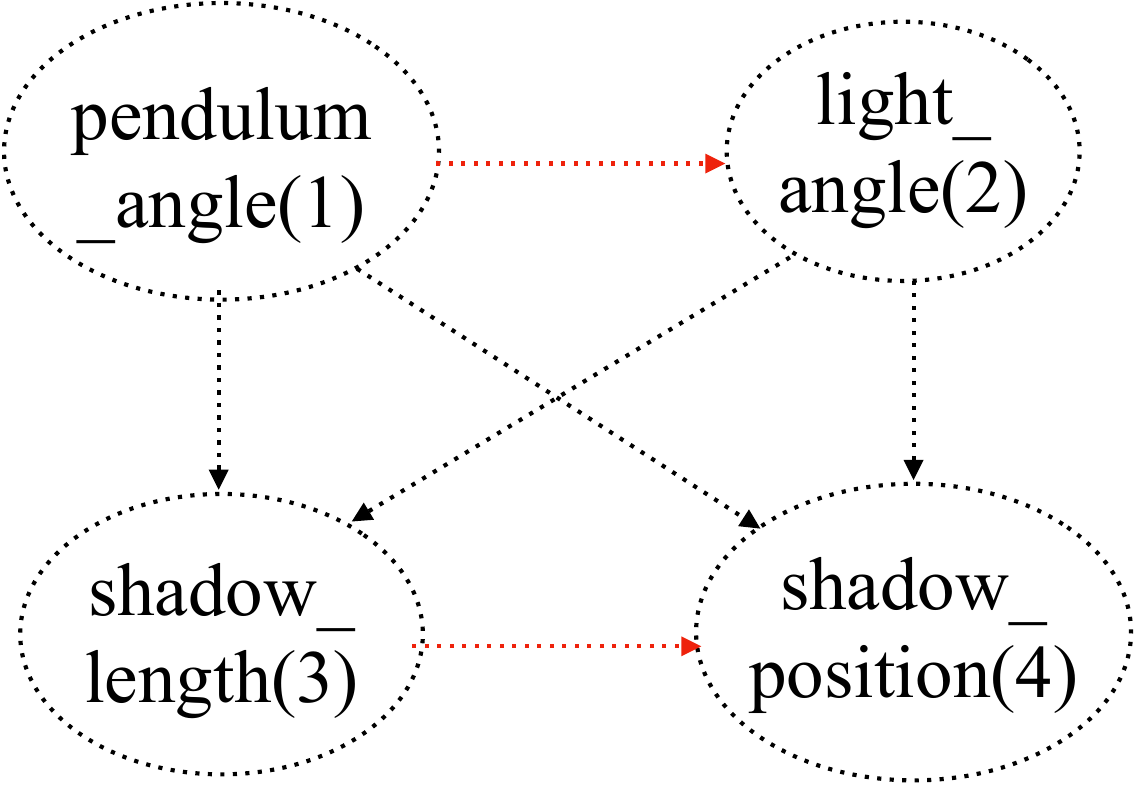}}
\subfigure[CelebA-Attractive-SG]{
\includegraphics[width=0.35\textwidth]{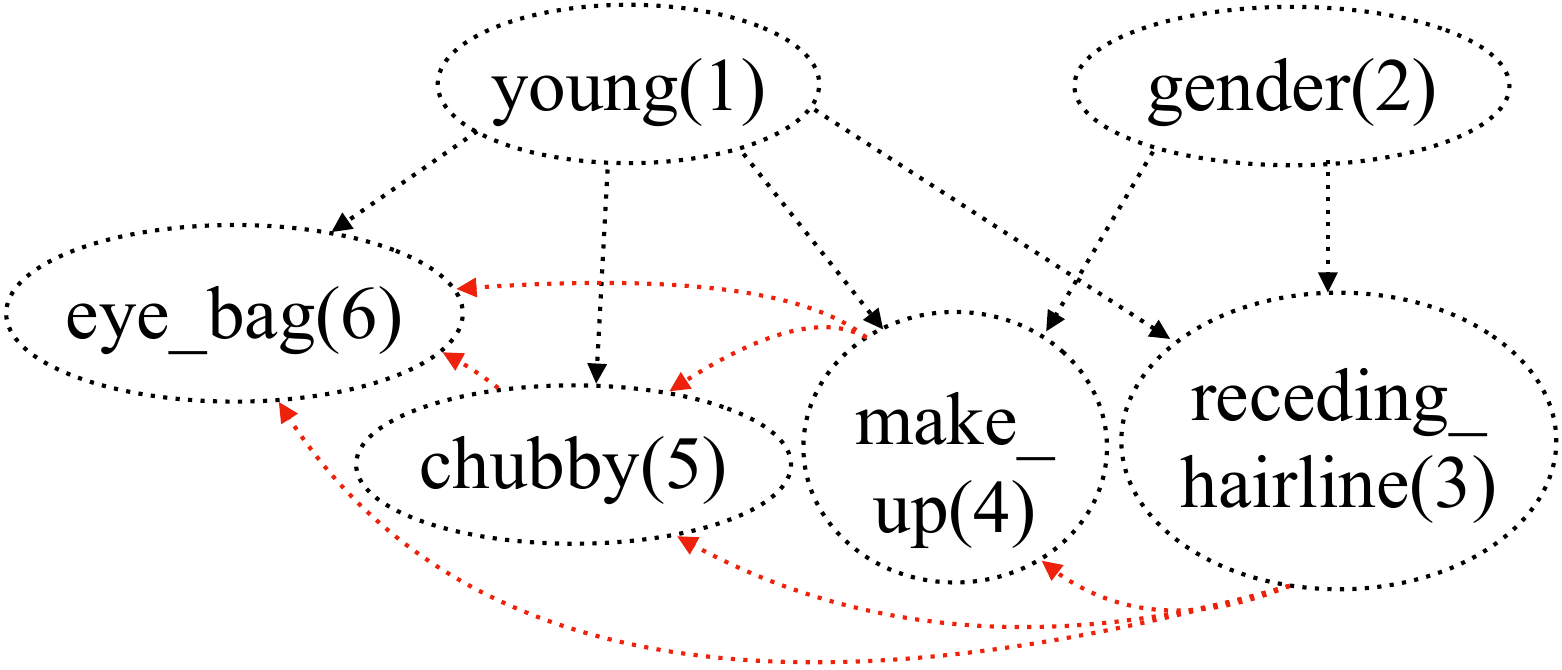}}
\subfigure[CelebA-Attractive-O]{
\includegraphics[width=0.35\textwidth]{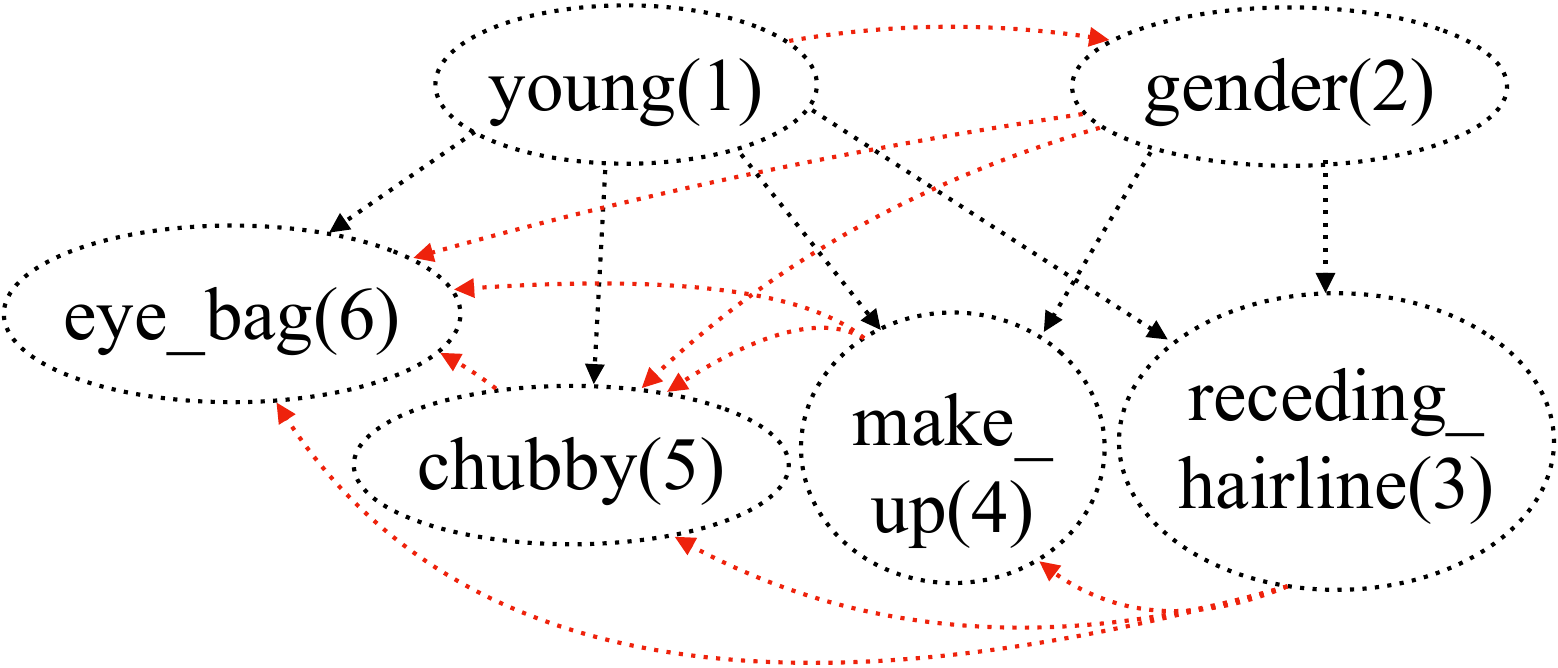}}
\caption{\small The given causal structures. -O and -SG stand for the causal ordering and super-graph. The black edges are true and red edges are in fact redundant. }\label{fig:causal_structure_o}
\end{figure}

Next, we evaluate the performance of DEAR in structure learning with less prior knowledge on the true graph, i.e. knowing a super-graph rather than the exact true graph. 
We first study on the synthetic data set Pendulum whose ground-truth structure is shown in Figure~\ref{fig:causal_structure}(a), where there are fewer causal factors and no hidden confounder. 
Consider the causal ordering \textit{pendulum\_angle}, \textit{light\_angle}, \textit{shadow\_position}, \textit{shadow\_length}, given which we start with a full graph (shown in Figure~\ref{fig:causal_structure_o}(a)) represented by an upper triangular adjacency matrix whose elements are randomly initialized around 0 (shown in Figure~\ref{fig:pend_super}(a)). Figure~\ref{fig:pend_super}(a-d) present the weighted adjacency matrices learned by DEAR at different training epochs. We observe that the weights of the two redundant edges $A_{12}$ and $A_{34}$ vanish gradually and it eventually leads to the weighted adjacency that nearly coincides with the one learned given the true graph shown in Figure~\ref{fig:learned_cause}(a). 
In contrast, Figure~\ref{fig:pend_super}(e) shows the structure learned by S-GraphVAE. Note that GraphVAE learns a binary structure with 0-1 elements and (e) shows the learned probabilities of each element being 1. We see that it learns a redundant edge $A_{12}$ from \textit{pendulum\_angle} to \textit{light\_angle} and misses the edge $A_{23}$ from \textit{light\_angle} to \textit{shadow\_position}. This experiment shows the advantage of DEAR over GraphVAE in learning the latent causal structure. 


\begin{figure}
\centering
\subfigure[Epoch 0]{
\includegraphics[width=0.18\textwidth]{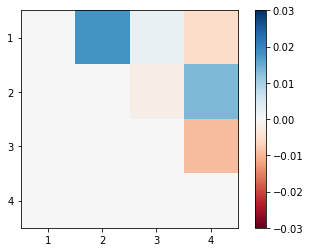}}
\subfigure[Epoch 100]{
\includegraphics[width=0.18\textwidth]{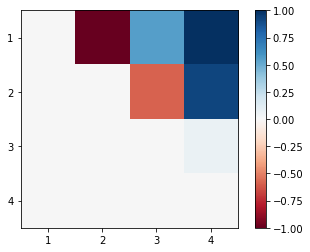}}
\subfigure[Epoch 200]{
\includegraphics[width=0.18\textwidth]{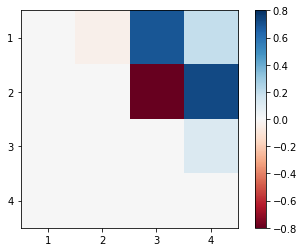}}
\subfigure[Epoch 500]{
\includegraphics[width=0.18\textwidth]{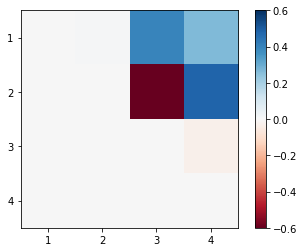}}
\hskip 0.15in
\subfigure[S-GraphVAE]{
\includegraphics[width=0.18\textwidth]{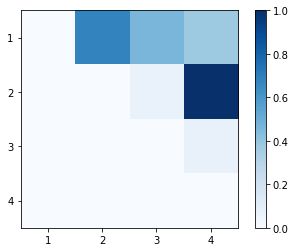}}
\caption{\small Learned weighted adjacency matrices on Pendulum given the causal ordering. (a-d) are the learned matrices from DEAR at different training epochs starting from random initialization around 0, and (e) is the result from S-GraphVAE.}
\label{fig:pend_super}
\end{figure}

\begin{figure}[h]
\centering
\subfigure[Epoch 0]{
\includegraphics[width=0.18\textwidth]{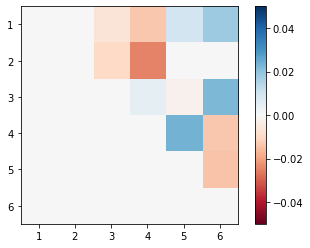}}
\subfigure[Epoch 5]{
\includegraphics[width=0.18\textwidth]{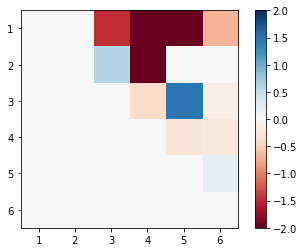}}
\subfigure[Epoch 50]{
\includegraphics[width=0.18\textwidth]{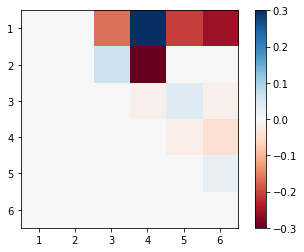}}
\subfigure[Epoch 150]{
\includegraphics[width=0.18\textwidth]{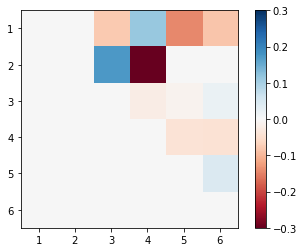}}
\hskip 0.15in
\subfigure[S-GraphVAE]{
\includegraphics[width=0.18\textwidth]{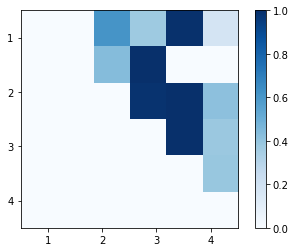}}
\vskip -0.1in
\caption{\small Learned weighted adjacency matrices on CelebA given a super-graph. (a) represents a random initialization around 0 of the weighted adjacency matrix corresponding to the super-graph in Figure~\ref{fig:causal_structure_o}(b); (b-d) are the learned matrices by DEAR at different training epochs; (e) is the result from S-GraphVAE.}
\label{fig:celeba_super}
\end{figure}

\begin{figure}
\centering
\subfigure[Epoch 0]{
\includegraphics[width=0.18\textwidth]{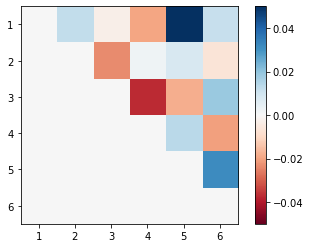}}
\subfigure[Epoch 5]{
\includegraphics[width=0.18\textwidth]{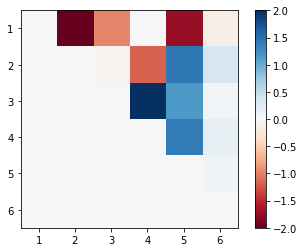}}
\subfigure[Epoch 50]{
\includegraphics[width=0.18\textwidth]{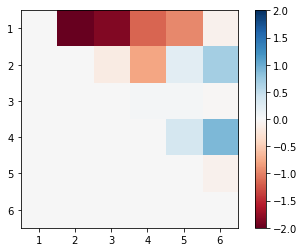}}
\subfigure[Epoch 150]{
\includegraphics[width=0.18\textwidth]{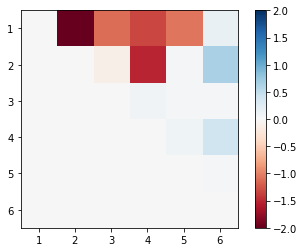}}
\hskip 0.15in
\subfigure[S-GraphVAE]{
\includegraphics[width=0.18\textwidth]{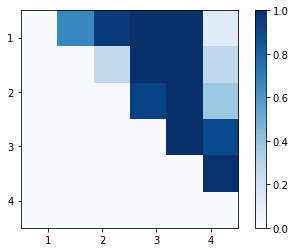}}
\vskip -0.1in
\caption{\small Learned weighted adjacency matrices on CelebA given the causal ordering. (a-d) are the learned matrices by DEAR at different training epochs starting from random initialization around 0; (e) is the result from S-GraphVAE.}
\label{fig:celeba_o}
\end{figure}

The case is more complicated on the real data set CelebA. Although the number of factors of interests, six, is not large, there are much more underlying generative factors. Some of the other factors that we are not interested to disentangle could serve as the hidden confounders of the factors that we are interested in. For example, staying up late may cause a person to have eye bags and look chubby and hence serves as a hidden confounder of the two factors \textit{eye\_bag} and \textit{chubby} in Figure~\ref{fig:causal_structure}(c).
These hidden confounders can be captured in the remaining dimensions of the learned representations through the composite prior introduced in Section~\ref{sec:comp_prior}. However, their existence makes it difficult to identify and learn the structure of the factors of interest. Another complication comes from some biases in the data, potentially caused by selection bias or unknown interventions. Such biases may result in spurious correlations even among the causal variables, bringing trouble to causal structure learning. There are orthogonal works \citep[e.g.,][]{ke2019learning,Bengio2020A,brouillard2020differentiable} focusing on causal discovery under hidden confounders or unknown interventions, which however is beyond the scope of this paper and will be systematically explored in future work. Here we only provide some empirical studies to evaluate our method under this complicated case. 

We conduct two experiments on CelebA. In the first one, we assume knowing a super-graph (Figure~\ref{fig:causal_structure_o}(b)) of the true graph (Figure~\ref{fig:causal_structure}(c)) and randomly initialize its weighted adjacency matrix around 0 as in Figure~\ref{fig:celeba_super}(a). Then Figure~\ref{fig:celeba_super}(a-d) show the weighted adjacency matrices learned by DEAR at different training epochs. Similar to the previous experiment on Pendulum, the weights corresponding to the redundant edges gradually vanish. Eventually, DEAR learns the weighted adjacency matrix that largely agrees with the one learned given the true graph shown in Figure~\ref{fig:learned_cause}(c). After edge pruning, one can essentially recover the true graph structure. This explains why DEAR-SG (the DEAR model given this super-graph) performs competitively with DEAR given the true structure in the downstream tasks in the previous two sections. 
In contrast, the graph learned by GraphVAE shown in Figure~\ref{fig:celeba_super}(e) fails to recover the true structure, although it is given the same known super-graph as DEAR. 

In the second experiment, we only assume knowing the causal ordering which leads to a full graph shown in Figure~\ref{fig:causal_structure_o}(c) with the upper-triangular weighted adjacency matrix randomly initialized in Figure~\ref{fig:celeba_o}(a). We observe that although DEAR can remove most of the redundant edges, it mistakenly learns a large weight on the edge from \textit{young} to \textit{gender}. This may be due to the spurious correlation between the two factors \textit{young} and \textit{gender} potentially caused by the selection bias during data collection. 
In comparison, as shown in Figure~\ref{fig:celeba_o}(e), the graph learned by GraphVAE given the same causal ordering turns out to be farther away from the true graph than DEAR. 
Nevertheless, as discussed in the previous two sections, DEAR-O (the DEAR model given the causal ordering) still achieves reasonably satisfying performance, which indicates the robustness of our DEAR method against the correctness of the learned graph structure. 

In summary, when given the true graph structure, DEAR can learn meaningful weights for each edge. If there is no hidden confounder or spurious correlation among the factors of interests, DEAR can learn the true graph given only the causal ordering. If there exist such biases, DEAR can only recover the true structure given some proper super-graphs and in general cannot learn all edges correctly when only the causal ordering is given. 
In all cases, DEAR outperforms GraphVAE in learning the causal structure.

\subsection{\revise{Ablation study}}\label{sec:ablation}

In this section, we conduct ablation studies to illustrate how DEAR performs when using different choices of the hyperparameter $\lambda$ which determines the weight of the supervised regularizer and varying amounts of labeled samples. 
According to Proposition \ref{thm:identify} and Theorem \ref{thm:cons}, at the population level, i.e., assuming an infinite amount of data, the regularization strength $\lambda$ in the objective (\ref{eq:obj_new}) can be any arbitrary positive value to make the theorems hold. However, in practice with a finite sample, $\lambda$ cannot be arbitrarily small roughly due to the estimation error. Therefore we suggest regarding $\lambda$ as a hyperparameter and investigate its sensitivity across different tasks and data sets. 
Figures \ref{fig:eff_lam}-\ref{fig:dr_lam} plot the metrics in sample efficiency and distributional robustness when using different choices of $\lambda$. We observe that all these results (with $\lambda$ ranging from 0.1 to 10) remain significantly superior to the baseline methods in Tables \ref{tab:sample_eff}-\ref{tab:dr}, which suggests that DEAR can perform reasonably well across a wide range of $\lambda$. As $\lambda$ becomes close to 0, we generally observe a performance decrease.

Next, we study how DEAR, as well as baseline methods, behave as we reduce the number of annotated samples. 
Figures \ref{fig:eff_prop}-\ref{fig:dr_prop} plot the metrics in sample efficiency and distributional robustness when using different amounts of labeled samples. Note that 0.1\% of the CelebA training set corresponds to 162 samples and 1\% of the Pendulum training set corresponds to 67 samples, both of which belong to weakly supervised settings according to \citet{locatello2019disentangling}. Such small numbers of supervised labels belong to weakly supervised settings according to \citet{locatello2019disentangling} and would make manual labeling feasible even if no label is available beforehand. 
Naturally, with fewer labeled samples, all methods basically perform worse. DEAR always outperforms the VAEs. In particular, as shown in Figure \ref{fig:dr_prop}(a), when training with 0.1\%-1\% labels of the CelebA training sample, S-$\beta$-VAE and S-TCVAE completely fail in the worst-case group, meaning that the classifiers trained upon them almost fully rely on the spurious correlation and exhibit no robustness to distribution shifts at all. 
In Figure \ref{fig:eff_prop}(a), when the supervised proportion is lower, although S-$\beta$-VAE and S-TCVAE have higher sample efficiency, they actually perform poorly with both small and large samples, leading to a misleadingly high efficiency score. 

\begin{figure}[H]
\centering
\subfigure[CelebA]{
\includegraphics[width=0.48\textwidth]{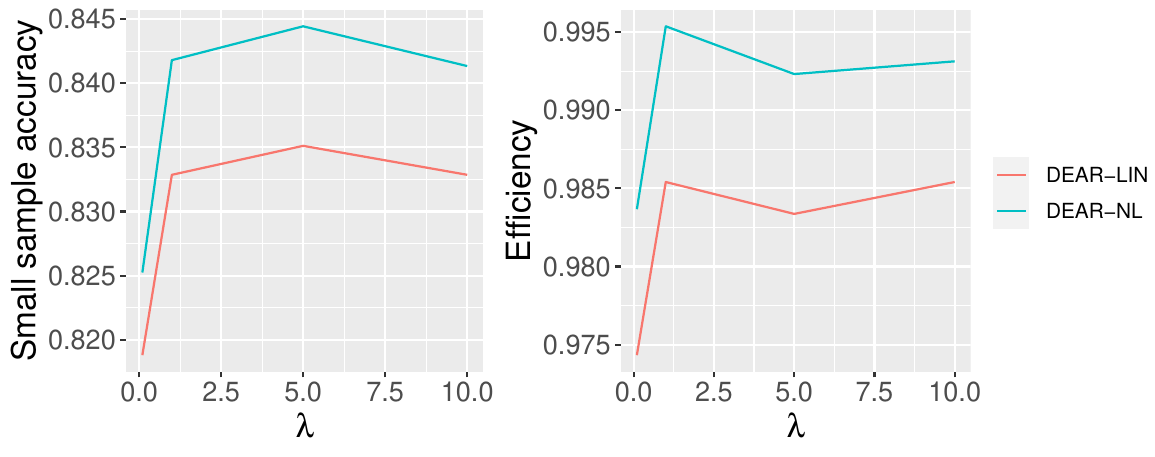}}
\subfigure[Pendulum]{
\includegraphics[width=0.48\textwidth]{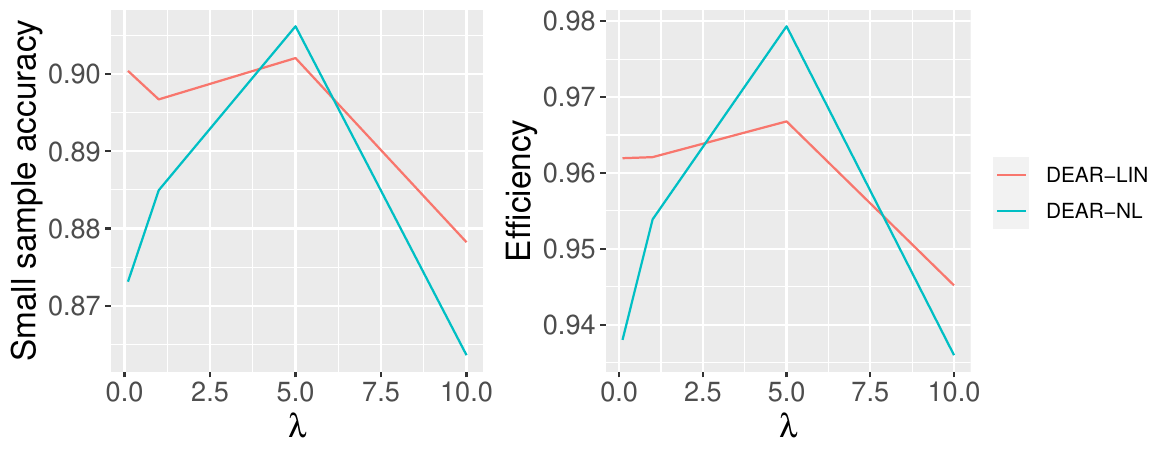}}
\vskip -0.1in
\caption{\small Test accuracy when training on a small sample \& sample efficiency, as defined in Section~\ref{sec:eff}, against four different choices of $\lambda$: 0.1, 1, 5, and 10.}
\label{fig:eff_lam}
\end{figure}

\begin{figure}[h]
\centering
\subfigure[CelebA]{
\includegraphics[width=0.48\textwidth]{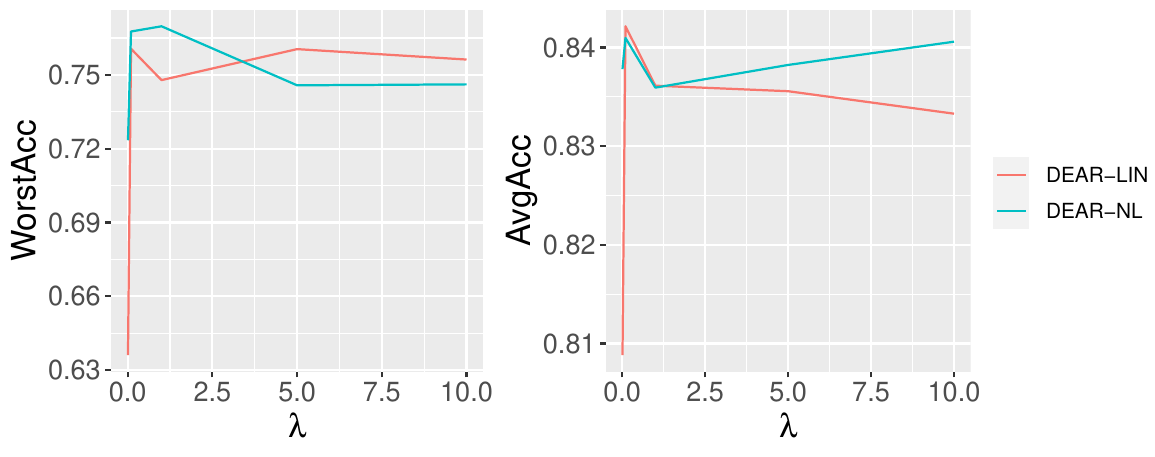}}
\subfigure[Pendulum]{
\includegraphics[width=0.48\textwidth]{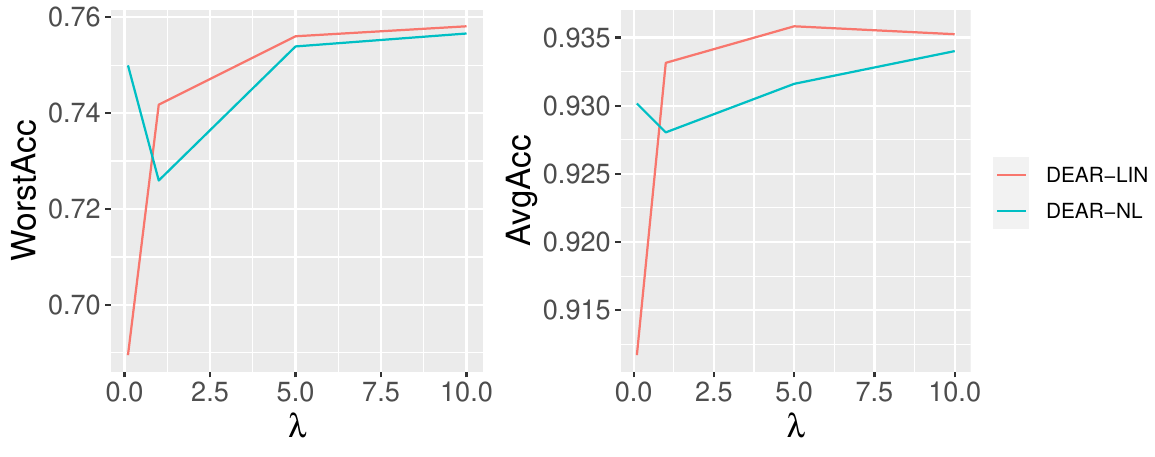}}
\vskip -0.1in
\caption{\small Worst-case and average test accuracy, as defined in Section~\ref{sec:dr}, against different choices of $\lambda$. On Pendulum, we experiment with $\lambda=0.1, 1, 5, 10$; on CelebA, we experiment with $\lambda=0.01,0.1, 1, 5, 10$.}
\label{fig:dr_lam}
\end{figure}

\begin{figure}[h]
\centering
\subfigure[CelebA]{
\includegraphics[width=0.48\textwidth]{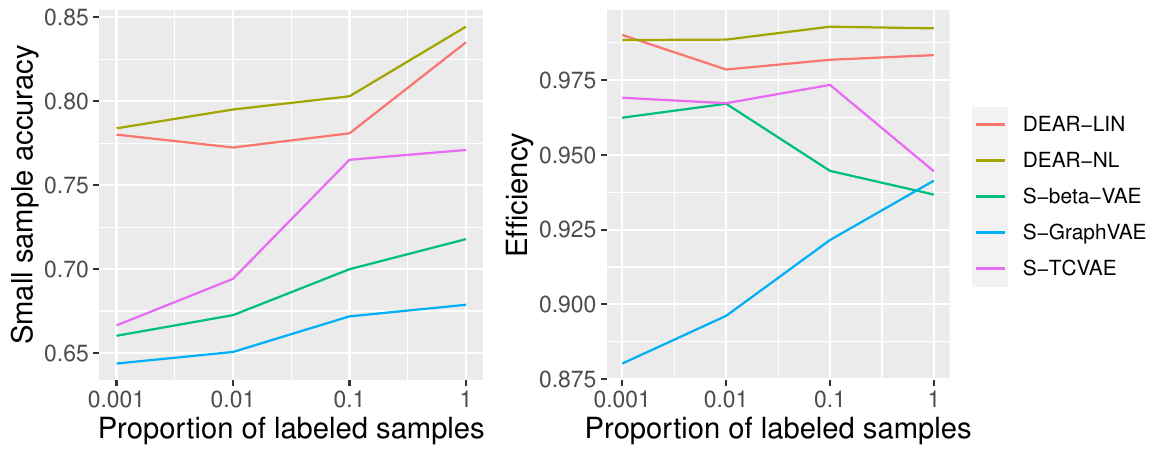}}
\subfigure[Pendulum]{
\includegraphics[width=0.48\textwidth]{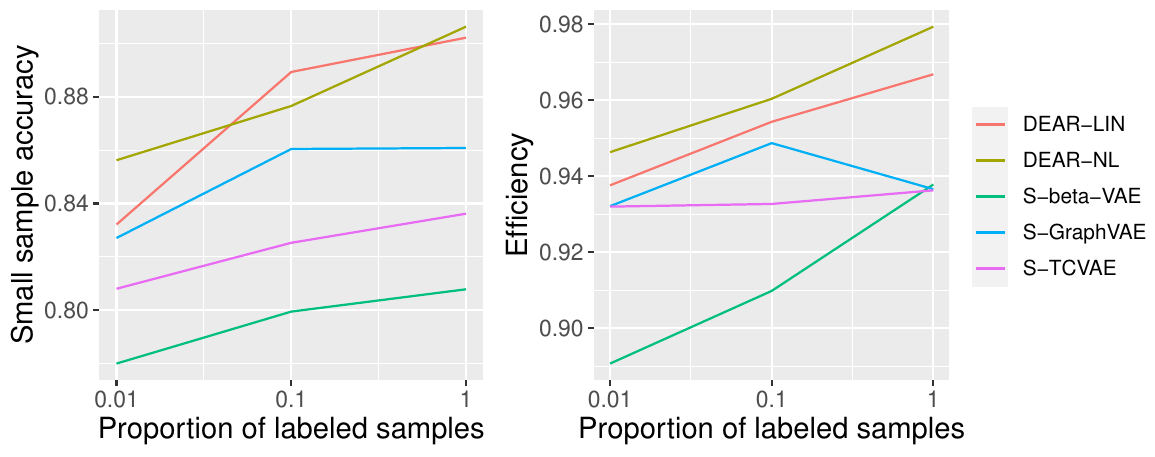}}
\vskip -0.1in
\caption{\small Test accuracy with a small training sample \& sample efficiency against different proportions of labeled samples among full data. On the larger data set CelebA, we consider proportion=0.001, 0.01, 0.1, 1; on the smaller Pendulum data, we consider 0.01, 0.1, 1.}
\label{fig:eff_prop}
\end{figure}

\begin{figure}[h]
\centering
\subfigure[CelebA]{
\includegraphics[width=0.48\textwidth]{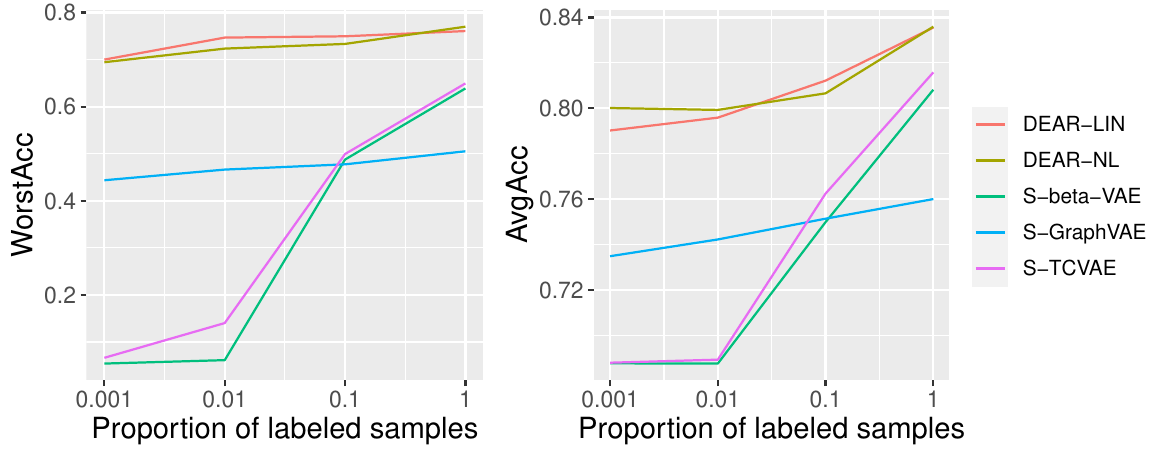}}
\subfigure[Pendulum]{
\includegraphics[width=0.48\textwidth]{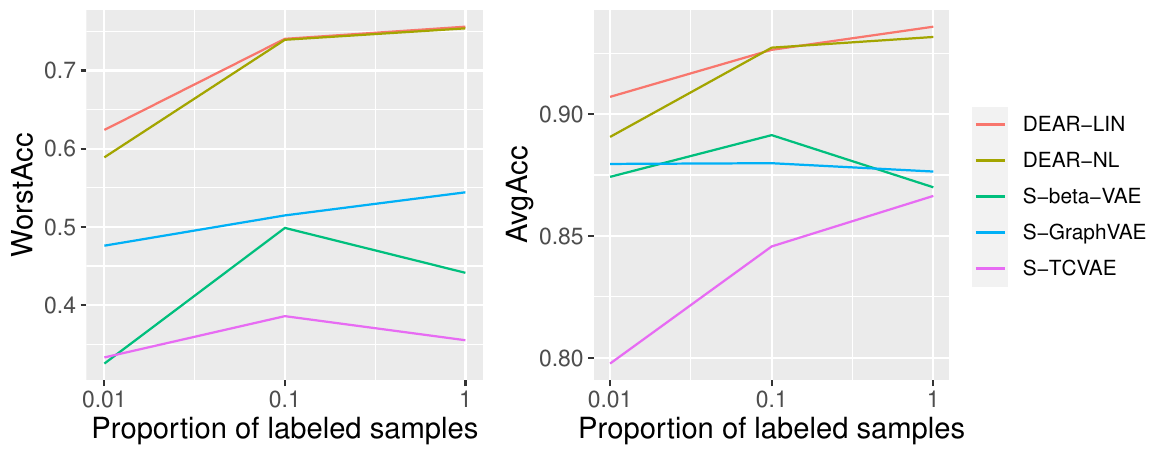}}
\vskip -0.1in
\caption{\small Worst-case and average test accuracy against different proportions of labeled samples among full data.}
\label{fig:dr_prop}
\end{figure}

\section{Conclusion}\label{sec:conclude}

In this paper, we showed that previous methods with the independent latent prior assumption fail to learn disentangled representation when the underlying factors of interests are causally related. We then proposed a new disentangled learning method called DEAR with theoretical guarantees for identifiability and asymptotic consistency. Extensive experiments demonstrated the effectiveness of DEAR in causal controllable generation and structure learning, and the benefits of the learned representations for downstream tasks.


\revise{
Several future directions are worth exploring. Although in our ablation experiments, we demonstrated that DEAR exhibits promising performance in weakly supervised settings in terms of annotated labels and the graph structure, it is worth considering more flexible forms of supervision to make DEAR widely adopted in more real-world applications. On one hand, regarding the annotated labels of the factors of interests, one may consider utilizing other forms of supervision, such as restricted labeling or rank pairing \citep{shu2019weakly}. \rr{Besides, instead of using direct supervision about the true factors, one may consider some additionally observed variables such as class labels or time index \citep{khemakhem2020variational} which serve as auxiliary information to ensure more general identifiability of the true latent factors in the causal case.} 
On the other hand, regarding the graph structure, our experiments in Section~\ref{sec:exp_learnA} indicated the potential of DEAR in latent structure learning. As in many real applications, even the causal ordering may not be available, it is promising to incorporate causal discovery methods in the DEAR framework to learn the latent causal structure from scratch (i.e., without any prior information) with a guarantee of the structure identifiability.}

\revise{
In addition, the proposed method applies to the case where the observational data are IID, as commonly considered in the literature of generative models and disentanglement. It would be interesting to extend the current approach to non-IID settings, in particular, to the scenarios where one can perform interventions during data collection. For example, in reinforcement learning, the interactive environment allows the agent to perform actions and observe their outcomes. The resulting data set that contains a mixture of interventional distributions \citep[e.g.,][]{ke2021systematic} could be leveraged in causal disentanglement learning.
}

\acks{We would like to thank the anonymous reviewers for their valuable comments that were very useful for improving the quality of this work. The work was supported by the General Research Fund (GRF) of Hong Kong (No. 16201320). 
F. Liu's research was supported in part by a Key Research Project of Zhejiang Lab (No. 2022PE0AC04).
}

\newpage
\appendix

\section{Proofs}\label{app:proof}

\subsection{Preliminaries}

This section presents some preliminary notions and lemmas which will be used in proofs.

\begin{definition}[Bracketing covering number~\citep{geer2000empirical}]\label{def:cover_num}
	Consider a function class $\cG=\{g(x)\}$ and a probability measure $\mu$ defined on $\cX$. Given any positive number $\delta>0$. Let $N_{1,B}(\delta,\cG,\mu)$ be the smallest value of $N$ for which there exist pairs of functions $\{[g_j^L,g_j^U]\}_{j=1}^N$ such that $\int |g_j^L(x)-g_j^U(x)|d\mu\leq\delta$ for all $j=1,\dots,N$, and such that for each $g\in\cG$, there is a $j=j(g)\in\{1,\dots,N\}$ such that $g_j^L\leq g\leq g_j^U$. Then $N_{1,B}(\delta,\cG,\mu)$ is called the $\delta$-bracketing covering number of $\cG$.
\end{definition}

\begin{lemma}[Uniform continuous mapping theorem]\label{lem:ucmt}
	Let $X_n$, $X$ be random vectors defined on $\cX$. Let $f:\bbR^d\to\bbR^m$ be uniformly continuous and $T_\theta:\cX\to\bbR^d$ for $\theta\in\Theta$. Suppose $T_\theta(X_n)$ converges uniformly in probability to $T_\theta(X)$ over $\Theta$, i.e., as $n\to\infty$ we have  $\sup_{\theta\in\Theta}\|T_\theta(X_n)-T_\theta(X)\|\pto0$. Then $f(T_\theta(X_n))$ converges uniformly in probability to $f(T_\theta(X))$, i.e., $\sup_{\theta}\|f(T_\theta(X_n))-f(T_\theta(X))\|\pto0$.
\end{lemma}
\begin{proof}
Given any $\epsilon>0$. 
Because $f$ is uniformly continuous, there exists $\delta>0$ such that $\|f(x)-f(y)\|\leq\epsilon$ for all $\|x-y\|\leq\delta$.

We have
\begin{align}
\bbP\Big(\sup_{\theta\in\Theta}\|T_\theta(X_n)-T_\theta(X)\|\leq\delta\Big)&=\bbP\big(\forall\theta\in\Theta:\|T_\theta(X_n)-T_\theta(X)\|\leq\delta\big)\label{eq:ucmt1}\\
&\leq\bbP\big(\forall\theta\in\Theta:\|f(T_\theta(X_n))-f(T_\theta(X))\|\leq\epsilon\big)\nonumber\\
&=\bbP\Big(\sup_{\theta\in\Theta}\|f(T_\theta(X_n))-f(T_\theta(X))\|\leq\epsilon\Big)\label{eq:ucmt2}.
\end{align}
By the uniform convergence of $T_\theta(X_n)$, we know the left-hand side of (\ref{eq:ucmt1}) converges to 1.
Hence (\ref{eq:ucmt2}) goes to 1, which implies the desired result.
\end{proof}

\begin{lemma}\label{lem:conv_density}
	Let $\mu_n$ and $\mu$ be a sequence of measures on probability space $(\cX,\Sigma)$ with densities $p_n(x)$ and $p(x)$. Given any compact subset $K$ of $\cX$. Suppose $p_n$ is uniformly bounded and Lipschitz on $K$ ($*$). If $H^2(\mu_n,\mu)\pto0$, then $\sup_{x\in K}|p_n(x)-p(x)|\pto0$ as $n\to\infty$, where $H(q_1,q_2)=\left(\int(q_1^{1/2}-q_2^{1/2})dxdz/2\right)^{1/2}$ denotes the Hellinger distance between two distributions with densities $q_1$ and $q_2$.
\end{lemma}
\begin{proof}
Note that assumptions in ($*$) satisfy the requirements in the Arzel\`a-Ascoli theorem. Thus, for each subsequence of $p_n$, there is a further subsequence $p_{n_m}$ which converges uniformly on compact set $K$, i.e., for some $p_0$ as $m\to\infty$ we have
\begin{equation*}
	\sup_{x\in K}|p_{n_m}(x)-p_0(x)|\to0.
\end{equation*}

By Scheff\'e's Theorem we have $H(p_{n_m},p_0)\to0$. On the other hand we have $H(p_{n_m},p)\pto0$. By triangle inequality,
\begin{equation*}
	H(p,p_0)\leq H(p_{n_m},p_0)+H(p_{n_m},p)\pto0.
\end{equation*}
Since the inequality holds for all $m$ and the LHS is deterministic, we have $H(p,p_0)=0$, which implies $p=p_0$, a.e. wrt the Lebesgue measure. Hence we have
\begin{equation*}
	\sup_{x\in K}|p_{n_m}(x)-p(x)|\to0,\ a.e.
\end{equation*}
By \citet[Theorem 2.3.2]{durrett2019probability}, we have $\sup_{x\in K}|p_{n}(x)-p(x)|\pto0$ as $n\to\infty$.
\end{proof}

\subsection{Proof of Proposition~\ref{prop}}\label{app:prof_prop}
\begin{proof}
On one hand, by the assumption that  the elements of $\xi$ are connected by a causal graph whose adjacency matrix is not a zero matrix, there exist $i\neq j$ such that $[\xi]_i$ and $[\xi]_j$ are not independent, indicating that the probability density of $\xi$ cannot be factorized. Since $E^*$ is disentangled with respect to $\xi$, by Definition~\ref{def}, $\forall i=1,\dots,m$ there exists $g_i$ such that $[E^*(x)]_i=g_i([\xi]_i)$. This implies that the probability density of $E^*(x)$ is not factorized.

On the other hand, notice that the distribution family of the latent prior is contained in $\{p_z:p_z\text{ is factorized}\}$. Hence the intersection of the marginal distribution families of $z$ and $E^*(x)$ is an empty set. Then the joint distribution families of $(x,E^*(x))$ and $(G(z),z)$ also have an empty intersection. 

We know that $L_{\rm{gen}}(E^*,G)=0$ implies $q_{E^*}(x,z)=p_G(x,z)$ which contradicts the above. Therefore, we have $a=\min_{G}L_{\rm{gen}}(E^*,G)>0$. 

Let $(E',G')$ be the solution of the optimization problem $\min_{\{(E,G):L_{\rm{gen}}=0\}}L_{\rm{sup}}(E)$. From the above we know $E'$ cannot be disentangled with respect to $\xi$. Then we have $L'=L(E',G')=\lambda b$, and $L^*=L(E^*,G)\geq a+\lambda b^*>\lambda b^*$ for any generator $G$. When $b^*\geq b$ we directly have $L'<L^*$. When $b^*< b$ and $\lambda$ is not large enough, i.e., $\lambda<\frac{a}{b-b^*}$, we have $L'<L^*$. 
\end{proof}

\noindent{\revise{\textbf{Discussion on \citet[Proposition 1]{trauble2021disentangled}}}}\medskip

Proposition 1 in \citet{trauble2021disentangled} and our Proposition~\ref{prop} state the same unidentifiability issue from different perspectives. Proposition 1 in \citet{trauble2021disentangled} says that maximum likelihood estimation (MLE) cannot identify the disentangled representation, while our Proposition~\ref{prop} says that the formulation \eqref{eq:obj_new} in our paper cannot identify the disentangled representation. The relationship of the two formulations, MLE and \eqref{eq:obj_new}, is that the first term in \eqref{eq:obj_new} is an upper bound of the negative log-likelihood. Therefore, our Proposition~\ref{prop} is more straightforward in the sense that it directly studies the formulation that is used in disentanglement methods.

\subsection{Proof of Proposition~\ref{thm:identify}}\label{sec:pf_thm}
In this section, we prove a full statement of Proposition~\ref{thm:identify}. Specifically, we add an assumption on structure identifiability and the consequent result in learning the true structure. 
Assumption~\ref{assu_str_iden} states the identifiability of the true causal structure $\mathbf{I}_{A_0}$ of $\xi$, which is applicable given the true causal ordering under the basic Markov and causal minimality conditions~\citep{pearl2014probabilistic,zhang2011intervention}. 
\begin{assu}
For all $\beta=(f,h,A)\in\cB$ with $p_\beta=p_{\beta_0}$, it holds that $\mathbf{I}_{A}=\mathbf{I}_{A_0}$.
\label{assu_str_iden}
\end{assu}

\begin{proposition}[Full statement of Proposition~\ref{thm:identify}]
Assume the infinite capacity of $E$ and $G$. Further under Assumptions~\ref{assu1} and \ref{assu_str_iden}, DEAR formulation (\ref{eq:obj_new}) learns the disentangled encoder $E^*$ and the true causal structure $\mathbf{I}_{A_0}$. 
Specifically, we have $g_i(x)=\sigma^{-1}(x)$ with the CE loss as the supervised regularizer, and $g_i(x)=x$ with the $L_2$ loss.
\end{proposition}

\begin{proof} To simplify the notations in this section, for a vector $x$, let $x_i$ denote the $i$-th element of $x$ instead of $[x]_i$. For a vector function $g(x)$, let $g_i(x)$ denote the $i$-th component function.  

Assume $E$ is deterministic. 

On one hand, for each $i=1,\dots,m$, first consider the cross-entropy loss \
\begin{equation*}
\begin{split}
L_{\text{sup},i}(E)&=\bbE_{(x,y)}[\mathrm{CE}(E_i(x),y_i)]\\
&=-\int q_x(x)p(y_i|x)[y_i\log\sigma(E_i(x))+(1-y_i)\log(1-\sigma(E_i(x)))]dxdy_i,
\end{split}
\end{equation*}
where $p(y_i|x)$ is the probability mass function of the binary label $y_i$ given $x$, characterized by $\bbP(y_i=1|x)=\bbE(y_i|x)$ and $\bbP(y_i=0|x)=1-\bbE(y_i|x)$.
Let
\begin{align*}
\frac{\partial L_{\text{sup},i}}{\partial\sigma(E_i(x))}=\int q_x(x)p(y_i|x)\left(\frac{1}{1-\sigma(E_i)}-y_i\frac{1}{\sigma(E_i)(1-\sigma(E_i))}\right)dxdy_i=0.
\end{align*}
Then we know that $E_i^*(x)=\sigma^{-1}(\bbE(y_i|x))=\sigma^{-1}(\xi_i)$ minimizes $L_{\text{sup},i}$.

Consider the $L_2$ loss 
\begin{equation*}
L_{\text{sup},i}(\phi)=\bbE_{(x,y)}[E_i(x)-y_i]^2=\int q_x(x)p(y_i|x)[E_i(x)-y_i]^2dxdy_i.
\end{equation*}
Let 
\begin{equation*}
\frac{\partial L_{\text{sup},i}}{\partial E_i(x)}=2\int q_x(x)p(y_i|x)(E_i(x)-y_i)dxdy_i=0.
\end{equation*}
Then we know that $E_i^*(x)=\bbE(y_i|x)=\xi_i$ minimizes $L_{\text{sup},i}$ in this case.

On the other hand, by Assumption~\ref{assu1} there exists $\beta_0=(f_0,h_0,A_0)$ such that $p_\xi=p_{\beta_0}$. 
Further due to the infinite capacity of $G$ and Assumption~\ref{assu1}, we have the distribution family of $p_{G,F}(x,z)$ contains $q_{E^*}(x,z)$. Then by minimizing the loss in (\ref{eq:obj_new}) over $G$, we can find $G^*$ and $F^*$ such that $p_{G^*,F^*}(x,z)$ matches $q_{E^*}(x,z)$ and thus $L_{\text{gen}}(E^*,G^*,F^*)$ reaches 0, where $F^*$ corresponds to parameter $\beta^*=(f^*, h^*,A^*)$. 

Note that $p_{G^*,F^*}(x,z)=q_{E^*}(x,z)$ implies that the marginal distributions match, i.e., $p_{F^*}(z)=q_{E^*}(z)$. Generally denote $E_i^*(x)=g_i(\xi_i)$ for $i=1,\dots,m$. Then, for $i=1,\dots,m$, the distributions of $g_i^{-1}(E_i^*(x))=\xi_i$ and $g_i^{-1}(F_i^*(\epsilon))$ are identical. It can be seen that $p_{\beta_0}=p_{\beta^*_0}$ with $\beta^*_0=(g^{-1}\circ f^*,h^*,A^*)$, where $\circ$ denotes elementwise composition. Then according to Assumption~\ref{assu_str_iden}, we have $\mathbf{I}_{A^*}=\mathbf{I}_{A_0}$.

Hence minimizing $L=L_{\text{gen}}+\lambda L_{\text{sup}}$, which is the DEAR formulation (\ref{eq:obj_new}), leads to the solution with $E^*_i(x)=g_i(\xi_i)$ with $g_i(\xi_i)=\sigma^{-1}(\xi_i)$ if CE loss is used, and $g_i(\xi_i)=\xi_i$ if $L_2$ loss is used, and the true binary adjacency matrix $\mathbf{I}_{A_0}$.

For a stochastic encoder, we establish the disentanglement of its deterministic part as above, and follow Definition~\ref{def} to obtain the desired result.
\end{proof}

\subsection{Proof of Lemma~\ref{lem:grad_conv}}

\begin{proof}[Proof of Lemma~\ref{lem:grad_conv}]
The proof proceeds in three steps after an introduction on logistic regression under the scenario of generative models. 

For pair $(x,z)$, let label $w=1$ if $(x,z)\sim q_E$ and $w=0$ if $(x,z)\sim p_G$, which states that $p(x,z|w=1)=q_E(x,z)$ and $p(x,z|w=0)=p_{G,F}(x,z)$. In generative models, the prior is given by $\bbP(w=1)=\bbP(w=0)=1/2.$ Then the marginal distribution of $(x,z)$ is given by $p^*(x,z)=q_E(x,z)/2+p_{G,F}(x,z)/2$ which induces the probability measure $\mu^*$. Note that the analysis below holds for all $\btheta\in\Theta$ in a pointwise manner unless indicated otherwise, so for simplicity of notation, we omit the subscript $\btheta$.  By the Bayes formula we have $\bbP(w=1|x,z)=q_E(x,z)/(q_E(x,z)+p_{G,F}(x,z))$ and $\bbP(w=0|x,z)=p_{G,F}(x,z)/(q_E(x,z)+p_{G,F}(x,z))$ which defines the probability mass function $p^*(w|x,z)$. Let $p^*(x,z,w)=p^*(x,z)p^*(w|x,z)$. 

Recall the definition $D^*(x,z)=\log(q_E(x,z)/p_{G,F}(x,z))$, so we notice $\bbP(w=1|x,z)=1/(1+e^{-D^*(x,z)})$. Consider the family of conditional distributions $\cP=\{P_D(w=1|x,z)=1/(1+e^{-D(x,z)}):D\in\cD\}$. 

Logistic regression maximizes the log-likelihood  
\begin{equation*}
	\bbE_{x,z,w\sim p^*(x,z,w)}[\log p_D(x,z,w)]=\bbE_{p^*(x,z,w)}\log[p^*(x,z)p_D(w|x,z)]
\end{equation*}
over $\cP$ or equivalently over $\cD$.

Given IID samples $(x_i,z_i,w_i),i=1,\dots,N_d$ from $p^*(x,z,w)$, the empirical loss to be minimized is given by
\begin{align*}
	\hat{L}_d(D)&=-\frac{1}{N_d}\sum_{i=1}^{N_d}[\log p_D(w_i|x_i,z_i)+\log p^*(x_i,z_i)]\\
	&=\frac{1}{N_d}\bigg[\sum_{i:w_i=1}\log(1+e^{-D(x_i,z_i)}) + \sum_{i:w_i=0}\log(1+e^{D(x_i,z_i)})\bigg]+c,
\end{align*}
which is equivalent to (\ref{eq:logistic}) up to a constant $c$.
Let $\hat D=\argmin_{D\in\cD} \hat{L}_d(D)$ and $\hat{P}(w=1|x,z)=1/(1+e^{-\hat{D}(x,z)})$.

\bigskip
\noindent
\textbf{Step I} \ \ We now establish the consistency of $\hat{D}(x,z)$ to $D^*(x,z)$ as defined in (\ref{eq:d_conv}) below based on the generalization analysis of maximum likelihood estimation. 

Let the class
\begin{equation*}
	\cG=\left\{g(x,z,w)=\frac{1}{2}\log\frac{p_D(x,z,w)+p^*(x,z,w)}{2p^*(x,z,w)}:D\in\cD\right\}.
\end{equation*}
Note that each element of $\cG$ can be written as
\begin{equation*}
	g(x,z,w)=\frac{1}{2}\log\frac{p_D(w|x,z)+p^*(w|x,z)}{2p^*(w|x,z)}.
\end{equation*}
By the boundedness of $D^*$ in condition~\ref{ass:D_bound}, we know the mass function $p^*(w|x,z)$ is bounded in a closed interval within $(0,1)$, so $g$ is uniformly bounded. Let $g_\infty=\sup_{g\in\cG}|g|$. Then $\bbE_{p^*(x,z,w)}[g_\infty(x,z,w)]<\infty$. Moreover for all $\delta>0$, the compactness of $\cD$ assumed in condition~\ref{ass:d_compact} implies a finite bracketing covering number defined in Definition~\ref{def:cover_num}, i.e., $N_{1,B}(\delta,\cD,\mu^*)<\infty$. Then it follows from \citet[Theorem~4.3]{geer2000empirical} that  
\begin{equation}\label{eq:conv_in_h}
	H(p_{\hat{D}}(x,z,w),p^*(x,z,w))\to0
\end{equation}
almost surely as $N_d\to\infty$, where $H$ denotes the Hellinger distance.

Consider any compact subset $K$ of $\cX\times\cZ$. We know that for all $D\in\cD$, $D(x,z)$ is continuous and thus is bounded and Lipschitz on $K$. Also from the boundedness of $D^*$, we know that $p^*(x,z)$ is bounded away from 0 on $K$. Then it follows from (\ref{eq:conv_in_h}) and Lemma~\ref{lem:conv_density} that
\begin{equation*}
	\sup_{(x,z)\in K}|\hat{P}(w=1|x,z)-\bbP(w=1|x,z)|\pto0.
\end{equation*}
Then by continuous mapping theorem (Lemma~\ref{lem:ucmt}) and noting that $l(p)=\log(p/(1-p))$ is uniformly continuous on a closed interval within $(0,1)$, we have as $N_d\to\infty$
\begin{equation}\label{eq:d_conv}
	\sup_{(x,z)\in K}|\hat{D}(x,z)-D^*(x,z)|\pto0.
\end{equation}

\medskip
\noindent
\textbf{Step II} \ \  We then prove the pointwise consistency of $\nabla \hat{D}(x,z)$ to $\nabla D^*(x,z)$ as defined in (\ref{eq:d_grad_conv}).

Construct an arbitrary probability measure $\mu$ on $\cX\times\cZ$ that satisfies the following (e.g., a Gaussian measure):
\begin{itemize}\vspace{-0.1cm}
\setlength{\itemsep}{0pt}
\item $\mu$ is absolutely continuous with respect to Lebesgue measure with a density $\rho$.
\item $\mu$ is tight, i.e., for any $\epsilon>0$, there is a compact subset $K_\epsilon$ of $\cX\times\cZ$ such that $\mu(K_\epsilon)\geq1-\epsilon$.
\item $\nabla\log\rho$ is uniformly bounded, i.e., there exists $B_1>0$ such that $\|\nabla\log\rho(x,z)\|\leq B_1$ for all $(x,z)\in \cX\times\cZ$.
\item $\rho$ vanishes at infinity rapidly enough, i.e., $\rho(x,z)=o(r^{-d-k})$ with $r=\sqrt{\|x\|^2+\|z\|^2}$. 
\end{itemize}

For a function $u$ that is uniformly bounded on $\cX\times\cZ$, we have from integration by parts and Cauchy-Schwartz inequality that 
\begin{align}
	\int\|\nabla u\|^2d\mu &= -\int u\ tr(\nabla^2u)d\mu-\int u\nabla u^\top \nabla\log\rho d\mu\nonumber\\
	&\leq \sqrt{\int|u|^2d\mu\int [tr(\nabla^2u)]^2d\mu} + \sqrt{\int|u|^2d\mu\int(\nabla u^\top \nabla\log\rho)^2d\mu}.\label{eq:ibp1}
\end{align}

%

Recall from condition~\ref{ass:D_bound} that there exists a positive number $B_0<\infty$ such that $\forall x,z$, $\forall D\in\cD$, we have $|D(x,z)|\leq B_0$, $\|\nabla D(x,z)\|\leq B_0$ and $|tr(\nabla^2 D(x,z))|\leq B_0$.

Given arbitrary $\epsilon>0$, we know from the tightness of $\mu$ that there exits a compact subset $K_\epsilon$ of $\cX\times\cZ$ such that $\mu(K_\epsilon)\geq1-\epsilon$.
Let $B=\max\{B_0,B_1\}$.
Then we have for all $\btheta\in\Theta$ that
\begin{equation}
\begin{split}\label{eq:ibp}
	&\int_{\cX\times\cZ} \|\nabla\hat{D}(x,z)-\nabla D^*(x,z)\|^2d\mu \\
	\leq\ & 2B\sqrt{\int_{\cX\times\cZ}|\hat{D}(x,z)-D^*(x,z)|^2d\mu} + \sqrt{2}B^2\sqrt{\int_{\cX\times\cZ}|\hat{D}(x,z)-D^*(x,z)|^2d\mu}\\
	=\ & (2B+\sqrt{2}B^2)\sqrt{\int_{K_\epsilon}|\hat{D}(x,z)-D^*(x,z)|^2d\mu + \int_{K_\epsilon^c}|\hat{D}(x,z)-D^*(x,z)|^2d\mu}\\
	\leq\ & (2B+\sqrt{2}B^2)\sqrt{\int_{K_\epsilon}|\hat{D}(x,z)-D^*(x,z)|^2d\mu + 2B^2\epsilon}, 
\end{split}
\end{equation}
where $K_\epsilon^c=(\cX\times\cZ)\setminus K_\epsilon$ is the complement, the first inequality is an application of (\ref{eq:ibp1}) and further due to the boundedness of $\nabla\log\rho$ and gradients and Hessians of functions in $\cD$, and the second inequality comes from the boundedness of functions in $\cD$ and the tightness of $\mu$.

By the uniform convergence in (\ref{eq:d_conv}) over $K_\epsilon$, we have for all $(x,z)\in K_\epsilon$, there exist a sequence $a_{N_d}=o_p(1)$ which is free of $(x,z)$ such that $|\hat{D}(x,z)-D(x,z)|^2\leq a_{N_d}$. Then we have
\begin{equation*}
	\int_{K_\epsilon}|\hat{D}(x,z)-D^*(x,z)|^2d\mu \leq \int_{K_\epsilon}a_{N_d}d\mu = a_{N_d}\mu(K_\epsilon)=o_p(1),
\end{equation*}
by noting that $\mu$ is finite. Further by the arbitrariness of $\epsilon$, we let $\epsilon\to0$ and obtain from (\ref{eq:ibp}) that 
\begin{equation*}
	\int_{\cX\times\cZ} \|\nabla\hat{D}(x,z)-\nabla D^*(x,z)\|^2d\mu\pto0.
\end{equation*}
Recall the arbitrariness of $\mu$. For all $(x,z)\in\cX\times\cZ$, construct $\mu$ such that $\rho(x,z)>0$. 
Let $v(x,z)=\|\nabla\hat{D}(x,z)-\nabla D^*(x,z)\|^2\rho(x,z)$.
By the converse of the mean value theorem, if $(x,z)$ is not an extremum of $v$, then there exists a bounded subset $S(x,z)\subseteq\cX\times\cZ$ such that 
\begin{align*}
	\|\nabla\hat{D}(x,z)-\nabla D^*(x,z)\|^2\rho(x,z)&= \frac{1}{\nu(S(x,z))}\int_{S(x,z)}v(x',z')dx'dz' \\
	&\leq \frac{1}{\nu(S(x,z))}\int_{\cX\times\cZ} v(x',z')dx'dz'\pto0
\end{align*}
where $\nu$ denotes the Lebesgue measure. Since $\rho(x,z)>0$, this implies $\|\nabla\hat{D}(x,z)-\nabla D^*(x,z)\|\pto0$ for all non-extrema. By Lipschitz continuity of $v$ on any compact set, we have $\|\nabla\hat{D}(x,z)-\nabla D^*(x,z)\|\pto0$ for all extrema. 

Up to now we have shown that for all $\btheta\in\Theta$ and $(x,z)\in\cX\times\cZ$, we have $\|\nabla\hat{D}(x,z)-\nabla D^*(x,z)\|\pto0$ as $N_d\to\infty$. 
Further from the smoothness in condition~\ref{ass:D_smooth} and the compactness of $\Theta$, we have  $\forall x,z$,  as $N_d\to\infty$
\begin{equation}\label{eq:d_grad_conv}
	\sup_{\btheta\in\Theta}\|\nabla\hat{D}(x,z)-\nabla D^*(x,z)\|\pto0.
\end{equation}

\medskip
\noindent
\textbf{Step III} \ \  Based on the convergence statements established above, we proceed to show the consistency of the approximate gradient $h_{\hat{D}}(\btheta)$ and complete the proof.

%

By condition~\ref{ass:tight}, $\{\mu^*\}$ is uniformly tight. For arbitrary $\epsilon>0$, there exists a compact subset $K_\epsilon$ of $\cX\times\cZ$ such that $\mu^*(K_\epsilon^c)<\epsilon$. Because $\nabla D(x,z)$ is Lipschitz continuous with respect to $(x,z)$ on $K_\epsilon$, we have as  $N_d\to\infty$
\begin{equation}\label{eq:unif_theta_xz}
	\sup_{\btheta\in\Theta,(x,z)\in K_\epsilon}\|\nabla\hat{D}(x,z)-\nabla D^*(x,z)\|\pto0.
\end{equation}

Given any $S_N=\{(x_i,z_i)\sim\mu^*,y_j:i=1,\dots,N,j=1,\dots,N_s\}$ and $\delta>0$. Define events $A_{N_d}=\{\sup_{\btheta}\|h_{\hat{D}}(\btheta)-h_{D^*}(\btheta)\|\leq\delta\}$ and $B_{N,\epsilon}=\{\forall i:(x_i,z_i)\in K_\epsilon\}$. We have from the tightness of $\mu^*$ that $\bbP(B_{N,\epsilon})\geq(1-\epsilon)^N$. We know from (\ref{eq:unif_theta_xz}) and the continuous mapping theorem (Lemma~\ref{lem:ucmt}) that for any $S_N$ and $\epsilon>0$, as $N_d\to\infty$ (free of $S_N$), we have $\bbP(A_{N_d}|B_{N,\epsilon})\to1$. Then as $N_d\to\infty$, we have $$\bbP(A_{N_d})\geq\bbP(A_{N_d}\cap B_{N,\epsilon})=\bbP(A_{N_d}|B_{N,\epsilon})\bbP(B_{N,\epsilon})\geq\bbP(A_{N_d}|B_{N,\epsilon})(1-\epsilon)^N\to (1-\epsilon)^N.$$
By letting $\epsilon\to0$ we have $\bbP(A_{N_d})\to1$ as $N_d\to\infty$. Since $\delta$ is arbitrary, we have that for any $S_N$, $\sup_{\btheta\in\Theta}\|h_{\hat{D}}(\btheta)-h_{D^*}(\btheta)\|\pto0$ as $N_d\to\infty$.

On the other hand, by condition~\ref{ass:g} and the boundedness of $D^*$, and according to the gradient formulas in Lemma~\ref{lem:grad_formula}, it follows from the uniform law of large numbers that $\|h_{D^*}(\btheta)-\nabla L(\btheta)\|\pto0$ uniformly over $\Theta$ as $N,N_s\to\infty$.

By triangle inequality we have 
\begin{equation*}
	\sup_{\btheta\in\Theta}\|h_{\hat D}(\btheta)-\nabla L(\btheta)\|\leq \sup_{\btheta\in\Theta}\|h_{\hat D}(\btheta)-h_{D^*}(\btheta)\| + \sup_{\btheta\in\Theta}\|h_{D^*}(\btheta)-\nabla L(\btheta)\|.
\end{equation*}

Therefore there exists a sequence of $(N,N_s,N_d)\to\infty$ such that
\begin{equation*}
	\sup_{\btheta\in\Theta}\|h_{\hat D}(\btheta)-\nabla L(\btheta)\|\pto0
\end{equation*}
which completes the proof.
\end{proof}

\subsection{Proof of Theorem~\ref{thm:cons}}\label{app:pf_cons}

\begin{proof}[Proof of Theorem~\ref{thm:cons}]
Consider the gradient descent step based on the approximate gradient
\begin{equation*}
	\btheta_t = \btheta_{t-1} - \eta h_{\hat D}(\btheta_{t-1}),
\end{equation*}
where $\eta$ is the learning rate.

Suppose $L(\btheta)$ is $\ell_0$-smooth. Then we have
\begin{equation*}
	L(\btheta_t) \leq L(\btheta_{t-1}) - \eta h_{\hat D}(\btheta_{t-1})^\top\nabla L(\theta_{t-1}) + \frac{\eta^2\ell_0}{2}h_{\hat D}(\btheta_{t-1})^\top h_{\hat D}(\btheta_{t-1}).
\end{equation*}

Let $\hat\epsilon(\btheta)=\nabla L(\btheta)-h_{\hat{D}}(\btheta)$. By Lemma~\ref{lem:grad_conv}, there exists a sequence of $(N,N_s,N_d)\to\infty$ such that $\hat\epsilon=\sup_{\btheta}\|\hat\epsilon(\btheta)\|\pto0$. Then we have
\begin{align*}
	-\eta h_{\hat D}(\btheta_{t-1})^\top\nabla L(\theta_{t-1})&=-\eta h_{\hat D}(\btheta_{t-1})^\top\left(h_{\hat D}(\btheta_{t-1})+\hat\epsilon(\btheta_{t-1})\right)\\
	&\leq \eta \left(-\|h_{\hat D}(\btheta_{t-1})\|^2 +(\|h_{\hat D}(\btheta_{t-1})\|^2+\|\hat\epsilon\|^2)/2\right)\\
	&=-\frac{\eta}{2}\left(\|h_{\hat D}(\btheta_{t-1})\|^2-\hat\epsilon^2\right)\\
	&\leq -\frac{\eta}{4}\|h_{\hat D}(\btheta_{t-1})\|^2,
\end{align*}
under the case where $\|h_{\hat D}(\btheta_{t-1})\|^2\geq 2\hat\epsilon^2$. We note that
\begin{align*}
	L(\btheta_t) &\leq L(\btheta_{t-1}) - \frac{\eta}{4}\|h_{\hat D}(\btheta_{t-1})\|^2 + \frac{\eta^2\ell_0}{2}\|h_{\hat D}(\btheta_{t-1})\|^2\\
	&\leq L(\btheta_{t-1}) - \frac{\eta}{8}\|h_{\hat D}(\btheta_{t-1})\|^2,
\end{align*}
when $\eta<1/4\ell_0$ which can be satisfied with a sufficiently small learning rate.

By summing over $t=1,\dots,T$, we have  
\begin{equation*}
	L(\btheta_T) \leq L(\btheta_0) - 0.125 \eta \sum_{t=1}^T \|h_{\hat D}(\btheta_{t-1})\|^2.
\end{equation*}
Note that $L(\btheta)$ is lower bounded by 0. Then we have $\sum_t \|h_{\hat D}(\btheta_{t-1})\|^2 = O(1)$. Thus there exists $t$ in $\{0,\dots,T\}$ such that $\|h_{\hat D}(\btheta_{t-1})\|^2 = O(1/T)$. 

Otherwise there exists $t$ such that $\|h_{\hat D}(\btheta_{t-1})\|<\sqrt{2}\hat\epsilon=o_p(1)$.

Therefore we have the empirical estimator $\|h_{\hat D}(\hat\btheta)\|\pto0$.

%
By the uniform convergence (\ref{eq:unif_grad}) from Lemma~\ref{lem:grad_conv}, we have $\|\nabla L(\hat\btheta)\|=0$.  
Then by the PL condition, there exists a sequence of $(N,N_s,N_d)\to\infty$ such that
\begin{equation*}
	L(\hat\theta) - L^* \pto 0,
\end{equation*}
which leads to the desired result.
\end{proof}

\subsection{Proof of Lemma~\ref{lem:grad_formula}}

We follow the same proof scheme as in \citet{shen2020bidirectional} where the only difference lies in the gradient with respect to the prior parameter $\beta$. To make this paper self-contained, we restate some proof steps here using our notations.

Let $\|\cdot\|$ denote the vector 2-norm. For a scalar function $h(x,y)$, let $\nabla_x h(x,y)$ denote its gradient with respect to $x$. For a vector function $g(x,y)$, let $\nabla_x g(x,y)$ denote its Jacobi matrix with respect to $x$.  Given a differentiable vector function $g(x):\mathbb{R}^k\to\mathbb{R}^k$, we use $\nabla\cdot g(x)$ to denote its divergence, defined as 
$$\nabla\cdot g(x):=\sum_{j=1}^k\frac{\partial[g(x)]_j}{\partial[x]_j},$$ 
where $[x]_j$ denotes the $j$-th component of $x$. We know that 
\begin{eqnarray*}
\int\nabla\cdot g(x)dx=0
\end{eqnarray*}
for all vector function $g(x)$ such that $g(\infty) = 0$. Given a matrix function $w(x)=(w_1(x),\dots,w_l(x)):\mathbb{R}^k\to\mathbb{R}^{k\times l}$ where each $w_i(x),i=1\dots,l$ is a $k$-dimensional differentiable vector function, its divergence is defined as $\nabla\cdot w(x)=(\nabla\cdot w_1(x),\dots,\nabla\cdot w_l(x))$.

To prove Lemma~\ref{lem:grad_formula}, we need the following lemma which specifies the dynamics of the generator joint distribution $p_g(x,z)$ and the encoder joint distribution $p_e(x,z)$, denoted by $p_\theta(x,z)$ and $p_\phi(x,z)$ here.

\begin{lemma}\label{lemma:lemma}
Using the definitions and notations in Lemma~\ref{lem:grad_formula}, we have 
\begin{align}
\nabla_\theta p_{\theta,\beta}(x,z) &= -\nabla_x p_{\theta,\beta}(x,z)^\top g_\theta(x) - p_{\theta,\beta}(x,z)\nabla\cdot g_\theta(x),\label{eq:grad_pgen}\\
\nabla_\phi q_\phi(x,z) &= -\nabla_z q_\phi(x,z)^\top e_\phi(z) - q_\phi(x,z)\nabla\cdot e_\phi(z),\label{eq:grad_penc}\\
\nabla_\beta p_{\theta,\beta}(x,z) &= \nabla_x p_{\theta,\beta}(x,z)^\top \tilde{f}_\beta(x) - \nabla_z p_{\theta,\beta}(x,z)^\top f_\beta(z) - p_{\theta,\beta}(x,z)\nabla\cdot\binom{\tilde{f}_\beta(x)}{f_\beta(z)},\label{eq:grad_prior}
\end{align}
for all data $x$ and latent variable $z$, where $g_\theta(G_\theta(z,\epsilon))=\nabla_\theta G_\theta(z,\epsilon)$, $e_\phi(E_\phi(x,\epsilon))=\nabla_\phi E_\phi(x,\epsilon)$, $f_\beta(F_\beta(\epsilon))=\nabla_\beta F_\beta(\epsilon)$, and $\tilde{f}_\beta(G(F_\beta(\epsilon)))=\nabla_\beta G(F_\beta(\epsilon))$.
\end{lemma}

\begin{proof}[Proof of Lemma~\ref{lemma:lemma}]
We only prove (\ref{eq:grad_prior}) which is the distinct part from \citet{shen2020bidirectional}.

Let $l$ be the dimension of parameter $\beta$.
To simplify notation, let random vector $Z=F_\beta(\epsilon)$ and $X=G(Z)\in\mathbb{R}^d$ and $Y=(X,Z)\in\mathbb{R}^{d+k}$, and let $p$ be the probability density of $Y$. For each $i=1,\dots,l$, let $\Delta=\delta e_i$ where $e_i$ is a $l$-dimensional unit vector whose $i$-th component is one and all the others are zero, and $\delta$ is a small scalar. Let $Z'=F_{\beta+\delta}(\epsilon)$, $X'=G(Z')$ and $Y'=(X',Z')$ so that $Y'$ is a random variable transformed from $Y$ by $$Y'=Y+ \binom{\tilde{f}_\beta(X)}{f_\beta(Z)} \Delta + o(\delta).$$ Let $p'$ be the probability density of $Y'$. For an arbitrary $y'=(x',z')\in\mathbb{R}^{d+k}$, let $y'=y+\binom{\tilde{f}_\beta(x)}{f_\beta(z)}\Delta+o(\delta)$ and $y=(x,z)$. Then we have
\begin{align*}
p'(y')&=p(y)|\det(dy'/dy)|^{-1}\nonumber\\
&=p(y)|\det(I_d+(\nabla\tilde{f}_\beta(x),\nabla f_\beta(z))^\top\Delta+o(\delta))|^{-1}\\
&=p(y)(1+\Delta^\top\nabla\cdot (\tilde{f}_\beta(x),f_\beta(z))^\top+o(\delta))^{-1}\\ 
&=p(y)(1-\Delta^\top\nabla\cdot (\tilde{f}_\beta(x),f_\beta(z))^\top+o(\delta))\\
&=p(y)-\Delta^\top p(y')\nabla\cdot (\tilde{f}_\beta(x'),f_\beta(z'))^\top +o(\delta)\\
&=p(y')-\Delta^\top (\tilde{f}_\beta(x'),f_\beta(z'))\cdot\nabla_{x'} p(x',z) - \Delta^\top p(y') (\nabla\cdot\tilde{f}_\beta(x'),\nabla\cdot f_\beta(z'))^\top +o(\delta).
\end{align*}
Since $y'$ is arbitrary, above implies that
\begin{align*}
p'(x,z)&= p(x,z)-\Delta^\top (\tilde{f}_\beta(x),f_\beta(z))\cdot(\nabla_x p(x,z),\nabla_z p(x,z))^\top\cdot\nabla_x p(x,z)\\
&- \Delta^\top p(x,z)(\nabla\cdot\tilde{f}_\beta(x'),\nabla\cdot f_\beta(z'))^\top + o(\delta)
\end{align*}
for all $x\in\mathbb{R}^d,z\in\mathbb{R}^k$ and $i=1,\dots,l$, leading to (\ref{eq:grad_prior}) by taking $\delta\to0$, and noting that $p=p_\beta$ and $p'=p_{\beta+\Delta}$. Similarly we can obtain (\ref{eq:grad_pgen}) and (\ref{eq:grad_penc}).
\end{proof}

\begin{proof}[Proof of Lemma~\ref{lem:grad_formula}]
Recall the objective $\dkl(q,p)=\int q(x,z)\log(p(x,z)/q(x,z))dxdz$. Denote its integrand by $\ell(q,p)$. Let $\ell'_2(q,p)=\partial\ell(q,p)/\partial p$. We have
\begin{align*}
\nabla_\beta\ell(q(x,z),p(x,z))&=\ell'_2(q(x,z),p(x,z))\nabla_\beta p_{\theta,\beta}(x,z)
\end{align*}
where $\nabla_\beta p_{\theta,\beta}(x,z)$ is computed in Lemma~\ref{lemma:lemma}.

Besides, we have
\begin{align*}
\nabla_x\cdot[\ell'_2(q,p)p(x,z)\tilde f_\beta(x)]&=\ell'_2(q,p)p(x,z)\nabla\cdot\tilde f_\beta(x)\\
&\quad+\ell'_2(q,p)\nabla_x p(x,z)\cdot\tilde f_\beta(x)\\
&\quad+\nabla_x\ell'_2(q,p) p(x,z)\tilde f_\beta(x),\\
\nabla_z\cdot[\ell'_2(q,p)p(x,z) f_\beta(z)]&=\ell'_2(q,p)p(x,z)\nabla\cdot f_\beta(z)\\
&\quad+\ell'_2(q,p)\nabla p(x,z)\cdot f_\beta(z)\\
&\quad+\nabla\ell'_2(q,p) p(x,z) f_\beta(z).
\end{align*}

Thus,
\begin{equation*}
\nabla_\beta L_{\rm{gen}} = \int\nabla_\beta\ell(q(x,z),p(x,z))dxdz=\int p(x,z)[\nabla_x\ell'_2(q,p)\tilde f_\beta(x)+\nabla_z\ell'_2(q,p)f_\beta(z)]
\end{equation*}
where we have $\nabla_x\ell'_2(q,p)=s(x,z)\nabla_x D^*(x,z)$ and $\nabla_x\ell'_2(q,p)=s(x,z)\nabla_z D^*(x,z)$. 

Hence
\begin{align*}
\nabla_\beta L_{\rm{gen}} &= -\bbE_{(x,z)\sim p(x,z)}\left[s(x,z)(\nabla_x D^*(x,z)^\top 
\tilde f_\beta(x)+\nabla_z D^*(x,z)^\top f_\beta(z))\right]\\
&= -\bbE_{\epsilon}\left[s(x,z)(\nabla_x D^*(x,z)^\top 
\nabla_\beta G(F_\beta(\epsilon))+\nabla_z D^*(x,z)^\top \nabla_\beta F_\beta(\epsilon))|^{x=G(F_\beta(\epsilon))}_{z=F_\beta(\epsilon)}\right].
\end{align*}
where the second equality follows from reparametrization.
\end{proof}

\section{\revise{Causal disentanglement and downstream tasks}}\label{app:disen_metric}

In the main text, we first demonstrate the good performance of DEAR in causal disentanglement through causal controllable generation in Section~\ref{sec:contgen}, and then show the advantages of the DEAR representations in downstream tasks in terms of sample efficiency (Section~\ref{sec:eff}) and distributional robustness (Section~\ref{sec:dr}). In comparison with previous methods, majorly the VAE-based disentanglement methods, we adopt the same network architectures for the encoder and decoder, and use the same amount of annotated labels. In addition, for GraphVAE, we also assume the same prior information on the graph structure as DEAR. Therefore, we conclude that the superior performance of DEAR is due to better modeling.  
To further justify whether such advantages come from the disentanglement of the learned representations, in this section, we propose a metric for causal disentanglement based on the FactorVAE metric, and investigate the correlation between the disentanglement metric and the metrics for downstream tasks. 

\subsection{Metric for causal disentanglement}

Many existing disentanglement papers also propose their metrics for disentanglement, including the $\beta$-VAE metric~\citep{Higgins2017betaVAELB}, the FactorVAE metric~\citep{Kim2018DisentanglingBF}, the Mutual Information Gap (MIG)~\citep{Chen2018IsolatingSO}, the Separated Attribute Predictability (SAP) score~\citep{Kumar2017VariationalIO}, etc. We refer the reader to \citet{Locatelloetal19} for a comprehensive introduction and discussion on these metrics. 

However, all of these metrics only apply to the case where the ground-truth generative factors are mutually independent and do not apply when the factors are correlated. For example, the MIG score measures for each factor the normalized gap in mutual information with the highest and second highest coordinate in $E(x)$. Suppose a factor $\xi_1$ is correlated with $\xi_2$ and a disentangled representation $E(x)$ so that there exists 1-1 functions $g_1$ and $g_2$ such that $E_1(x)=g_1(\xi_1)$ and $E_2(x)=g_2(\xi_2)$. Then the mutual information of $\xi_1$ with (supposedly the highest coordinate) $E_1(x)$ and (supposedly the second highest coordinate) $E_2(x)$ will both be large, and then their difference will be small. As such, a disentangled representation in this case will not correspond to a large MIG score as expected. 

To this end, we propose a metric for causal disentanglement (i.e., disentanglement of causally related ground-truth factors) based on the FactorVAE metric. Suppose there are $m$ generative factors of interest $\xi_1,\dots,\xi_m$ which are causally related following the true SCM $\mathfrak{C}$ which is available. The procedure to compute the metric is presented in Algorithm~\ref{alg:metric}. 
These steps largely follow those of the FactorVAE metric with the distinct parts tailored for causal disentanglement which we explain below the algorithm. 


\vskip 0.1in
\begin{algorithm}[H]
\DontPrintSemicolon
\KwInput{Encoder $E$, meta-parameters $M,N$}
\For{$k=1,\dots,m$}{
\For{$i=1,\dots,M$}{
Fix $\xi_k$ to a randomly sampled value.\\
Randomly sample other factors $\xi_{-k}$ from $\mathfrak{C}$ conditioning on $\xi_k$ for $N$ times. \\
Generate data with the $N$ factors.\\
Obtain their representations using the learned encoder. \\
Normalize each dimension by its empirical standard deviation over the full data (or a large enough random subset).\\
Compute the empirical variance in each dimension of these normalized representations.\\
Take the index of the dimension with the lowest variance.\\
If the index matches $k$, it counts as a correct sample. 
}
Let $C_k$ be the total number of correct samples among the $M$ samples.\\
}
Obtain score $S=\sum_{k=1}^m C_k/(KM)$.\\
\KwReturn{$S$}
\caption{Metric for causal disentanglement}
\label{alg:metric}
\end{algorithm}

\begin{itemize}
\setlength{\itemsep}{2pt}
\setlength{\parskip}{2pt}
\item Line 4: FactorVAE metric samples all factors independently from Uniform distributions, which does not match (and can be far away) from the true distribution of the causal factors. Instead, we sample the factors following the true SCM and hence respect the data distribution. 
\item Lines 10-12: FactorVAE metric uses the error rate of the majority-vote classifier as the metric, because in an unsupervised setting, one does not know which factor each representation captures. In contrast, the weakly-supervised setting can guarantee the alignment between each representation and a particular factor. Thus, we do not need the majority-vote classifier to identify this correspondence. Instead, we directly check whether the dimension with the lowest empirical variance matches the given index $k$.
\end{itemize}

As we notice, this metric is limited in that it not only requires the ground-truth factors of data for sufficient coverage of the data distribution as previous metrics do, but also requires the ground-truth SCM, which only happens in synthetic data. Nevertheless, in this work, we only use such a metric to provide evaluations and justification on the relationship between causal disentanglement and performance in downstream tasks. We leave a widely-applied quantitive metric for causal disentanglement to future work. 

\subsection{Experimental results}

Figure~\ref{fig:disen_down} shows the scatter plots of the metrics that we considered in downstream tasks (Section~\ref{sec:downstream}) and the metric for causal disentanglement (with $M=200$ and $N=50$). Each metric is used to evaluate seven disentanglement models, including S-$\beta$-VAE, S-TCVAE, S-GraphVAE, and multiple DEAR-LIN models with $\lambda=0.1,1,5,10$. All models are trained using fully supervised labels and GraphVAE and DEAR are given the true graph structure. The network architectures for the encoders and decoders are all the same. We observe a positive correlation between causal disentanglement and performance in downstream tasks, which indicates that the learned representations with a higher disentanglement score tend to perform better in terms of sample efficiency and distributional robustness in downstream tasks. In particular, we notice that the small sample accuracy and worst-case accuracy benefit the most from better causal disentanglement for the corresponding fitted lines have the largest scope. 

\begin{figure}[h]
\centering
\subfigure[Sample efficiency]{
\includegraphics[width=0.48\textwidth]{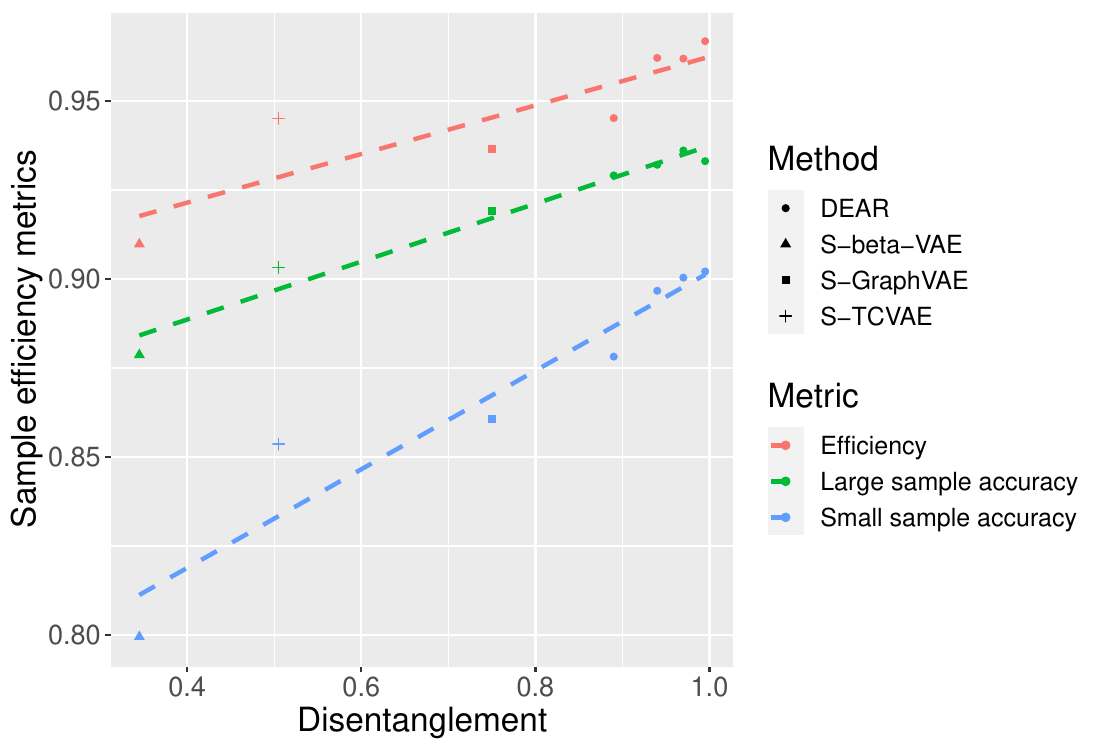}}
\subfigure[Distributional robustness]{
\includegraphics[width=0.48\textwidth]{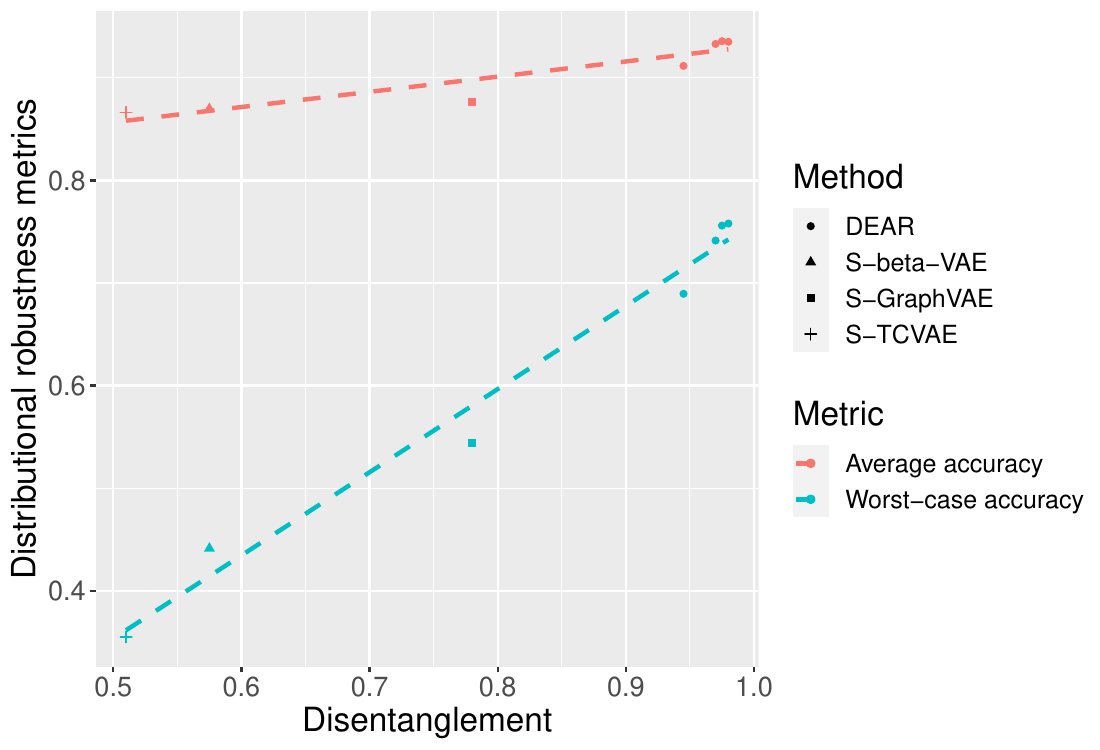}}
\caption{\small Relationship between causal disentanglement and performance in downstream tasks.}
\label{fig:disen_down}
\end{figure}

%

\section{\rr{Discussion: supervision for disentanglement learning}}\label{app:sup}

We comment on the two forms of supervision that may be available and commonly considered in literature for the task of disentangled representation learning.
\begin{itemize}
\item Form 1 (direct and few labels): in some scenarios, we may have some conceptual knowledge about the data in the sense that we know the concepts of the underlying generative factors of data, especially those concepts that we are interested in. In such cases, a weakly supervised setting is feasible where only a few samples have annotated labels of the factors, since at least manual labeling of a few examples is practical. A representative work uses this form of supervision is \citet{locatello2019disentangling}. 
\item Form 2 (auxiliary information of a full sample): in some other scenarios, we have no prior knowledge on what the ground-truth concepts are and thus cannot get the direct annotated labels of them. Auxiliary information is then needed for all samples with some assumptions on the variability of such side information as well as its correlation with the true generative factors. A representative work along this line is \citet{khemakhem2020variational}. 
\end{itemize}

Both settings have some real applications and limitations which make them complementary. 
On one hand, Form 2 in general tends to require ``weaker" supervision than Form 1 in the sense that it does not require direct annotations of the true factors themselves.  Thus, efforts towards general provable disentanglement should be put in studying along Form 2. 
However, in fact, the auxiliary observed variables in Form 2 also require certain knowledge on the true factors in order to verify the mathematical assumptions required in identifiability, e.g. the variability condition in \citet{khemakhem2020variational}. Intuitively, the auxiliary variables which can guarantee disentanglement should have enough variability and correlation with the true factors. In addition, current identifiability theory with Form 2 still assumes relatively strong and limited structure assumption on the true factors, e.g., conditional independence in \citet{khemakhem2020variational}.

On the other hand, current research on disentanglement mostly focuses on the scenarios where we indeed have some conceptual knowledge on the true factors, which makes Form 1 at least a feasible and practical setting. For simple structures of true factors (e.g., independence or conditional independence, as assumed in most previous work), existing methods with Form 1 can achieve disentanglement, which is much more straightforward compared to provable disentanglement with supervision of Form 2. However, for more complex structures (e.g., a causal graph, as considered in our paper), existing methods using independent or conditionally independent priors generally cannot identify disentanglement even with supervision in Form 1, as shown in our Proposition 4. In particular, existing formulations (e.g., \citet{locatello2019disentangling}) in general cannot even reach the optimum of the supervised loss, so they cannot disentangle. To this end, our paper proposes a bidirectional generative model with an SCM prior trained using a GAN-type algorithm, which resolves this problem under the clearly stated setup and assumptions.

\section{\revise{Discussion: generalization to unseen interventions}}\label{app:unseen}

We recall that DEAR is trained on observational data, that is, the training data is IID sampled from the data distribution $q_x$ and the latent variables follow IID a joint distribution $p_z$, e.g., induced by a SCM, without a mixture with interventional distributions. When the generative model is perfectly learned, we have $q_x(x)=\int p_{G^*}(x|z)p_z(z)dz$.
Then an interesting question would be how our method generalizes to unseen interventions. Specifically let $p_z^I(z)$ be an interventional distribution. The consequent data distribution $q_x^I(x)=\int p_{G^*}(x|z)p_z^I(z)dz$ does not match the observational distribution $q_x$ and model trained on an IID sample from $q_x$ have not seen $q_x^I$.

Now we give some insights on how given the true graph structure, DEAR trained on observational data can sample from an interventional distributions $q_x^I$. We start with the general definition of SCM.
A structural causal model (SCM) over variables $Z_i,i=1\dots,m$ can be generally expressed as
\begin{equation}\label{eq:scm}
	Z_i=f_i(Pa(Z_i;A),\epsilon_i),\ i=1,\dots,m,
\end{equation}
where $A$ denotes the adjacency matrix, $Pa(Z_i;A)$ denotes the set of parents of node $Z_i$, and $\epsilon_i$ is the exogenous noise. 
Learning of an SCM consists of structure learning of $A$ and parameter estimation of all the assignments $f_i, i=1,\dots,m,$ in the SCM, i.e., how each node is generated given its parents and exogenous noise. 
When given the underlying causal structure, standard parameter estimation methods like maximum likelihood estimation can yield a consistent estimator of the true SCM assignments from the observational data: 
\begin{equation}\label{eq:scm_est}
	Z_i=\hat{f}_i(Pa(Z_i;A),\epsilon_i),\ i=1,\dots,m.
\end{equation}

Note that an intervention can be defined as operations that modify a subset of assignments in \eqref{eq:scm}, e.g., changing $\epsilon_i$, or setting $f_i$ (and thus $Z_i$) to a constant~\citep{pearl2000models,scholkopf2019causality}. Therefore, with the estimated SCM \eqref{eq:scm_est} at hand, we can sample from any interventional distributions. 

We illustrate this through some experimental results shown in Figure~\ref{fig:counterfactual}. In (a), we intervene on the two factors \textit{bald} and \textit{gender}. In each line, we keep \textit{gender} = female and gradually increase the probability of them being bald. Particularly in the red box, we obtain images of bald female faces which have never been seen from the observational data. In (b), we intervene on \textit{beard} and \textit{gender} to generate images of female with beard which are shown in the red box. In (c), we show some generated samples that gradually wear (sun)glasses, while in the training data, there are only images with or without glasses but no intermediate states. In (d), we intervene on all four factors. In each line, the image in the middle follows the true SCM (described later in Appendix~\ref{app:exp_detail}) so that the factors satisfy the projection law. Then we change the value of only one factor while keeping others fixed, which leads to samples not satisfying the projection law. 
In summary, we see that although these interventions are not appearing in the observational data, DEAR is able to generate samples from such interventional distributions, suggesting its generalizability to unseen interventions.

\begin{figure}
\centering
\subfigure[Bald female]{\begin{tabular}{c}
	\includegraphics[width=0.45\textwidth]{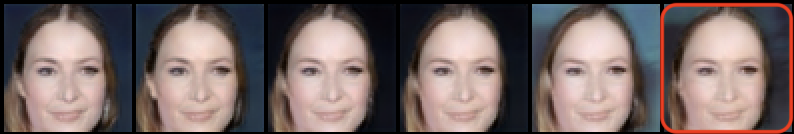}\\
	\includegraphics[width=0.45\textwidth]{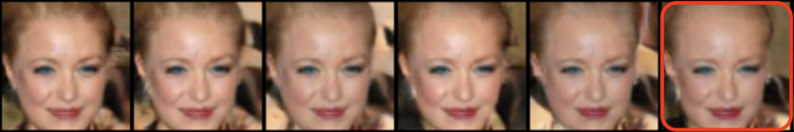}\\
	\includegraphics[width=0.45\textwidth]{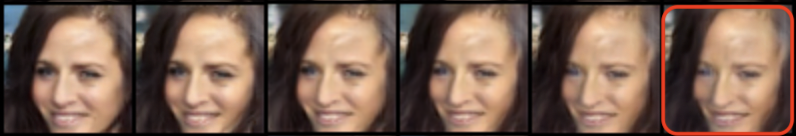}
\end{tabular}}
\subfigure[Female with beard]{\begin{tabular}{c}
	\includegraphics[width=0.45\textwidth]{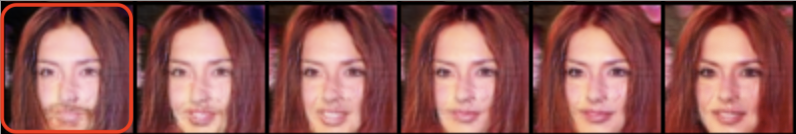}\\
	\includegraphics[width=0.45\textwidth]{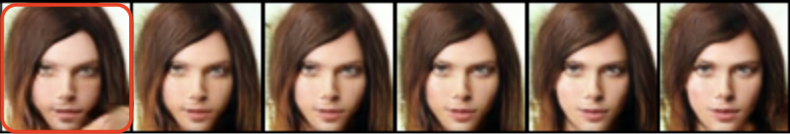}\\
	\includegraphics[width=0.45\textwidth]{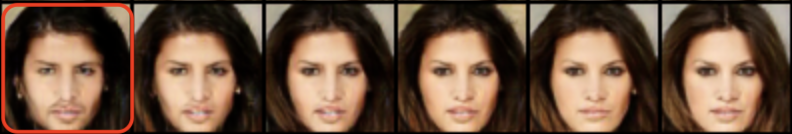}
\end{tabular}}
\subfigure[\textit{glasses}: gradually wearing (sun)glasses]{
\begin{tabular}{c}
	\includegraphics[width=0.45\textwidth]{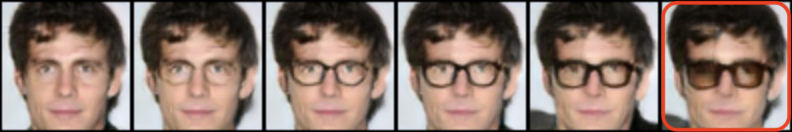}\\
	\includegraphics[width=0.45\textwidth]{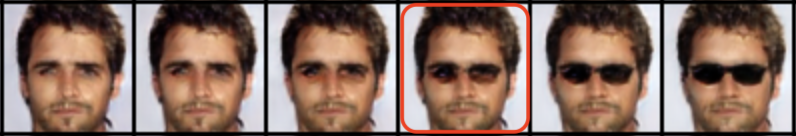}\\
	\includegraphics[width=0.45\textwidth]{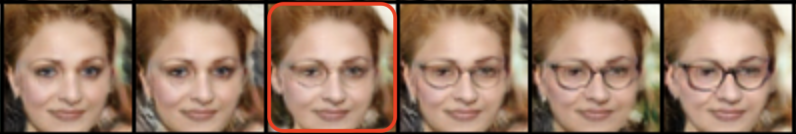}
\end{tabular}}
\subfigure[Images not following the projection law]{\begin{tabular}{c}
\includegraphics[width=0.45\textwidth]{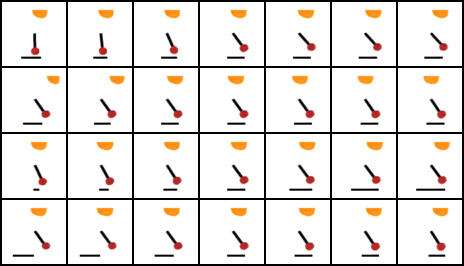}
\end{tabular}}
\caption{Samples from unseen interventional distributions.}
\label{fig:counterfactual}
\end{figure}

More systematic analysis on the out-of-distribution generalizability of the encoder is to be explored in future work. One potential direction is to utilize the generalizability of the generator to unseen interventions to improve the OOD performance of the encoder. Along this direction, for example, \citet{sauer2021counterfactual} recently combined disentangled generative models and out-of-distribution classification, but adopted a different disentanglement framework. 

%

\section{\revise{Experiments in the independent case}}\label{app:ind_benchmark}

In this section, we test our method on benchmark data sets where the ground-truth generative factors are independent, which is a spacial case of the causal case with no edge in the graph structure. 
\citet{gondal2019transfer} proposed a real-world benchmark data set, MPI3D-real, which consists of over one million images of physical 3D objects with seven independent factors of variation such as object color, shape, size and position. They also provided two simulation data sets. We test a simplified version of DEAR with an independent prior (a standard Gaussian) on real data MPI3D-real and simulated data MPI3D-simu (MPI3D-realistic in \citet{gondal2019transfer}). Both data sets consist of 1,036,800 images with resolution $64\times64\times3$. We assume 0.01\% of the data (around 100 samples) have annotated labels. No prior information on the graph structure is needed since we directly use an independent prior for the latent variable instead of an SCM. 

We are interested in the disentanglement performance on both real and simulated data, as well as the transferability of the representations from simulation to the real-world or reverse. As shown in the experiments by \citet{gondal2019transfer}, most existing VAE-based methods perform similarly in disentanglement and all the metric for disentanglement also gives similar results. Hence, we consider weakly-supervised TCVAE as a representative of the baseline methods and consider FactorVAE metric to measure disentanglement. As we have mentioned, the weakly-supervised setting can guarantee the alignment between each representation and a particular factor. Therefore, when computing the FactorVAE metric, we skip the majority-vote classifier and directly apply lines 10-12 in Algorithm~\ref{alg:metric} to obtain the score. 

As shown in Table~\ref{tab:mpi}, DEAR always significantly outperforms TCVAE in the disentanglement score and is particularly superior when training and testing on the same data set. In the transfer setting where we apply the encoder trained on one data set to another data set, both methods suffer from a performance decline. This is consistent with the discovery in \citet{gondal2019transfer} who found that direct transfer of learned representations from simulated to real data seems to work rather poorly. 
To sum up, this section suggests that DEAR can achieve state-of-the-art performance in data whose underlying factors are independent, though it is developed to handle the causal case. 

\begin{table}[h]
\centering
\begin{tabular}{cccc}
\toprule
Train & Test & Method & Disentangle\\\midrule
\multirow{2}{*}{MPI3D-simu} & \multirow{2}{*}{MPI3D-simu} & DEAR & 0.9543 \\
&& TCVAE & 0.5800 \\\midrule 
\multirow{2}{*}{MPI3D-real} & \multirow{2}{*}{MPI3D-real} & DEAR & 0.9579 \\
&& TCVAE & 0.5793 \\\midrule
\multirow{2}{*}{MPI3D-simu} & \multirow{2}{*}{MPI3D-real} & DEAR & 0.4879 \\
&& TCVAE & 0.3614 \\\midrule
\multirow{2}{*}{MPI3D-real} & \multirow{2}{*}{MPI3D-simu} & DEAR & 0.5571 \\
&& TCVAE & 0.3443 \\
\bottomrule
\end{tabular}
\caption{Results on MPI3D data.}\label{tab:mpi}
\end{table}

\section{Implementation details}\label{app:exp_detail}

In this section, we provide the details of the experimental setup and the network architectures used for all experiments, followed by a description of the synthesized Pendulum data set.

\medskip
\noindent\textbf{Preprocessing and hyperparameters.} \ 
We pre-process the images by taking center crops of $128\times128$ for CelebA and resizing all images in CelebA and Pendulum to the $64\times64$ resolution. We adopt Adam with $\beta_1=0$, $\beta_2=0.999$, and a learning rate of $1\times10^{-4}$ for $D$, $5\times10^{-5}$ for $E$, $G$ and $F$, and $1\times10^{-3}$ for the weighted adjacency matrix $A$. We use a mini-batch size of 128. For adversarial training in Algorithm~\ref{algo}, we train the $D$ once on each mini-batch. The coefficient $\lambda$ of the supervised regularizer is set to 5 unless indicated otherwise. We use CE supervised loss for both CelebA with binary observations of the underlying factors and Pendulum with bounded continuous observations. Note that $L_2$ loss works comparable to CE loss on Pendulum. The results of DEAR and baseline methods in controllable generation presented in Section~\ref{sec:contgen} and Appendix~\ref{app:more} use full supervision of underlying generative factors, i.e., $N_s=N$, since the qualitative results with 10\% labels have no big difference.

In downstream tasks, for BGMs with an encoder, we train a two-level MLP classifier with 100 hidden nodes using Adam with a learning rate of $1\times10^{-2}$ and a mini-batch size of 128. Models were trained for around 150 epochs on CelebA, 600 epochs on Pendulum, and 50 epochs on MPI3D on NVIDIA RTX 2080 Ti.

\medskip
\noindent\textbf{\revise{Description of the Pendulum data set.}} \ 
In Figure~\ref{fig:pend_detail}, we illustrate the generative factors of the synthesized Pendulum data set, following \citet{yang2020causalvae}. Given the $\PA (\xi_1)$ and $\LA (\xi_2)$, following the projection law, one can determine the $\SL(\xi_3)$ and $\SP(\xi_4)$. Note that we consider the parallel light in our simulator. Specifically, define some constants: $c_x=10,c_y=10.5$ are the axis's of the center (pendulum origin); $l_p=9.5$ be the pendulum length (including the red ball); the bottom line of a single plot corresponds to $y=b$ with base $b=-0.5$. Then the ground-truth structural causal model is expressed as follows. 
\begin{align*}
	\xi_1&\sim\cU(\pi/4,\pi/2)\\
	\xi_2&\sim\cU(0,\pi/4)\\
	\xi_3&=\left(c_x+l_p\sin\xi_1-\tfrac{c_y-l_p\cos\xi_1-b}{\tan\xi_2}\right)-\left(c_x-\tfrac{c_y-b}{\tan\xi_2}\right)\\
	\xi_4&=\left(c_x+l_p\sin\xi_1-\tfrac{c_y-l_p\cos\xi_1-b}{\tan\xi_2}+c_x-\tfrac{c_y-b}{\tan\xi_2}\right)/2.
\end{align*}
where $\cU(a,b)$ denotes the uniform distribution on interval $(a,b)$.

\begin{figure}
\centering
\includegraphics[width=0.35\textwidth]{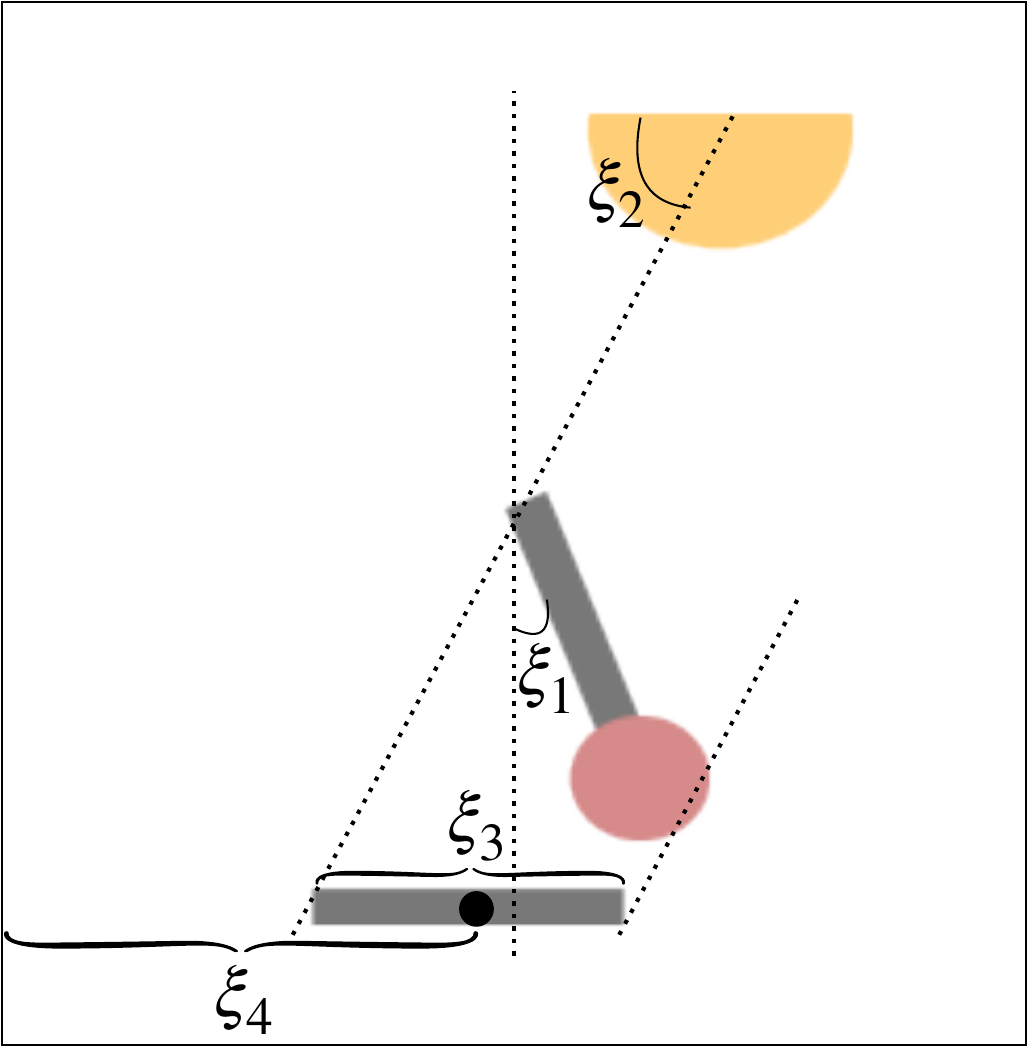}
\caption{Generative factors of the Pendulum data set. $\xi_1$: \textit{pendulum\_angle}, $\xi_2$: \textit{light\_angle}, $\xi_3$: \textit{shadow\_length}, $\xi_4$: \textit{shadow\_position}.}\label{fig:pend_detail}
\end{figure}

\medskip
\noindent\textbf{Implementation of the SCM.} \ 
Recall the nonlinear SCM as the prior $$z=f((I-A^\top)^{-1}h(\epsilon)):=F_\beta(\epsilon).
$$
We find Gaussians are expressive enough as unexplained noises, so we set $h$ as the identity mapping. As mentioned in Section~\ref{sec:gen_causal} we require the invertibility of $f$. We implement both linear and nonlinear ones. For a linear $f$, we formally refer to
$
f(z) = Wz + b,
$
where $W$ and $b$ are learnable weights and biases. Note that $W$ is a diagonal matrix to model the element-wise transformation. Its inverse function can be easily computed by 
$
f^{-1}(z) = W^{-1}(z - b).
$

For a non-linear $f$, we use piece-wise linear functions defined by
$$
[f([z]_{i})]_i = [w_0]_i[z]_i +\sum_{t=1}^{N_a} [w_t]_i([z]_i-a_i) \mathbf{I}([z]_i \geq a_i) + [b]_i
$$
where $a_0<a_1<\cdots<a_{N_a}$ are the points of division, $\mathbf{I}(\cdot)$ is the indicator function, and $\{b,w_t:t=0,\dots,N_a\}$ is the set of learnable parameters. According to the denseness of piecewise linear functions in $C[0,1]$~\citep{shekhtman1982piecewise}, the family of  such piece-wise linear functions is expressive enough to model general element-wise non-linear invertible transformations.

\medskip
\noindent\textbf{Network architectures.} \ 
We follow the architectures used in \citet{shen2020bidirectional}. Specifically, for such realistic data,  we adopt the SAGAN~\citep{zhang2019self} architecture for $D$ and $G$. The $D$ network consists of three modules as shown in Figure~\ref{fig:arch_sagan}(a) and detailed described in \citet{shen2020bidirectional}. Architectures for network $G$ and $D_x$ are given in Figure~\ref{fig:arch_sagan}(b-c) and Table~\ref{tab:sagan}.
The encoder architecture is the ResNet50~\citep{he2016deep} followed by a 4-layer MLP of size 1024 after ResNet's global average pooling layer. 

\begin{figure}[h]
\centering
\subfigure[]{
\includegraphics[width=0.6\textwidth]{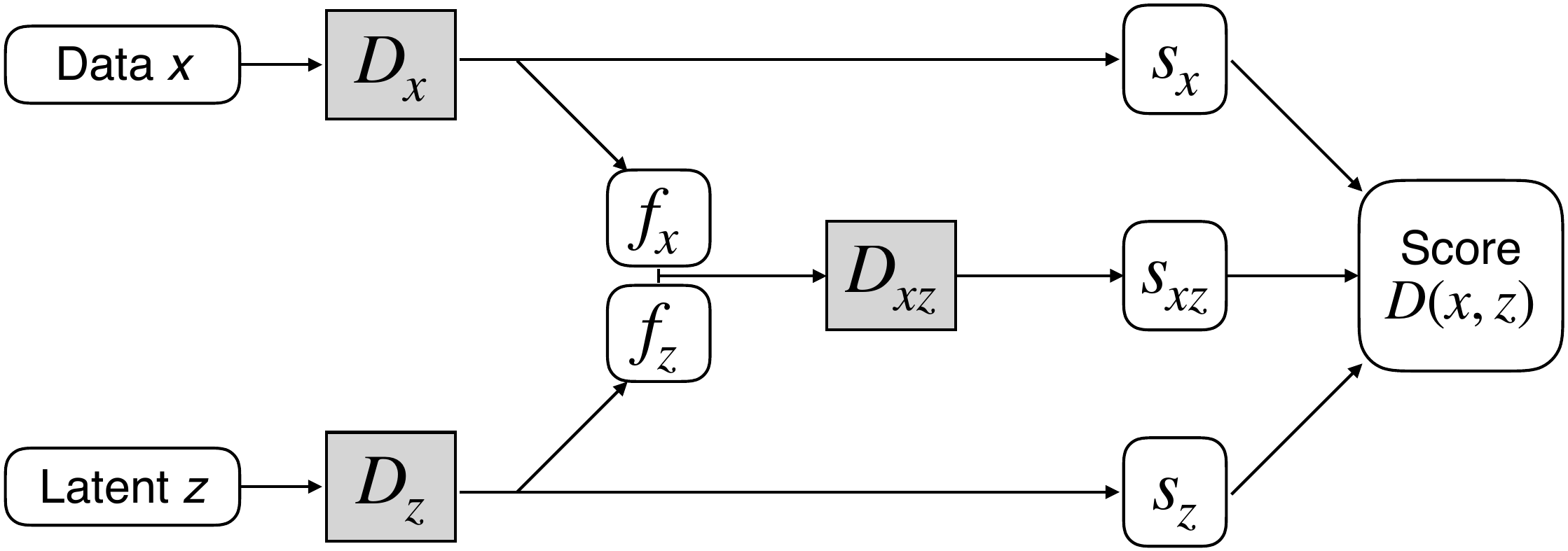}}\hspace{.1in}\\
\subfigure[]{
\includegraphics[width=0.23\textwidth]{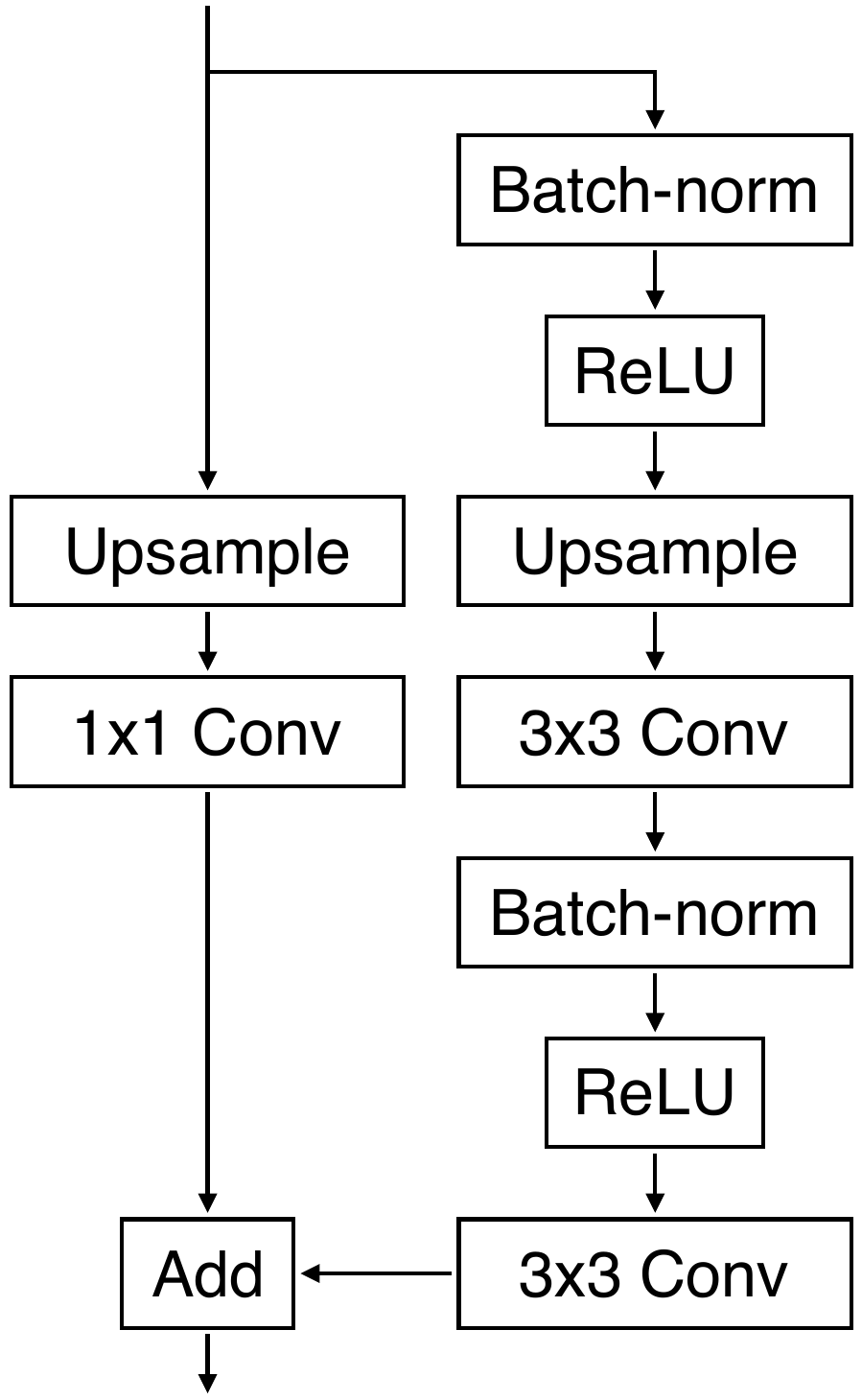}}\hspace{.5in}
\subfigure[]{
\includegraphics[width=0.23\textwidth]{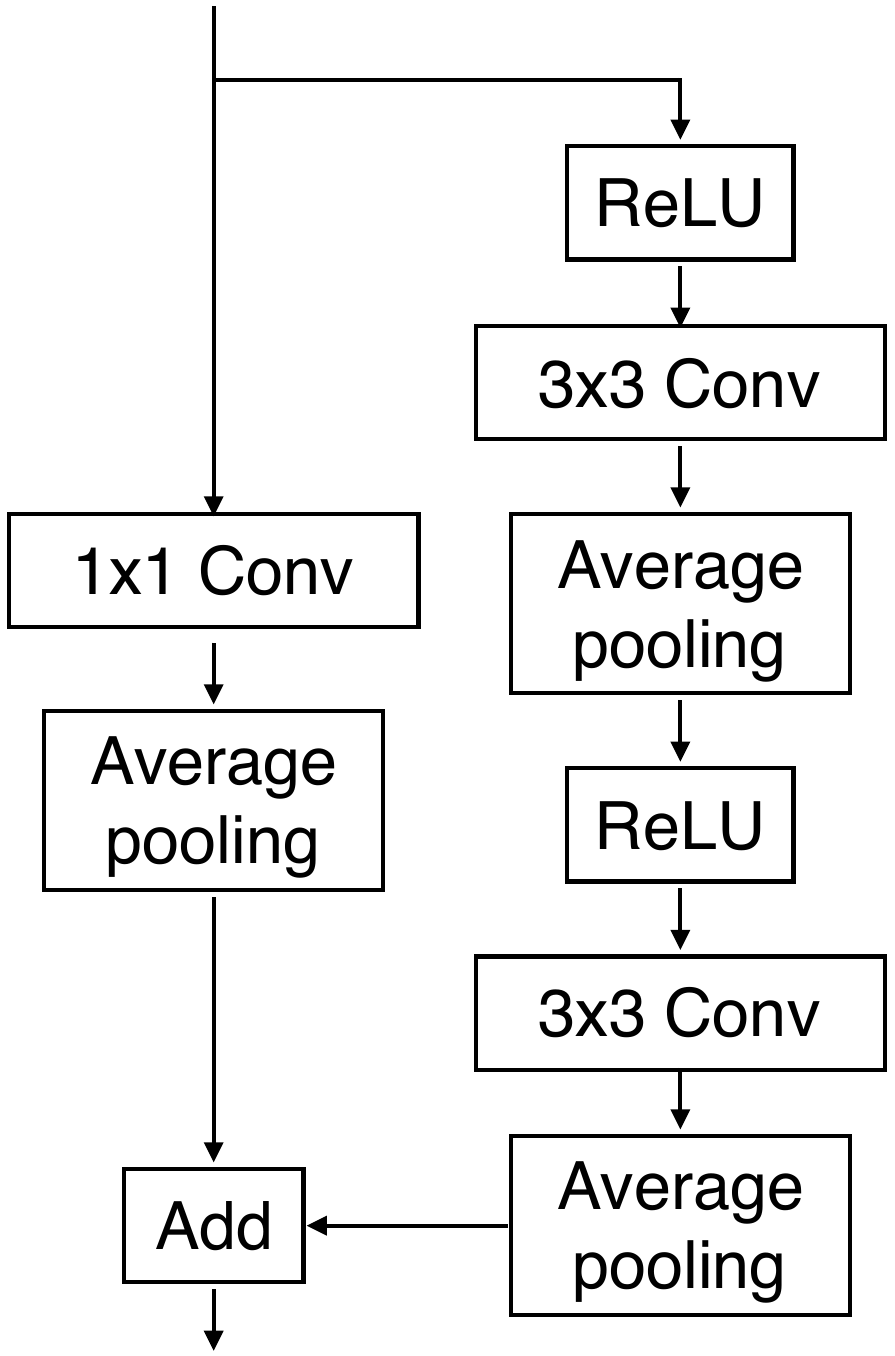}}
\caption{(a) Architecture of the discriminator $D(x,z)$; (b) A residual block (up scale) in the SAGAN generator where we use nearest neighbor interpolation for \emph{Upsampling}; (c) A residual block (down scale) in the SAGAN discriminator.}
\label{fig:arch_sagan}
\end{figure}

\begin{table}[h]
\centering\small
\caption{SAGAN architecture ($k=100$ for CelebA and $k=6$ for Pendulum and $ch=32$).}
\subtable[Generator]{
\begin{tabular}{c}
\toprule
Input: $z\in\mathbb{R}^k\sim p_z$\\\midrule
Linear $\to4\times4\times16ch$\\\midrule
ResBlock up $16ch\to16ch$\\\midrule
ResBlock up $16ch\to8ch$\\\midrule
ResBlock up $8ch\to4ch$\\\midrule
Non-Local Block $(64\times64)$ \\\midrule
ResBlock up $4ch\to2ch$\\\midrule
BN, ReLU, $3\times3$ Conv $2ch\to3$\\\midrule
Tanh\\\bottomrule
\end{tabular}
}
\hspace{1cm}
\subtable[Discriminator module $D_x$]{
\begin{tabular}{c}
\toprule
Input: RGB image $x\in\mathbb{R}^{64\times64\times3}$\\\midrule
ResBlock down $ch\to2ch$\\\midrule
Non-Local Block $(64\times64)$ \\\midrule
ResBlock down $2ch\to4ch$\\\midrule
ResBlock down $4ch\to8ch$\\\midrule
ResBlock down $8ch\to16ch$\\\midrule
ResBlock $16ch\to16ch$\\\midrule
ReLU, Global average pooling ($f_x$)\\\midrule
Linear $\to1$ ($s_x$)\\\bottomrule
\end{tabular}
}
\label{tab:sagan}
\end{table}

\medskip
\noindent\textbf{Experimental details for baseline methods.} \ 
We reproduce the S-VAEs including S-VAE, S-$\beta$-VAE and S-TCVAE using $E$ and $G$ with the same architectures as DEAR's and adopt the same optimization algorithm with same hyperparameters for training. The coefficient for the independence regularizer is set to 4 since we notice that setting a larger independence regularizer hurts disentanglement in the correlated case. We implement GraphVAE by ourselves using the same architectures (for the encoder and decoder) and optimizer as DEAR. 
The latent dependencies of GraphVAE consists of
a bottom-up network (approximate $z|x$):
\begin{center}
{\textsf{
    nn.Linear(latent\_dim, 32),
    nn.BatchNorm1d(32),
    nn.ELU(),\\
    nn.Linear(32, node\_dim),
    nn.Linear(node\_dim, 2*node\_dim)}}
      \end{center}   
and a top-down network  (approximate $z|\text{parent}$):
\begin{center}
{ \textsf{   nn.Linear(n\_parent\_nodes *node\_dim, 32), 
    nn.BatchNorm1d(32),
    nn.ELU(),\\
    nn.Linear(32, node\_dim),
    nn.Linear(node\_dim, 2*node\_dim)}}.
  \end{center}              
Note that this implementation follows the original one: $z|x, \text{parent}$ is obtained by
precision-weighted fusion in \citet{he2018graphvae}.
Since our factor dependency are explicit, we use 32 latent dimension for more efficient optimization.

For the supervised regularizer, we use $\lambda=1000$ for a balance of generative modeling and supervised regularizer. The ERM ResNet is trained using the same optimizer with a learning rate of $1\times10^{-4}$. We run the public source code from \url{https://github.com/mkocaoglu/CausalGAN} to produce the results of CausalGAN.

\section{Additional results in causal controllable generation}\label{app:more}
In this section, we present more qualitative results in causal controllable generation on two data sets using DEAR and baseline methods, including S-VAEs~\citep{locatello2019disentangling}, GraphVAE~\citep{he2018graphvae}, and CausalGAN~\citep{kocaoglu2017causalgan}.
We consider three underlying structures on two data sets: Pendulum in Figure~\ref{fig:causal_structure}(a), CelebA-Smile in Figure~\ref{fig:causal_structure}(b), and CelebA-Attractive in Figure~\ref{fig:causal_structure}(c). Note that the ordering of the rows in the traversals below matches the indices in Figure~\ref{fig:causal_structure}.

\begin{figure}[H]
\centering
\subfigure[Traversal (CelebA-Smile)]{
\includegraphics[width=0.4\textwidth]{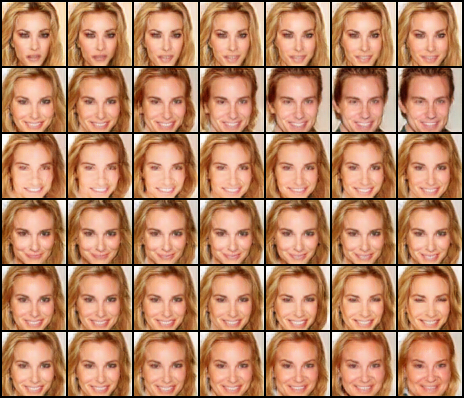}}
\subfigure[Intervention (CelebA-Smile)]{
\includegraphics[width=0.4\textwidth]{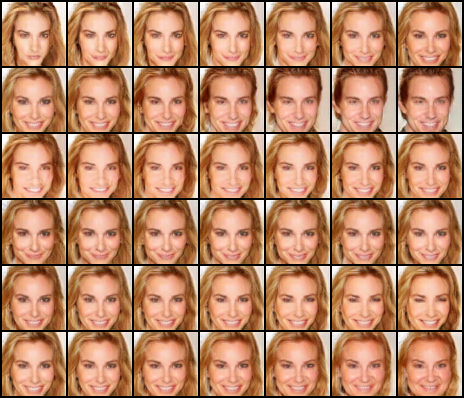}}

\subfigure[Traversal (CelebA-Attractive)]{
\includegraphics[width=0.4\textwidth]{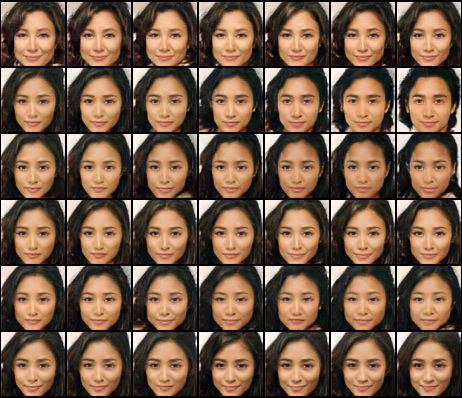}}
\subfigure[Intervention (CelebA-Attractive)]{
\includegraphics[width=0.4\textwidth]{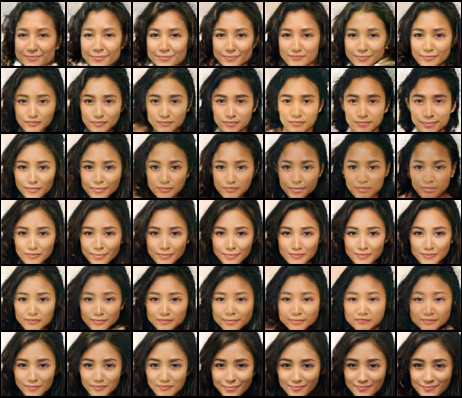}}

\subfigure[Traversal (Pendulum)]{
\includegraphics[width=0.4\textwidth]{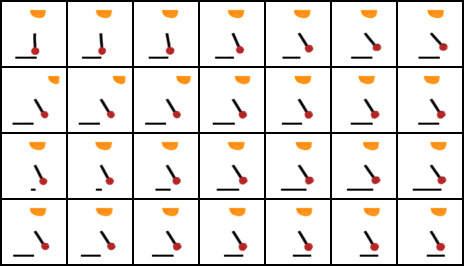}}
\subfigure[Intervention (Pendulum)]{
\includegraphics[width=0.4\textwidth]{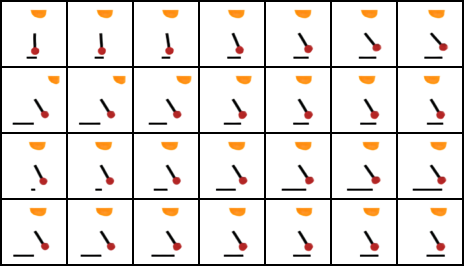}}
\caption{\small Results of DEAR. On the left we present the traditional latent traversals (the first type of intervention stated in Section~\ref{sec:contgen}) which show the disentanglement. On the right we show the results of intervening on one latent variable from which we see the consequent changes of the others (the second type of intervention). Specifically intervening on the cause variable influences the effect variables while intervening on effect variables makes no difference to the causes.}
\end{figure}

\begin{figure}[H]
\centering
\subfigure[S-TCVAE (CelebA-Smile)]{
\includegraphics[width=0.4\textwidth]{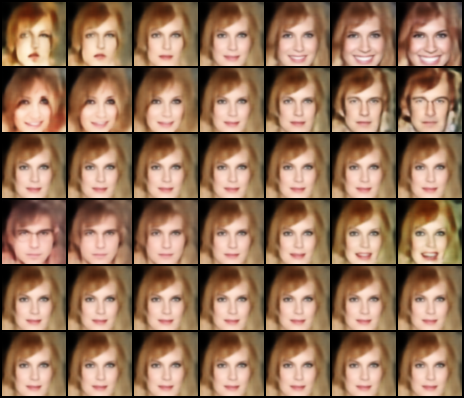}}
\subfigure[S-TCVAE (CelebA-Attractive)]{
\includegraphics[width=0.4\textwidth]{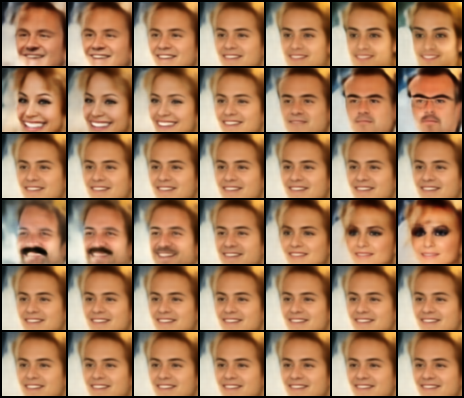}}

\subfigure[S-FactorVAE (CelebA-Smile)]{
\includegraphics[width=0.4\textwidth]{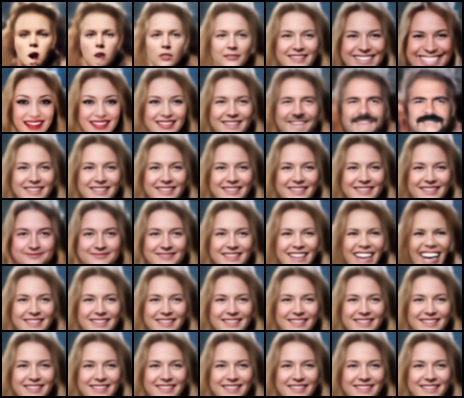}}
\subfigure[S-FactorVAE (CelebA-Attractive)]{
\includegraphics[width=0.4\textwidth]{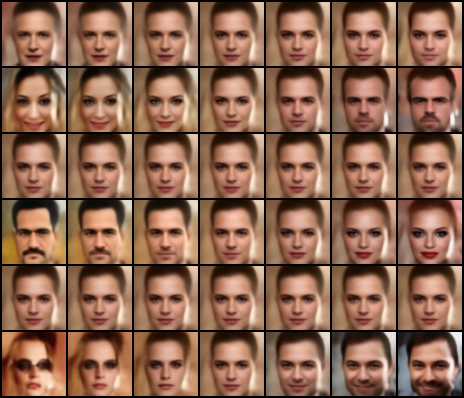}}

\subfigure[S-$\beta$-VAE (CelebA-Smile)]{
\includegraphics[width=0.4\textwidth]{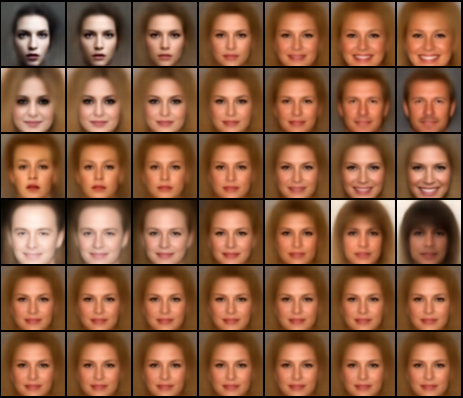}}
\subfigure[S-$\beta$-VAE (CelebA-Attractive)]{
\includegraphics[width=0.4\textwidth]{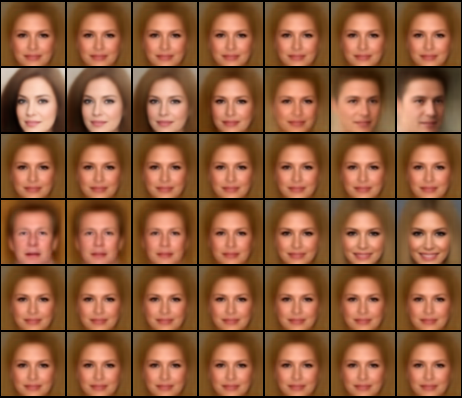}}
\caption{Traversal results of baseline methods. We see that entanglement occurs and some factors are not captured by the generative models (traversing on some dimensions of the latent vector makes no difference in the decoded images.) Besides, the generated images from VAEs are blurry.}\label{fig:vae_traversal}
\end{figure}

\begin{figure}[H]
\centering
\subfigure[CausalGAN (CelebA-Smile)]{
\includegraphics[width=0.4\textwidth]{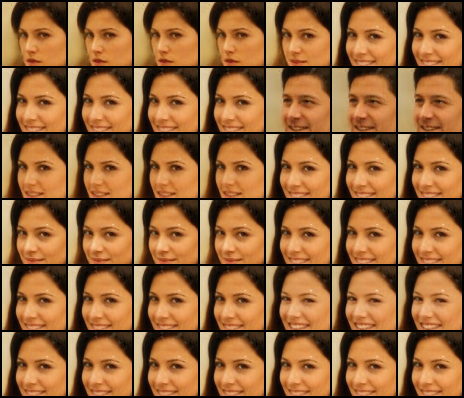}}
\subfigure[CausalGAN (CelebA-Attractive)]{
\includegraphics[width=0.4\textwidth]{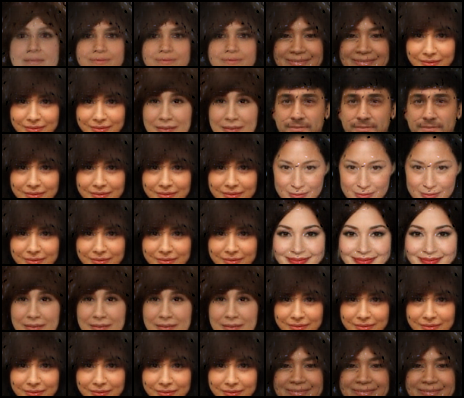}}

\subfigure[S-TCVAE (Pendulum)]{
\includegraphics[width=0.4\textwidth]{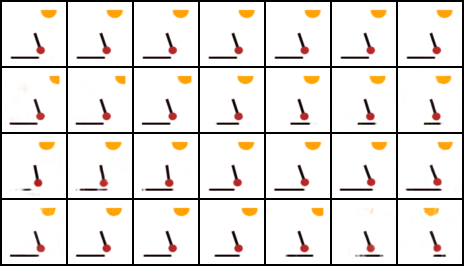}}
\subfigure[S-FactorVAE (Pendulum)]{
\includegraphics[width=0.4\textwidth]{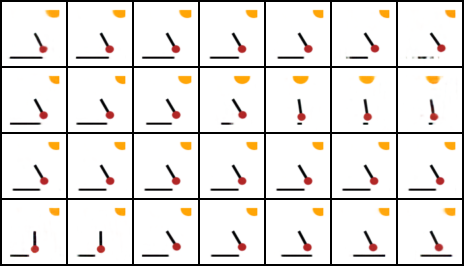}}

\subfigure[S-GraphVAE (CelebA-Attractive)]{
\includegraphics[width=0.4\textwidth]{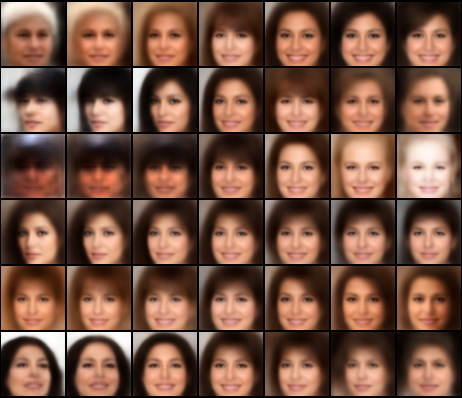}}
\subfigure[S-GraphVAE (Pendulum)]{
\includegraphics[width=0.4\textwidth]{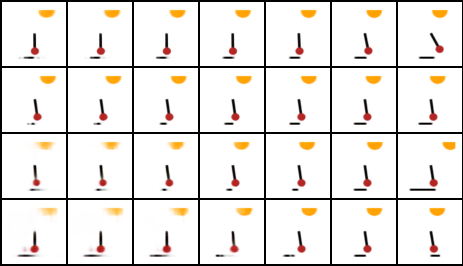}}
\caption{Traversal results of baseline methods. CausalGAN uses the binary factors as the conditional attributes, so the traversals (a-b) appear some sudden changes. In contrast, we regard the continuous logit of binary labels as the underlying factors and hence enjoy smooth manipulations. In addition, the controllability of CausalGAN is also limited, since entanglement still exists. Results of S-VAEs are explained in Figure~\ref{fig:vae_traversal}. The traversal of S-GraphVAE on Pendulum looks better than those of S-VAEs, especially in the first two factors, while the performance on CelebA is poor. Besides, S-GraphVAE has poor generation quality.}
\end{figure}

\vskip 0.2in
\nocite{*} 
\bibliography{main}

\end{document}